\documentclass{article}
\pdfoutput=1
\PassOptionsToPackage{numbers, compress}{natbib}

\usepackage[final]{neurips_2022}

\usepackage[utf8]{inputenc} 
\usepackage[T1]{fontenc}    
\usepackage{hyperref}       
\usepackage{url}            
\usepackage{booktabs}       
\usepackage{amsfonts}       
\usepackage{nicefrac}       
\usepackage{microtype}      
\usepackage{xcolor}         
\usepackage{amsmath}
\usepackage{lineno}
\usepackage{array}
\usepackage{color}
\usepackage{colortbl}
\usepackage{amssymb}
\usepackage{amsthm}
\usepackage{algorithm}
\usepackage{algorithmic}
\usepackage{graphicx}

\usepackage{adjustbox}
\usepackage{wrapfig}

\usepackage{multirow,enumitem}

\newtheorem{theorem}{Theorem}
\newtheorem{lemma}{Lemma}
\newtheorem{definition}{Definition}
\newtheorem{assumption}{Assumption}
\newtheorem{remark}{Remark}
\newtheorem{cor}{Corollary}
\newtheorem{proposition}{Proposition}

\newcommand{\stderr}[1]{\scriptsize $\pm #1$}
\usepackage{xspace}
\newcommand{\method}{\texttt{GAGA}\xspace }
\newcommand\mdoubleplus{\mathbin{+\mkern-10mu+}}
\title{\method: Deciphering Age-path of Generalized\\ Self-paced Regularizer}

\author{%
   Xingyu Qu$^{1}$\thanks{These authors contributed equally.}~~~ Diyang Li$^{2*}$~~~ Xiaohan Zhao$^{2}$~~~ Bin Gu$^{1,2}$\\
  $^1$ Mohamed bin Zayed University of Artificial Intelligence \\$^2$  Nanjing University of Information Science \& Technology\\
  \texttt{\{Xingyu.Qu,bin.gu\}@mbzuai.ac.ae,} \\
  \texttt{Diyounglee@gmail.com, xiaohan.zhao42@foxmail.com}
 }

\begin{document}

\maketitle
\begin{abstract}
	Nowadays self-paced learning (SPL) is an important machine learning paradigm that mimics the cognitive process of humans and animals. The SPL regime involves a self-paced regularizer and a gradually increasing age parameter, which plays a key role in SPL but where to optimally terminate this process is still non-trivial to determine. A natural idea is to compute the solution path w.r.t. age parameter (\emph{i.e.}, age-path). However, current age-path algorithms are either limited to the simplest regularizer, or lack solid theoretical understanding as well as computational efficiency. To address this challenge, we propose a novel \underline{G}eneralized \underline{Ag}e-path \underline{A}lgorithm (\method) for SPL with various self-paced regularizers based on ordinary differential equations (ODEs) and sets control, which can learn the entire solution spectrum w.r.t. a range of age parameters. To the best of our knowledge, \method is the first \emph{exact} path-following algorithm tackling the age-path for \emph{general} self-paced regularizer. Finally the algorithmic steps of classic SVM and Lasso are described in detail. We demonstrate the performance of \method on real-world datasets, and find considerable speedup between our algorithm and competing baselines.

\end{abstract}
\section{Introduction}
\paragraph{The SPL.} Self-paced learning (SPL) \cite{Kumar} is a classical learning paradigm and has attracted increasing attention in the communities of machine learning \cite{jiang2014self,zhao2015self,ma2017self,lin2017active, ghasedi2019balanced}, data mining \cite{gao2018self,gan2020supervised} and computer vision \cite{supancic2013self,pan2020self, wu2021bispl}. The philosophy under this paradigm is simulating the strategy that how human-beings learn new knowledge. In other words, SPL starts learning from easy tasks and gradually levels up the difficulty while training samples are fed to the model sequentially. 
{At its core, SPL can be viewed as an automatic variant of curriculum learning (CL) \cite{bengio2009curriculum, soviany2022curriculum}, which uses prior knowledge to discriminate between simple instances and hard ones along the training process. Different from CL, the SPL assigns a real-valued ``easiness weight'' to each sample implicitly by adding a self-paced regularizer (SP-regularizer briefly) to the primal learning problem and optimizes the original model parameters as well as these weights.}
Considering this setting, the SPL is reported to alleviate the problem of getting stuck in bad local minima and provides better generalization as well as robustness for the models, especially in hard condition of heavy noises or a high outlier rate \cite{Kumar, jiang2015self}.
There are two critical aspects in SPL, namely the SP-regularizer and a gradually increasing age parameter. Different SP-regularizers can be designed for different kinds of training tasks. At the primary stage of SPL, only the hard SP-regularizer is utilized and leads to a binary variable for weighting samples \cite{Kumar}. Going with the advancing of the diverse SP-regularizers \cite{meng2017theoretical}, SPL equipped with different types of SP-regularizers has been successfully applied to various applications \cite{jiang2014easy, zhao2015self,ma2020self}. As for the age parameter (\emph{a.k.a}. pace parameter), the users are expected to increase its value continually under the SPL paradigm, given that the age parameter represents the maturity of current model. A lot of empirical practices have turned out that seeking out an appropriate age parameter is crucial to the SPL procedure \cite{peng2018self}. The SPL tends to obtain a worse performance in the presence of noisy samples/outliers when the age parameter gets larger, or conversely, an insufficient age parameter makes the gained model immature (\emph{i.e.} underfitting. See Figure \ref{fig:curve}).
\paragraph{On the age-path of SPL. }Although the SPL is a classical and widespread learning paradigm, when to stop the increasing process of age parameter in implementation is subject to surprisingly few theoretical studies. In the majority of  practices \cite{gong2018decomposition}, the choice of the optimal model age has, for the time being, remained restricted to be made by experience or by using the trial-and-error approach, which is to adopt the alternate convex search (ACS) \cite{wendell1976minimization} multiple times at a predefined sequence of age parameters. This operation is time-consuming and could miss some significant events along the way of age parameter. In addition, the SPL regime is a successive training process, which makes existing hyperparameter tuning algorithms like parallelizing sequential search \cite{bergstra2011algorithms} and bilevel optimization \cite{bard2013practical} difficult to apply. Instead of training multiple subproblems at different age parameters, a natural idea is to calculate the \emph{solution path} about age parameter, namely age-path (e.g., see Figure \ref{fig:path}). A solution path is a set of curves that demonstrate how the optimal solution of a given optimization problem changes w.r.t. a hyperparameter. Several papers like \cite{efron2004least,hastie2004entire} laid the foundation of solution path algorithm in machine learning by demonstrating the rationale of path tracking, which is mainly built on the Karush-Khun-Tucker (KKT) theorem \cite{karush1939minima}. Existing solution path algorithms involve generalized Lasso \cite{tibshirani2011solution}, 
semi-supervised support vector classification \cite{ogawa2013infinitesimal}, general parametric quadratic programming \cite{gu2017solution}, etc. However, none of the existing methods is available to SPL regime because they are limited to uni-convex optimization while the SPL objective is a \emph{biconvex} formulation. Assume we've got such an age-path, we can observe the whole self-paced evolution process clearly and recover useful intrinsic patterns from it. 

\paragraph{State of the art.}Yet, a rapidly growing literature \cite{li2016multi, gong2018decomposition, li2019pareto, gu2021finding} is devoted to developing better algorithms for solving the SPL optimization with ideas similar to age-path. However, despite countless theoretical and empirical efforts, the understanding of  age-path remains rather deficient. Based on techniques from incremental learning \cite{gu2018new}, \cite{gu2021finding} derived an exact age-path algorithm for mere \emph{hard SP-regularizer}, where the path remains piecewise constant. \cite{li2016multi, li2019pareto} proposed a multi-objective self-paced learning (MOSPL) method to approximate the age-path by evolutionary algorithm, which is not theoretically stable. Unlike previous studies, the difficulty of revealing the exact generalized age-path lies in the continuance of imposed weight and the alternate optimization procedure used to solve the minimization function. From this point of view, the technical difficulty inherent in the study of age-path with general SP-regularizer is intrinsically more challenging.
\paragraph{Proposed Method.} In order to tackle this issue, we establish a novel \underline{G}eneralized \underline{Ag}e-path \underline{A}lgorithm (\method) for various self-paced regularizers, which prevents a straightforward calculation of every age parameter. Our analysis is based on the theorem of partial optimum while previous theoretical results are focused on the implicit SPL objective. In particular, we enhance the original objective to a single-variable analysis problem, and use different sets to partition samples and functions by their confidence level and differentiability. Afterward, we conduct our main theorem results based on the technique of ordinary differential equations (ODEs). In the process, the solution path hits, exits, and slides along the various constraint boundaries. The path itself is piecewise smooth with kinks at the times of boundary hitting and escaping. Moreover, from this perspective we are able to explain some shortcomings of conventional SPL practices and point out how we can improve them. We believe that the proposed method may be of independent interest beyond the particular problem studied here and might be adapted to similar biconvex schemes.  
\paragraph{Contributions.} Therefore, the main contributions brought by this work are listed as follows.  
\begin{itemize}[leftmargin=0.2in]
	\item  We firstly connect SPL paradigm to the concept of partial optimum and emphasize
	its importance here that has been ignored before, which gives a novel viewpoint to the robustness of SPL. Theoretical studies are conducted to reveal that our result does exist some equivalence with previous literature, which makes our study more stable.
	\item A framework of computing the \emph{exact} age-path for \emph{generalized} SP-regularizer is derived using the technique of ODEs, which allows for the time-consuming ACS to be avoided. Concrete algorithmic steps of classic SVM \cite{vapnik1999nature} and Lasso \cite{tibshirani1996regression} are given for implementation.	
	\item  Simulations on real and synthetic data are provided to validate our theoretical findings and justify their impact on the designing future SPL algorithms of practical interest.
\end{itemize}
\paragraph{Notations.}We write matrices in uppercase (\emph{e.g.}, $X$)  and vectors in lowercase with bold font (\emph{e.g.}, $\boldsymbol{x}$). Given the index set $\mathcal{E}$ (or $\mathcal{D}$), $X_{\mathcal{E}}$ (or $X_{\mathcal{E}\mathcal{D}}$) denotes the submatrix that taking rows with indexes in $\mathcal{E}$ (or rows/columns with indexes in $\mathcal{E}$/$\mathcal{D}$, respectively). Similarly notations lie on  $\boldsymbol{v}_{\mathcal{E}}$ for vector $\boldsymbol{v}$, $\boldsymbol{\ell}_{\mathcal{E}}(x)$ for vector functions $\boldsymbol{\ell}(x)$. For a set of scalar functions $\{\ell_{i}(x)\}_{i=1}^{n}$, we denote the vector function $\boldsymbol{\ell}(x)$ where $\boldsymbol{\ell}(x)=\left(\ell_{i}(x)\right)_{i=1}^{n}$ without statement and vice versa. Moreover, we defer the full proofs as well as the algorithmic steps on applications to the Appendix.

\begin{figure}[tb]
	\centering
		\begin{minipage}{0.31\textwidth}
		\begin{figure}[H]
			\centering
			\begin{adjustbox}{width=1.05\textwidth}
				\includegraphics{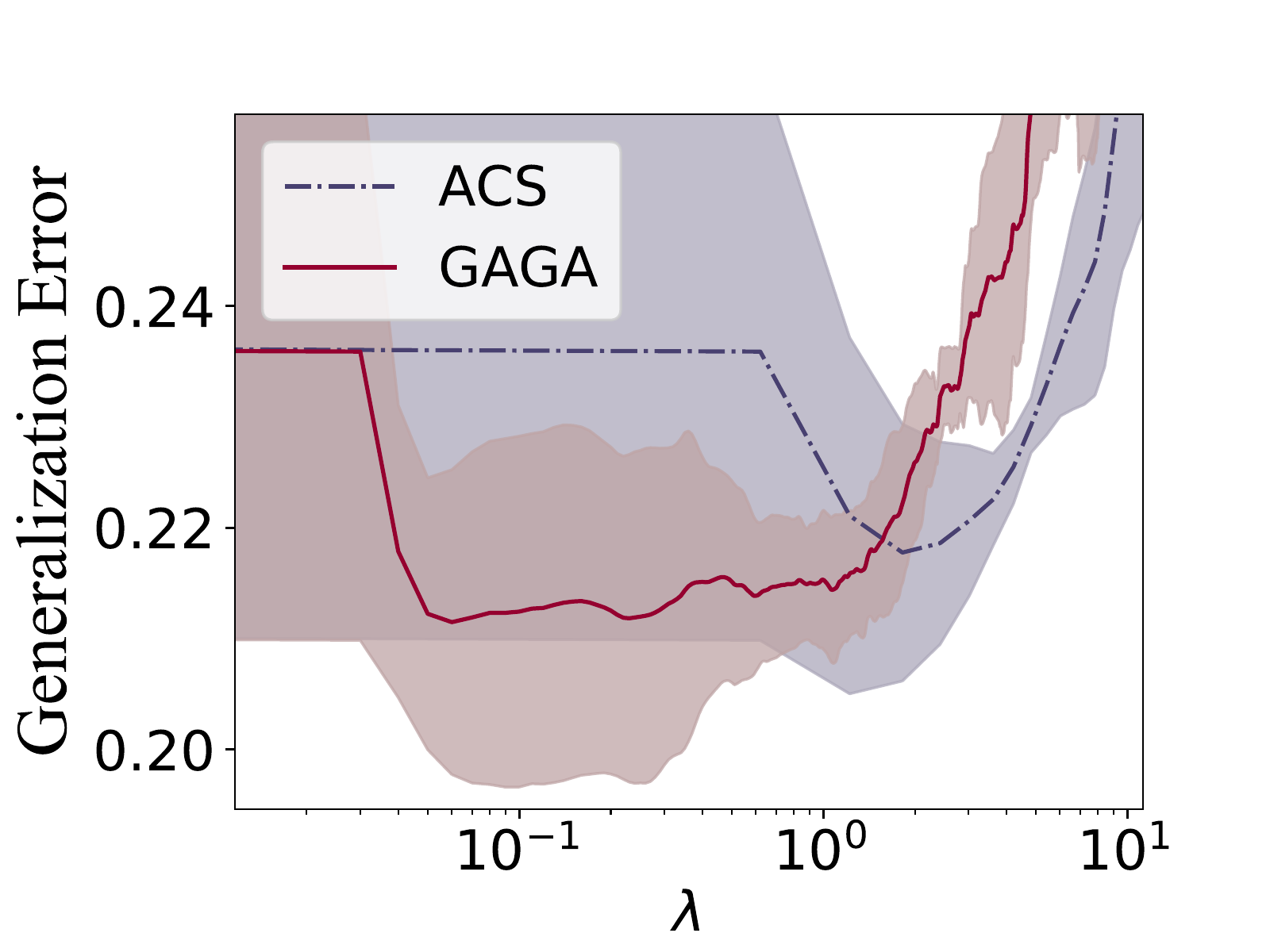}
			\end{adjustbox}
			\vspace{-1em}
			\caption{Learning curve against age $\lambda$. The curve is recorded when running linear regression on music dataset.}
			\vspace{3em}
			\label{fig:curve}
		\end{figure}
	\end{minipage}
	\hspace{6pt}
	\begin{minipage}{0.31\textwidth}
		\begin{figure}[H]
			\centering
			\begin{adjustbox}{width=1.05\textwidth}
				\includegraphics{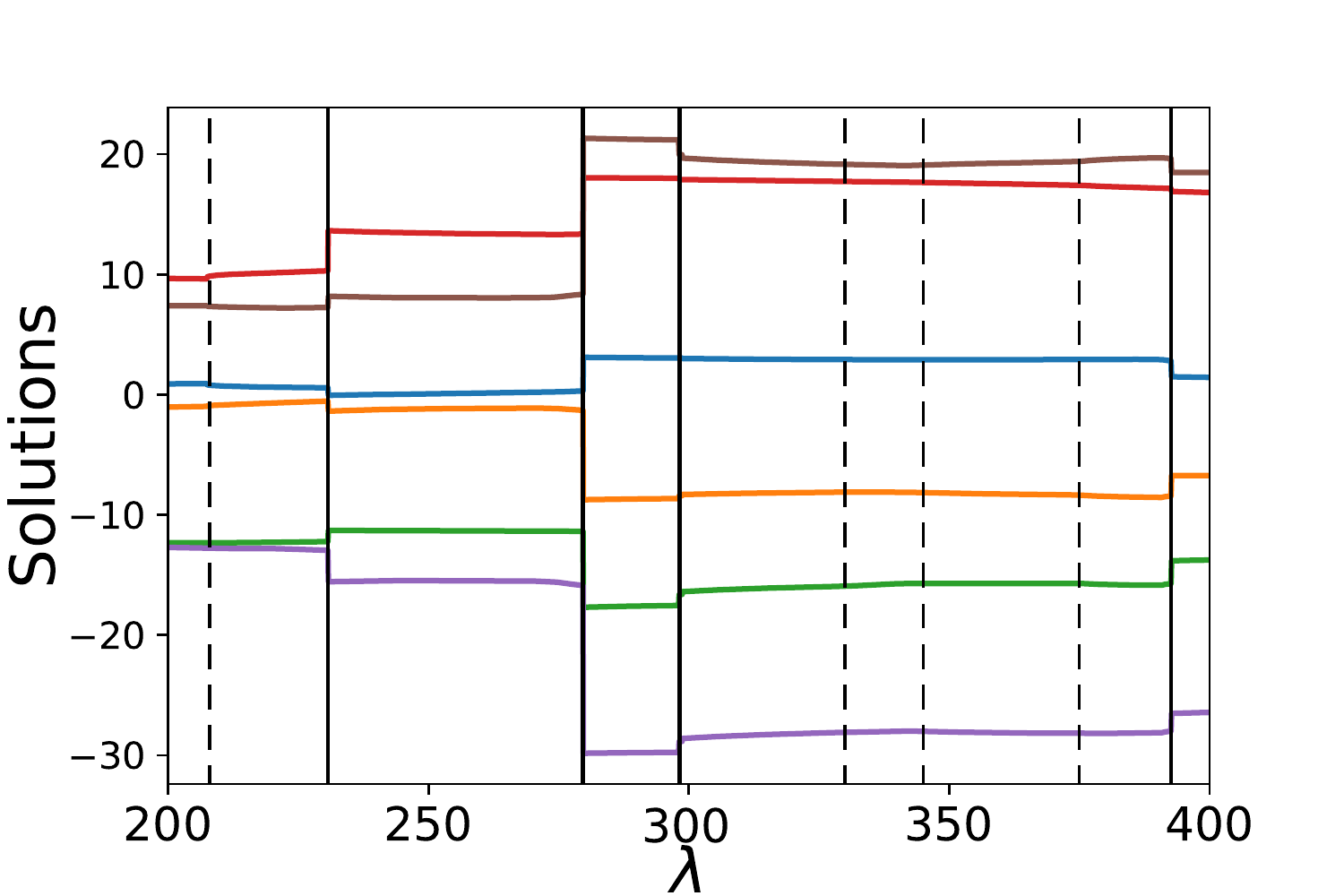}
			\end{adjustbox}
			\caption{An age-path visualisation. Different vertical lines represent different types of critical points. The figure is plotted on random 60\% features from diabetes dataset using Lasso with $\alpha=0.01$.}
			\label{fig:path}
		\end{figure}
	\end{minipage}
	\hspace{6pt}
\begin{minipage}{0.31\textwidth}
	\begin{figure}[H]
		\centering
		\begin{adjustbox}{width=.88\textwidth}
			\includegraphics{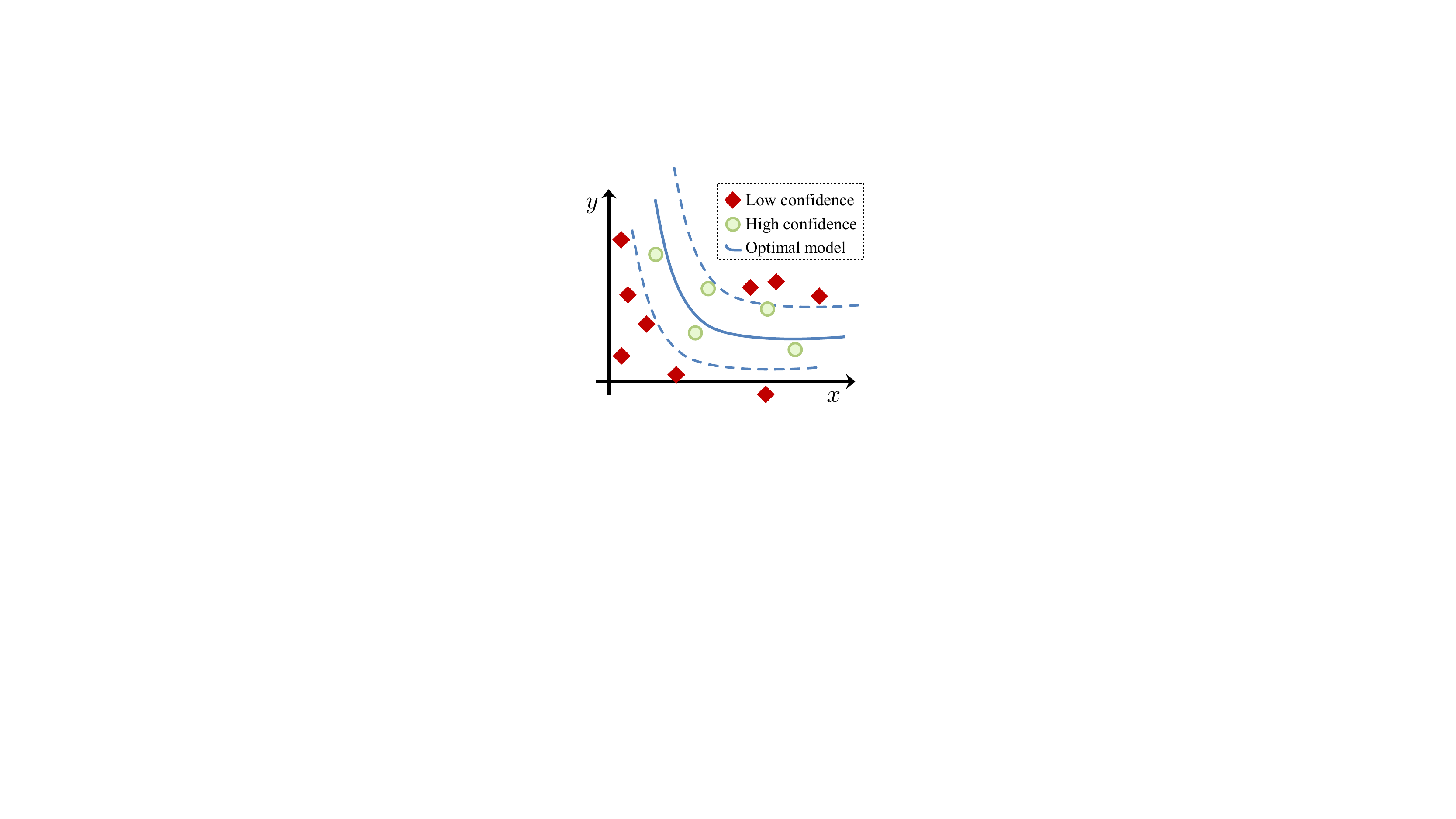}
		\end{adjustbox}
		\caption{An \emph{example} of set partition in 2-D space. Sample points of same colors belong to one set. The two dashed lines represent partition boundaries (smooth surfaces), which satisfies $l_i=\lambda$.}
		\label{fig:set}
	\end{figure}
\end{minipage}
\end{figure}

\section{Preliminaries}\label{preli}

\subsection{Self-paced Learning}

Suppose we have a dataset containing the label vector $\boldsymbol{y}\in \mathbb{R}^{n}$ and $X\in \mathbb{R}^{n\times d}$, where $n$ samples with $d$ features are included. The $i$-th row $X_{i}$ represents the $i$-th data sample $x_{i}$ (\emph{i.e.}, the $i$-th observation). In this paper, the following unconstrained learning problem is considered\vspace{-4pt}
\begin{equation}\label{obj}
	\min _{\boldsymbol{w} \in \mathbb{R}^{d}}~ \sum_{i=1}^{n} \ell\left(x_{i}, y_{i} ; \boldsymbol{w}\right)+\sum_{j=1}^{m} \alpha_{j} \mathcal{R}_{j}(\boldsymbol{w}),
\end{equation}
where $\mathcal{R}_j(\cdot)$ is the regularization item with a positive trade-off parameter $\alpha_{j}$, and $\ell_{i}(\boldsymbol{w})$\footnote{Without ambiguity,  we use $\ell_{i}(\boldsymbol{w})$ as the shorthand notations of $\ell\left(x_{i}, y_{i} ; \boldsymbol{w}\right)$.} denotes loss function w.r.t. $\boldsymbol{w}$.
\begin{definition}[$\mathbb {PC}^{r}$ Function]
	Let $f: U \rightarrow \mathbb{R}$ be a continuous function on the open set $U \in \mathbb{R}^{n}$. If $\left\{f_{i}\right\}_{i \in I_{f}}$ is a set of $\mathbb C^{r}$ (i.e., r-times continuously differentiable) functions such that $f(\boldsymbol{x}) \in\left\{f_{i}(\boldsymbol{x})\right\}_{i \in I_{f}}$ holds for every $\boldsymbol{x} \in U$, then $f$ is an r-times piecewise continuously differentiable function, namely  $\mathbb{PC}^{r}$ function. The $\left\{f_{i}\right\}_{i \in I}$ is a set of selection functions of $f$.
\end{definition}
\begin{assumption}
	\label{smooth}
	We assume that $\ell_{i}(\boldsymbol{w})$ and $\mathcal{R}_{j}(\boldsymbol{w})$ are convex $\mathbb {PC}^{r}$ functions each with a set of selection functions $\bigcup_{k \in I_{\ell_{i}}}\left\{D_{\ell_{i}}^{k}\right\}$ and $\bigcup_{k \in I_{\mathcal{R}_{j}}}\left\{D_{\mathcal{R}_{j}}^{k}\right\}$, respectively.
\end{assumption}
In self-paced learning, the goal is to jointly train the model parameter $\boldsymbol{w}$ and the latent weight variable $\boldsymbol{v}$ by minimizing
\vspace{-8pt}
\begin{equation}\label{spl}
	\underset{\boldsymbol{w} \in \mathbb{R}^{d}, \boldsymbol{v} \in[0,1]^{n}}{\operatorname{argmin}}\mathcal{L}(\boldsymbol{w}, \boldsymbol{v}):= \sum_{j=1}^{m} \alpha_{j} \mathcal{R}_{j}(\boldsymbol{w})+\sum_{i=1}^{n}\left[  v_{i} l_i\left(\boldsymbol{w}\right)+\rule{0pt}{10pt} f\left(v_{i}, \lambda\right)\right] ,
\end{equation}
where $f\left(v, \lambda\right)$ represents the SP-regularizer.

\subsection{SP-regularizer}
\begin{definition}[SP-regularizer \cite{meng2015objective}]\label{SPR} 
	Suppose that $v$ is a weight variable, $\ell$ is the loss, and $\lambda$ is the age parameter. $f(v, \lambda)$ is called a self-paced regularizer, if\\
	\textbf{(i)}$f(v, \lambda)$ is convex with respect to $v \in[0,1]$;\\
	\textbf{(ii)}{$v^{*}(\ell, \lambda)$} is monotonically decreasing w.r.t. $\ell$, and holds {\small	$\lim _{\ell \rightarrow 0} v^{*}(\ell, \lambda)=1, \lim _{\ell \rightarrow \infty} v^{*}(\ell, \lambda)=0$};\\
	\textbf{(iii)}$v^{*}(\ell, \lambda)$ is monotonically increasing w.r.t. $\lambda$, and holds {\small$\lim _{\lambda \rightarrow \infty} v^{*}(\ell, \lambda) \leq 1, \lim _{\lambda \rightarrow 0} v^{*}(\ell, \lambda)=0$},
	
	where $v^{*}(\ell, \lambda)=\arg \min _{v \in[0,1]} v \ell+f(v, \lambda)$.
\end{definition}
The {Definition} \ref{SPR} gives axiomatic definition of SP-regularizer. Some frequently utilized SP-regularizers include $f^{H}(v, \lambda)=-\lambda v$, $f^{L}(v, \lambda)=\lambda\left(\frac{1}{2} v^{2}-v\right)$, $f^{M}(v, \lambda, \gamma)=\frac{\gamma^{2}}{v+\gamma / \lambda}$ and $f^{L O G}(v, \lambda, \alpha)=\frac{1}{\alpha} K L(1+\alpha \lambda, v)$, which represents hard, linear, mixture, LOG SP-regularizer, respectively.
\subsection{Biconvex Optimization}
\begin{definition}[Biconvex Function] \label{bifunc}
	A function $f: B \rightarrow \mathbb{R}$ on a biconvex set $B \subseteq \mathcal{X} \times \mathcal{Y}$ is called a biconvex function on $B$, if $f_{x}(\cdot):=f(x, \cdot): B_{x} \rightarrow \mathbb{R}$ is a convex function on $B_{x}$ for every fixed $x \in \mathcal{X}$ and $f_{y}(\cdot):=f(\cdot, y): B_{y} \rightarrow \mathbb{R}$ is a convex function on $B_{y}$ for every fixed $y \in \mathcal{Y}$.
\end{definition}
\begin{definition}[Partial Optimum]\label{po}
	Let $f: B \rightarrow \mathbb{R}$ be a given biconvex function and let $\left(x^{*}, y^{*}\right) \in B$. Then, $z^{*}=\left(x^{*}, y^{*}\right)$ is called a partial optimum of $f$ on $B$, if	$	f\left(x^{*}, y^{*}\right) \leq f\left(x, y^{*}\right) \forall x \in B_{y^{*}} \text { and } f\left(x^{*}, y^{*}\right) \leq f\left(x^{*}, y\right) \forall y \in B_{x^{*}}.	$
\end{definition}\vspace{-3pt}
\begin{wrapfigure}[]{R}{0.46\textwidth}
	\vspace{-2.2em}
	\begin{minipage}{0.46\textwidth}
		\begin{algorithm}[H] 
			\caption{~{Alternate Convex Search} (\textbf{ACS})}
			\label{acs}
			\begin{algorithmic}[1]
				\REQUIRE{Dataset ${X}$ and $y$, age parameter $\lambda$.}
				\STATE Initialize $\boldsymbol{w}$.
				\WHILE{not converged }
				\STATE Update $\boldsymbol{v}^{*}=\operatorname{argmin}_{\boldsymbol{v}}\mathcal{L}(\boldsymbol{w}^*, \boldsymbol{v})$.
				\STATE Update $\boldsymbol{w}^{*}=\operatorname{argmin}_{\boldsymbol{w}}\mathcal{L}(\boldsymbol{w}, \boldsymbol{v}^*)$.
				\ENDWHILE
				\ENSURE{ $\hat{\boldsymbol w}$}
			\end{algorithmic}
		\end{algorithm}
	\end{minipage}\vspace{-6pt}
\end{wrapfigure}
Optimizing (\ref{spl}) leads to a biconvex optimization problem and is generally non-convex with fixed $\lambda$, in which a number of local minima exist and previous convex optimization tools can't achieve a promising effect \cite{li2015iterative}. It's reasonably believed that algorithms taking advantage of the biconvex structure are more efficient in the corresponding setting. For frequently used one, ACS (\emph{c.f.} Algorithm \ref{acs}) is presented to optimize $x$ and $y$ in $f(x, y)$ alternately until terminating condition is met.

\begin{remark}%
	The order of the optimization subproblems in line 3 \& 4 in Algorithm \ref{acs} can be permuted.
\end{remark}
\begin{theorem}\label{conver1}\cite{gorski2007biconvex}
	Let $\mathcal{X} \subseteq \mathbb{R}^{n}$ and $\mathcal{Y} \subseteq \mathbb{R}^{m}$ be closed sets and let $f: \mathcal{X} \times \mathcal{Y} \rightarrow \mathbb{R}$ be continuous. Let the sequence $\left\{z_{i}\right\}_{i \in \mathbb{N}_+}$ generated by ACS converges to $z^{*} \in \mathcal{X} \times \mathcal{Y}$. Then $z^{*}$ is a partial optimum.
\end{theorem}

\subsection{Theoretical Consistency}

Researchers in earlier study \cite{meng2015objective} theoretically conducted the latent SPL loss (\emph{a.k.a.}, implicit objective) and further proved that the SPL paradigm  converges to the \emph{stationary point} of the latent objective under some mild assumptions, which gives explanation to the robustness of the SPL \cite{meng2015objective,ma2018convergence}.
In this paper, we focus on the \emph{partial optimum} of original SPL objective and result is given in Theorem \ref{conver2}. 

\begin{theorem}\label{conver2}
	Under the same assumptions in Theorem 2 of \cite{ma2018convergence}, the partial optimum of SPL objective consists with the stationary point of implicit SPL objective $G_{\lambda}$.
\end{theorem}\vspace{-4pt}
Factoring in both Theorem \ref{conver1} \& \ref{conver2}, the ACS procedure (or its variations) used in SPL paradigm indeed finds the partial optimum of SPL objective, which unifies the two proposed analysis frameworks and provides more in-depth understanding to the intrinsic mechanism behind the SPL regime.

\section{Age-Path Tracking}\label{tracking}
\subsection{Objective Reformulation}

For the convenience of derivation, we denote the set $I_{\mathcal{R}}$ or $\bar{I}_{\mathcal{R}}$ to be the set of indexes $j$ where $\mathcal{R}_{j}$ is differentiable or non-differentiable at $\boldsymbol{w}$, respectively. Similarly, we have $I_{\ell}$ and $\bar{I}_{\ell}$ w.r.t. $\ell_{i}$. 
\begin{wrapfigure}[]{R}{0.66\textwidth}
	\vspace{-2.0em}
	\begin{minipage}{0.66\textwidth}
		\begin{equation*}
			\footnotesize
			x_{i} \in ~\begin{cases}
				\mathcal{E}, & \text { if } l_i<~\lambda ~\text{(or $\left( \dfrac{\lambda \gamma}{\lambda+\gamma}\right)^2$ in mixture $f\left(v, \lambda\right)$)}\\
				\mathcal{D}, & \text { if } l_i \geqslant ~\lambda~\text{(or $\lambda^2$ in mixture $f\left(v, \lambda\right)$)}  \\
				\mathcal{M}, & \text { if }\left(\dfrac{\lambda \gamma}{\lambda+\gamma}\right)^{2} \leqslant l_{i} \leqslant \lambda^{2} ~\text{(only used in mixture $f\left(v, \lambda\right)$)} 
			\end{cases}
		\end{equation*}
	\end{minipage}\vspace{-13pt}
\end{wrapfigure}

Moreover, the training of the SPL is essentially a process of adaptive sample selection, so we classify all the sample points in the training set into different sets $\mathcal{P}:=\left\lbrace \mathcal{E},\mathcal{D},\mathcal{M},...\right\rbrace $ according to their confidence (or loss)\footnote{We only present the mainstream SP-regularizers here. The partition is similar in other cases.}.  Figure \ref{fig:set} illustrates a partition example when  hard, linear or LOG SP-regularizer is used. Since subproblem in \emph{line $3$} of Algorithm \ref{acs} always gives closed-form solutions in iterations\footnote{
For example, we have $v_i^*=\left\{\begin{array}{c}
		-\ell_i / \lambda+1, \text { if } \ell_i<\lambda \\
		0, \text { if } \ell_i \geq \lambda
	\end{array}\right.$ for linear $f\left(v, \lambda\right)$. 
	More results are shown in \cite{meng2015objective}.}, we can rewrite SPL optimization objective as (\ref{repart}), which is indeed equivalent to  searching a partial optimum of (\ref{spl}).

\begin{equation}\label{repart}
	\text { Compute } \hat{\boldsymbol{w}} \text {,~ s.t.~} \hat{\boldsymbol{w}} \in\arg \min _{ \boldsymbol{w} \in \mathbb{R}^{d}}~ \sum_{j=1}^{m} \alpha_{j} \mathcal{R}_{j}(\boldsymbol{w})+\sum_{J \in \mathcal{P}} \sum_{i \in J} v_i^{*}\left(l_{i}(\hat{\boldsymbol{w}}), \lambda\right) \cdot \ell_{i}(\boldsymbol{w}).
\end{equation}

\subsection{Piecewise Smooth Age-path}\label{path}
The KKT theorem \cite{karush1939minima} states that (\ref{repart}) holds iff
\begin{equation}\label{mainKKT}
	\mathbf{0} \in ~\sum_{j=1}^{m} \alpha_{j} \partial \mathcal{R}_{j}(\hat{\boldsymbol{w}})+\sum_{J \in \mathcal{P}} \sum_{i \in J} v^{*}\left(l_{i}(\hat{\boldsymbol{w}}), \lambda\right) \cdot \partial \ell_{i}(\hat{\boldsymbol{w}}),
\end{equation}

where $\partial(\cdot)$ denotes the {subdifferential (set of all subgradients)}. In the $\mathbb{P}\mathbb{C}^{r}$ setting, subgradient can be expressed explicitly by essentially active functions (\emph{c.f.} Lemma \ref{subgd}). 

\begin{definition}[Essentially Active Set]\label{active}
	Let $f: U \rightarrow \mathbb{R}$ be a $\mathbb{P C}^{r}$ function on the open set $U \in \mathbb{R}^{n}$ with a set of selection functions  $\left\{f_{i}\right\}_{i \in I_{f}}$. For $\boldsymbol{x} \in U$, we call $I_{f}^{a}(\boldsymbol{x}):=\left\{i \in I_{f}: f(\boldsymbol{x})=f_{i}(\boldsymbol{x})\right\}$
	is the active set at $\boldsymbol{x}$, and
	$I_{f}^{e}(\boldsymbol{x}):=\left\{i \in I_{f}: \boldsymbol{x} \in \operatorname{\bf cl}\left(\rule{0pt}{10pt}\operatorname{\bf int}\left(\rule{0pt}{9pt}\left\{\boldsymbol{y} \in U: ~f(\boldsymbol{y})=f_{i}(\boldsymbol{y})\right\}\right)\right)\right\}$
	is the essentially active set at $\boldsymbol{x}$, where \textbf{cl}$(\cdot)$ and \textbf{int}$(\cdot)$ denote the closure and interior of a set.
\end{definition}
\begin{lemma}\label{subgd}\cite{held1974validation}
	Let $f: U \rightarrow \mathbb{R}$ be a $\mathbb{P}\mathbb{C}^{r}$ function on an open set $U$ and $\bigcup_{i \in I_{f}}\left\{f_{i}\right\}$ is a set of selection functions of $f$, then $
	\partial f(\boldsymbol{x})=\operatorname{conv}(\bigcup_{i \in I_{f}^{e}(\boldsymbol{x})}\left\{f_{i}(x)\right\})=\{\sum_{i \in I_{f}^{e}(\boldsymbol{x})} t_{i} \nabla f_{i}(\boldsymbol{x}): \sum_{i \in I_{f}^{e}(\boldsymbol{x})} t_{i}=1,~ t_{i} \geq 0\} .$
	Especially, if $f$ is differentiable at $\boldsymbol{x}$, $\partial f(\boldsymbol{x})=\{\nabla f(\boldsymbol{x})\}$.
\end{lemma}

\begin{assumption}
	\label{essActive}
	We assume that $I_{\mathcal{R}_{j}}^{a}(\boldsymbol{x})=I_{\mathcal{R}_{j}}^{e}(\boldsymbol{x}),I_{\ell_{i}}^{a}(\boldsymbol{x})=I_{\ell_{i}}^{e}(\boldsymbol{x})$ holds for all $\boldsymbol{x}$ considered and all $\mathcal{R}_{j},\ell_{i}$ in the following. 

\end{assumption}
\vspace{-4pt}
We adopt a mild relaxation as shown in Assumption \ref{essActive}. Investigation \cite{gebken2021structure} confirmed that it can be easily established in most practical scenarios. Without loss of generality, we suppose the following Assumption \ref{easeNotation} also holds to further ease the notation burden.
\begin{assumption}\label{easeNotation}
    We assume that $\mathcal{R}_{j},\ell_{i}$ are non-differentiable at $\boldsymbol{x}$ with multiple active selection functions, where $j\in \{1,\ldots,m\},~i\in \{1,\ldots, n\}.$
\end{assumption}

\vspace{-7pt}
Therefore, the condition (\ref{mainKKT}) can be rewritten in detail. Formally, there exists $\hat{\boldsymbol{t}}_{\mathcal{R}}$ and $\hat{\boldsymbol{t}}_{\ell}$ such that

\begin{equation}\label{KKT_}
	\begin{aligned}
		\sum_{j=1}^{m} \sum_{k \in I_{\mathcal{R}_{j}}^{a}(\hat{\boldsymbol{w}})} \alpha_{j} \hat{t}_{\mathcal{R}_{j}}^{k}(\hat{\boldsymbol{w}}) \nabla D_{\mathcal{R}_{j}}^{k}(\hat{\boldsymbol{w}})&+\sum_{J \in \mathcal{P}} \sum_{i \in J} \sum_{k \in I_{\ell_{i}}^{a}(\hat{\boldsymbol{w}})} v^{*}_i\left(\ell_{i}(\hat{\boldsymbol{w}}), \lambda\right) \hat{t}_{\ell_{i}}^{k}(\hat{\boldsymbol{w}}) \nabla D_{\ell_{i}}^{k}(\hat{\boldsymbol{w}})=\boldsymbol{0}, \\
		D_{\mathcal{R}_{j}}^{k}(\hat{\boldsymbol{w}})-D_{\mathcal{R}_{j}}^{r_{j}}(\hat{\boldsymbol{w}})&=0,
		\quad \forall k \in I_{\mathcal{R}_{j}}^{a}(\hat{\boldsymbol{w}})\backslash \{r_{j}\},
		\quad \forall j \in \bar{I}_{\mathcal{R}} \\
		D_{\ell_{i}}^{k}(\hat{\boldsymbol{w}})-D_{\ell_{i}}^{l_{i}}(\hat{\boldsymbol{w}})&=0, 
		\quad \forall k \in I_{\ell_{i}}^{a}(\hat{\boldsymbol{w}})\backslash \{l_{i}\},
		\quad \forall i \in \bar{I}_{\ell}\\
		\quad \sum_{k \in I_{\mathcal{R}_{j}}^{a}(\hat{\boldsymbol{w}})} \hat{t}_{\mathcal{R}_{j}}^{k}(\hat{\boldsymbol{w}})-1&=0,  \quad \hat{t}_{\mathcal{R}_{j}}^{k}(\hat{\boldsymbol{w}}) \geq 0, 
		\quad 1 \leq j \leq m \\
		\sum_{k \in I_{\ell_{i}}^{a}(\hat{\boldsymbol{w}})} \hat{t}_{\ell_{i}}^{k}(\hat{\boldsymbol{w}})-1&=0,  \quad \hat{t}_{\ell_{i}}^{k}(\hat{\boldsymbol{w}}) \geq 0, 
		\quad 1\leq i \leq n ,
	\end{aligned}
\end{equation}
where $r_{j},l_{i}$ is randomly selected from $I_{\mathcal{R}_{j}}^{a},I_{\ell_{i}}^{a}$ and being fixed. The second and third equations in (\ref{KKT_}) describe the active sets while the last two equations describe the subgradients.  When the partial optimum is on the smooth part, we denote the left side of equations (\ref{KKT_}) to be a $\mathbb C^{1}$ function $\mathcal{F}$, thus revealing that the solution path lies on the smooth manifold $\mathcal{F}\left(\boldsymbol{w}, \lambda, \boldsymbol{t}_{\mathcal{R}}, \boldsymbol{t}_{\ell}\right)=\mathbf{0}$. By the time it comes across the kink\footnote{$\hat{\boldsymbol{t}}_{\bar{I}_\mathcal{R}},\hat{\boldsymbol{t}}_{\bar{I}_{\ell}}$ hit the restriction bound in Lemma \ref{subgd} or $I_{\ell},\mathcal{P}$ are violated so that the entire structure changes.}, we need to 
refresh the index partitions and update (\ref{KKT_}) to run next segment of path. WLOG, we postulate that the initial point is non-degenerate (\emph{i.e.}, the  $\boldsymbol{J}_{\boldsymbol{w}, \boldsymbol{t}_{\mathcal{R}}, \boldsymbol{t}_{\ell}}$ is invertible). By directly applying the implicit function theorem, the existence and uniqueness of a local $\mathbb{C}^{1}$ solution path $\left(\hat{\boldsymbol{w}},\hat{\boldsymbol{t}}_{\mathcal{R}},\hat{\boldsymbol{t}}_{\mathcal{\ell}}\right)$ can be established over here. Drawing from the 
theory of differential geometry gives another intuitive understanding of age-path, which tells that the first equation in (\ref{KKT_}) indeed uses an analogue moving frame \cite{o2006elementary} to represent a smooth curve that consists of the smooth structure. 
\begin{theorem}\label{theorem}
	Given a partial optimum $\left(\hat{\boldsymbol{w}},\boldsymbol{v}^{*}\left(\hat{\boldsymbol{w}},\lambda\right)\right)$ at $\lambda_{0}$, $\hat{\boldsymbol{t}}_{\mathcal{R}}, \hat{\boldsymbol{t}}_{\ell}$ in (\ref{KKT_}) can be solved from $\mathcal{F}\left(\hat{\boldsymbol{w}}, \lambda_{0}, \hat{\boldsymbol{t}}_{\mathcal{R}}, \hat{\boldsymbol{t}}_{\ell}\right)=\mathbf{0}$. If the Jacobian $\boldsymbol{J}_{\boldsymbol{w}, \boldsymbol{t}_{\mathcal{R}}, \boldsymbol{t}_{\ell}}$ is invertible at $\left(\hat{\boldsymbol{w}}, \hat{\boldsymbol{t}}_{\mathcal{R}}, \hat{\boldsymbol{t}}_{\ell}\right)$, then in an open neighborhood of $\lambda_{0}$, $\left(\hat{\boldsymbol{w}}, \hat{\boldsymbol{t}}_{\mathcal{R}}, \hat{\boldsymbol{t}}_{\ell}\right)$ is a $\mathbb{C}^{1}$ function w.r.t. $\lambda$ and fits the ODEs
	\begin{equation}\label{mainODE}
		\frac{d\left(\begin{array}{c}
				\hat{\boldsymbol{w}} \\
				\hat{\boldsymbol{t}_{\mathcal{R}}} \\
				\hat{\boldsymbol{t}_{\ell}}
			\end{array}\right)}{\rule{0pt}{10pt}d \lambda}=-\boldsymbol{J}_{\boldsymbol{w}, \boldsymbol{t}_{\mathcal{R}}, \boldsymbol{t}_{\ell}}^{-1} \cdot \boldsymbol{J}_{\lambda},
	\end{equation}
	in which the explicit expressions of $\boldsymbol{J}_{\boldsymbol{w}, \boldsymbol{t}_{\mathcal{R}}, \boldsymbol{t}_{\ell}}^{-1},\boldsymbol{J}_{\lambda}$ are listed in Appendix \ref{details}.
\end{theorem}
\begin{cor}\label{cor}
	If all the functions are smooth in a neighborhood of the initial point, then (\ref{mainODE}) can be simplified as $d\hat{\boldsymbol{w}}/d\lambda=-\boldsymbol{J}_{\boldsymbol{w}}^{-1}\cdot\boldsymbol{J}_{\lambda}.$
\end{cor}
\begin{remark}
Our supplement parts in Appendix \ref{details} present additional discussions.
\end{remark}

\subsection{Critical Points}
By solving the initial value problem (\ref{mainODE}) numerically with ODE solvers, the solution path regarding to $\lambda$ can be computed swiftly before any of $\mathcal{P},I_{\mathcal{R}}$ or $I_{\ell}$ changes. We denote such point where the set changes a \emph{critical point}, which can be divided into \emph{turning point} or \emph{jump point} on the basis of path's property at that point. To be more specific, the age-path is discontinuous at a jump point, while being continuous but non-differentiable at the turning point. This is also verified by Figure \ref{fig:path} and large quantity of numerical experiments.
At turning points, the operation of the algorithm is to update $\mathcal{P},I_{\mathcal{R}}, I_{\ell}$ according to index violator(s) and move on to the next segment. At jump points, path is no longer continuous and warm-start\footnote{Reuse previous solutions. The subsequent calls to fit the model will not re-initialise parameters.} can be utilized to speed up the training procedure. The total number of critical points on the solution path is estimated at approximately $\mathcal{O}(|\mathcal{D}\cup\mathcal{M}|)$\footnote{Precisely speaking, it's related to interval length of $\lambda$, the nature of objective and the distribution of data.}. Consequently, we present a heuristic trick to figure out the type of a critical point with complexity $\mathcal{O}(d)$, so as to avoid excessive restarts. As a matter of fact, the solutions returned by the numerical ODE solver is continuous with the fixed set, despite it may actually passes a jump point. In this circumstance, the solutions returned by ODEs have deviated from the ground truth partial optimum. Hence it's convenient that we can detect KKT conditions to monitor this behavior. This approach enjoys higher efficiency than detecting the partition conditions themselves, especially when the set partition is extraordinarily complex.

\subsection{\method Algorithm}

\begin{algorithm}[t]
	\caption{~{G}eneralized {Ag}e-path {A}lgorithm (\method)}
	\label{apspl}
	\textbf{Input}: Initial solution $\hat{\boldsymbol w}|_{\lambda_t=\lambda_{min}}$, $X$, $y$, $\lambda_{min}$ and $\lambda_{max}$.\\
	\textbf{Output}: Age-Path $\hat{\boldsymbol w}\left( \lambda\right) $ on $\left[ \lambda_{min}, \lambda_{max}\right]$.\\\vspace{-12pt}
	\begin{algorithmic}[1] 
		\STATE $\lambda_t\gets\lambda_{min}$, set $\mathcal{P}, I_{\mathcal{R}}, I_{\ell}$ according to $\hat{\boldsymbol w}|_{\lambda_t}$.
		\WHILE{$\lambda_t\leq\lambda_{max}$  }
		\STATE Solve (\ref{mainODE}) and examine partition $\mathcal{P},I_{\mathcal{R}}, I_{\ell}$ simultaneously.
		\IF{Partition $\mathcal{P},I_{\mathcal{R}}, I_{\ell}$ was not met}
		\STATE Update $\mathcal{P},I_{\mathcal{R}}, I_{\ell}$ according to index violator(s).
		\STATE Solve (\ref{mainODE}) with updated $\mathcal{P},I_{\mathcal{R}}, I_{\ell}$.
		\IF {KKT conditions are not met}
		\STATE Warm start at $\lambda_t+\delta$ (for a small $\delta>0$).
		\ENDIF
		\ENDIF
		\ENDWHILE
	\end{algorithmic}
\end{algorithm}
There has been extensive research in applied mathematics on numerical methods for solving ODEs, 
where the solver could automatically determine the step size of $\lambda$ when solving (\ref{mainODE}). In the tracking process, we compute the solutions with regard to $\lambda$. After detecting a new critical point, we need to reset $\mathcal{P},I_{\mathcal{R}},I_{\ell}$ at turning point while warm-start is required for jump point. The above procedure is repeated until we traverse the entire interval $\left[ \lambda_{min}, \lambda_{max}\right] $. We show the detailed procedure in Algorithm \ref{apspl}. The main computational burden occurs in 
solving $\boldsymbol{J}^{-1}$ in (\ref{mainODE}) with an approximate complexity $\mathcal{O}(p^{3})$ in general, where $p$ denotes the dimension of $\boldsymbol{J}$. Further promotion can be made via decomposition or utilizing the sparse representation of $\boldsymbol{J}$ on specific learning problems.

\section{Practical Guides}\label{guides}
In this section, we provide practical guides of using the \method to solve two important learning problems, \emph{i.e.}, classic SVM and Lasso. The detailed steps of algorithms are displayed in Appendix \ref{algs}.
\subsection{Support Vector Machines}\label{svmGuide}
Support vector machine (SVM) \cite{vapnik1999nature} has attracted much attention from researchers in the areas of bioinformatics, computer vision and pattern recognition. Given the dataset $X$ and label $\boldsymbol{y}$, we focus on the classic support vector classification as 
\begin{equation}\label{svm}
	\min_{\boldsymbol{w},b} \frac{1}{2}\|\boldsymbol{w}\|_{\mathcal{H}}^{2}+\sum_{i=1}^{n}C\max\left\lbrace  0,~1-y_{i}(\langle \phi(x_{i}),\boldsymbol{w}\rangle +b)\right\rbrace,
\end{equation}

where $\mathcal{H}$ is the reproducing kernel Hilbert space (RKHS) with the inner product $\langle \cdot \rangle$ and corresponding kernel function $\phi$. {Seeing that (\ref{KKT_}) still holds in infinite dimensional $\mathcal{H}$, the above analyses can be directly applied here. We also utilize the \emph{kernel trick} \cite{scholkopf2000kernel} to  avoid involving the explicit expression of $\phi$. In consistent with the framework, we have $~\ell_{i}=C\max\left\lbrace 0,g_{i}\right\rbrace$ and $g_{i}=1-y_{i}(\langle \phi(x_{i}),\boldsymbol{w}\rangle +b).$ The $I_{\ell}$ and $\mathcal{P}$ are determined by $\boldsymbol{g}$, thus we merely need to refine the division of $\mathcal{E}$ as $\mathcal{E}_{N}=\{i\in\mathcal{E}:g_{i}<0\},\mathcal{E}_{Z}=\{i\in\mathcal{E}:g_{i}=0\}$ and $ \mathcal{E}_{P}=\{i\in\mathcal{E}:g_{i}>0\}$, which gives $I_{\ell}=\mathcal{E}_{N}\cup\mathcal{E}_{P}\cup\mathcal{D}(\cup\mathcal{M})$. Afterwards, with some simplifications and denoting $\hat{\boldsymbol{\alpha}}=C\boldsymbol{v}^{*}\odot \hat{\boldsymbol{t}},$ we can obtain a simplified version of (\ref{KKT_}), from where the age-path can be equivalently calculated w.r.t. optimal $(\hat{\boldsymbol{\alpha}}, \hat{b})$. }

\begin{proposition}\label{propSVM}
	When $\boldsymbol{\alpha},b$ indicate a partial optimum, the dynamics of optimal $\boldsymbol{\alpha},b$ in (\ref{svm}) w.r.t. $\lambda$ for the linear and mixture SP-regularizer are shown as\footnote{Notations such as $\boldsymbol{\ell}_{\mathcal{M}}^{-\frac{3}{2}}$ for vectors represent the element-wise operation in this section.} 
	\begin{equation}\label{SVMode}
		\frac{d\left(\begin{array}{c}
				\boldsymbol{\alpha}_{\mathcal{E}_{Z}} \\
				\boldsymbol{\alpha}_{\mathcal{E}_{P}} \\
				b
			\end{array}\right)}{\rule{0pt}{10pt}d\lambda} =
		\left(\begin{array}{ccc}
		-\boldsymbol{y}_{\mathcal{E}_{Z}}^{T} & -\boldsymbol{y}_{\mathcal{E}_{P}}^{T} & 0\\
			Q_{\mathcal{E}_{Z}\mathcal{E}_{Z}} & Q_{\mathcal{E}_{Z}\mathcal{E}_{P}} & \boldsymbol{y}_{\mathcal{E}_{Z}} \\
			\frac{C^{2}}{\lambda} Q_{\mathcal{E}_{P}\mathcal{E}_{Z}} & \frac{C^{2}}{\lambda} Q_{\mathcal{E}_{P}\mathcal{E}_{P}}-I_{\mathcal{E}_{P}\mathcal{E}_{P}} & \frac{C^{2}}{\lambda} \boldsymbol{y}_{\mathcal{E}_{P}} 
			
		\end{array}\right)^{-1}
		\left(\begin{array}{c}
		    0\\
			\boldsymbol{0}_{\mathcal{E}_{Z}} \\
			-\frac{C}{\lambda^{2}}\boldsymbol{\ell}_{\mathcal{E}_{P}} 
			
		\end{array}\right),
	\end{equation}
	\vspace{-15pt}
	\begin{equation}\label{SVMode1}
		\frac{d\left(\begin{array}{c}
				\boldsymbol{\alpha}_{\mathcal{E}_{Z}} \\
				\boldsymbol{\alpha}_{\mathcal{M}} \\
				b
			\end{array}\right)}{\rule{0pt}{10pt}d\lambda} =
		\left(\begin{array}{ccc}
		    -\boldsymbol{y}_{\mathcal{E}_{Z}}^{T} & -\boldsymbol{y}_{\mathcal{M}}^{T} & 0\\
			Q_{\mathcal{E}_{Z}\mathcal{E}_{Z}} & Q_{\mathcal{E}_{Z}\mathcal{M}} & \boldsymbol{y}_{\mathcal{E}_{Z}} \\
			\frac{C^{2}\gamma}{2}\tilde{Q}_{\mathcal{M}\mathcal{E}_{Z}} & \frac{C^{2}\gamma}{2} \tilde{Q}_{\mathcal{M}\mathcal{M}}-I_{\mathcal{M}\mathcal{M}} & \frac{C^{2}\gamma}{2}\tilde{\boldsymbol{y}}_{\mathcal{M}} 
			
		\end{array}\right)^{-1}
		\left(\begin{array}{c}
		    0\\
			\boldsymbol{0}_{\mathcal{E}_{Z}} \\
			-\frac{C\gamma}{\lambda^{2}}\boldsymbol{1}_{\mathcal{M}} 
			
		\end{array}\right),
	\end{equation}
	respectively, where $\tilde{Q}_{\mathcal{M}\mathcal{E}_{Z}}=Diag\{\boldsymbol{\ell}_{\mathcal{M}}^{-\frac{3}{2}}\}Q_{\mathcal{M}\mathcal{E}_{Z}},\tilde{Q}_{\mathcal{M}\mathcal{M}}=Diag\{\boldsymbol{\ell}_{\mathcal{M}}^{-\frac{3}{2}}\}Q_{\mathcal{M}\mathcal{M}},\tilde{\boldsymbol{y}}_{\mathcal{M}}=\boldsymbol{\ell}_{\mathcal{M}}^{-\frac{3}{2}}\odot \boldsymbol{y}_{\mathcal{M}}.$ 
	
	Other components are constant as $\boldsymbol{\alpha}_{\mathcal{E}_{N}}=\boldsymbol{0}_{\mathcal{E}_{N}},\boldsymbol{\alpha}_{\mathcal{D}}=\boldsymbol{0}_{\mathcal{D}}$. Only for mixture regularizer, $\boldsymbol{\alpha}_{\mathcal{E}_{P}}=\boldsymbol{1}_{\mathcal{E}_{P}}$.
\end{proposition}
\noindent\textbf{Critical Point.} We track $\boldsymbol{g}$ along the path. The critical point is sparked off by any set in $\mathcal{E}_{N},\mathcal{E}_{Z},\mathcal{E}_{P},\mathcal{D}(,\mathcal{M})$ changes.

\subsection{Lasso}\label{lassoGuide}
Lasso \cite{tibshirani1996regression} uses a sparsity based regularization term that can produce sparse solutions.

Given the dataset $X$ and label $\boldsymbol{y}$, the Lasso regression is stated as
\begin{equation}
	\label{lasso}
	\underset{\boldsymbol{w} \in \mathbb{R}^{d}}{\operatorname{min}}~ \frac{1}{2n}\|X \boldsymbol w- \boldsymbol{y}\|^{2}+\alpha \left\|\boldsymbol w\right\|_{1}.
\end{equation}

We expand $\|\boldsymbol{w}\|_{1}=\sum_{j=1}^{d}|w_{j}|$ and treat $|w_{j}|$ as $\mathcal{R}_{j}$ in (\ref{KKT_}), hence the $I_{\mathcal{R}}=\{1\leq j \leq d:w_{j}\neq0\}$. We denote the set of active or inactive functions (components) by $\mathcal{A}=I_{\mathcal{R}},\bar{\mathcal{A}}=\bar{I}_{\mathcal{R}}$, respectively. In view of the fact that $\partial{|w_{j}|}$ removes $\boldsymbol{t}_{\mathcal{R}_{j}}$ from the equations in (\ref{KKT_}) for $j\in \bar{\mathcal{A}}$, we only pay attention to the $\mathcal{A}$ part w.r.t. the $(\boldsymbol{w}_{\mathcal{A}},\lambda)$. The $\boldsymbol{\ell}$ is defined as  $\frac{1}{2n}{(}X \boldsymbol w- \boldsymbol{y}{)}^{2}$ in the following.
\begin{proposition}\label{proplasso}
	When $(\boldsymbol{w},\boldsymbol{v}^{*}(\boldsymbol{w},\lambda))$ is a partial optimum, the dynamics of optimal $\boldsymbol{w}$ in (\ref{lasso}) w.r.t. $\lambda$ for the linear and mixture SP-regularizer are described as
	\begin{equation}\label{lassoode}
		\frac{d\boldsymbol{w}_{\mathcal{A}}}{d\lambda}= -\frac{\sqrt{2n}}{\lambda^{2}}\left(X_{\mathcal{A}\mathcal{E}}^{T}Diag\left\{\boldsymbol{1}_{\mathcal{E}}-\frac{3}{\lambda}\boldsymbol{\boldsymbol{\ell}}_{\mathcal{E}}\right\}X_{\mathcal{E}\mathcal{A}}\right)^{-1}X_{\mathcal{A}\mathcal{E}}^{T}\boldsymbol{\ell}_{\mathcal{E}}^{\frac{3}{2}},
	\end{equation}
	\vspace{-6pt}
	\begin{equation}\label{lassoode2}
		\frac{d\boldsymbol{w}_{\mathcal{A}}}{d\lambda}= -\frac{\sqrt{2n}\gamma}{\lambda^{2}}\left(X_{\mathcal{A}\mathcal{E}\cup\mathcal{M}}^{T}
		\tilde{X}_{\mathcal{E}\cup\mathcal{M}\mathcal{A}}\right)^{-1}X_{\mathcal{A}\mathcal{E}\cup\mathcal{M}}^{T}\left(\begin{array}{c}
			\boldsymbol{0}_{\mathcal{E}}\\
			\boldsymbol{\ell}_{\mathcal{M}}
		\end{array}\right),
	\end{equation}
	respectively, where $\tilde{X}_{\mathcal{E}\cup\mathcal{M}\mathcal{A}}= \left(\begin{array}{c}
		X_{\mathcal{E}\mathcal{A}}\\
		-\frac{\gamma}{\lambda}X_{\mathcal{M}\mathcal{A}}
	\end{array}\right)$ and $\boldsymbol{w}_{\bar{\mathcal{A}}}=\boldsymbol{0}_{\bar{\mathcal{A}}}.$
\end{proposition}

\noindent\textbf{Critical Point.} The critical point is encountered when $\mathcal{A}$ or $\mathcal{P}$ changes. 

\section{Experimental Evaluation}\label{exp}
\begin{figure}[tb]
	\centering
	\begin{minipage}{0.57\textwidth}
		\begin{table}[H]
			\centering\scriptsize
			\begin{tabular}{lllll} \toprule
				Dataset & Source & Samples & Dimensions & Task\\ \midrule
				mfeat-pixel & UCI \cite{asuncion2007uci} & 2000 & 240 & \multirow{3}{*}{C}\\ 
				pendigits & UCI & 3498 & 16 \\
				
				hiva agnostic & OpenML & 4230 & 1620  \\ \midrule
				music  & OpenML \cite{vanschoren2014openml}  & 1060 & 117 & \multirow{6}{*}{R}\\ 
				cadata  & UCI & 20640 & 8 \\
				delta elevators  & OpenML & 9517 & 8 \\
				houses  & OpenML & 22600 & 8 \\
				ailerons  & OpenML & 13750 & 41 \\
				elevator  & OpenML & 16600 & 18 \\
				\bottomrule
			\end{tabular}
			\vspace{5pt}
			\caption{Datasets description in our experiments. The C=Classification, R=Regression.}
			\label{dataset}
		\end{table}
	\end{minipage}
	\hspace{5pt}
		\begin{minipage}{0.4\textwidth}
		\begin{figure}[H]
			\vspace{6pt}
			\centering
			\begin{adjustbox}{width=1\textwidth}
				\includegraphics{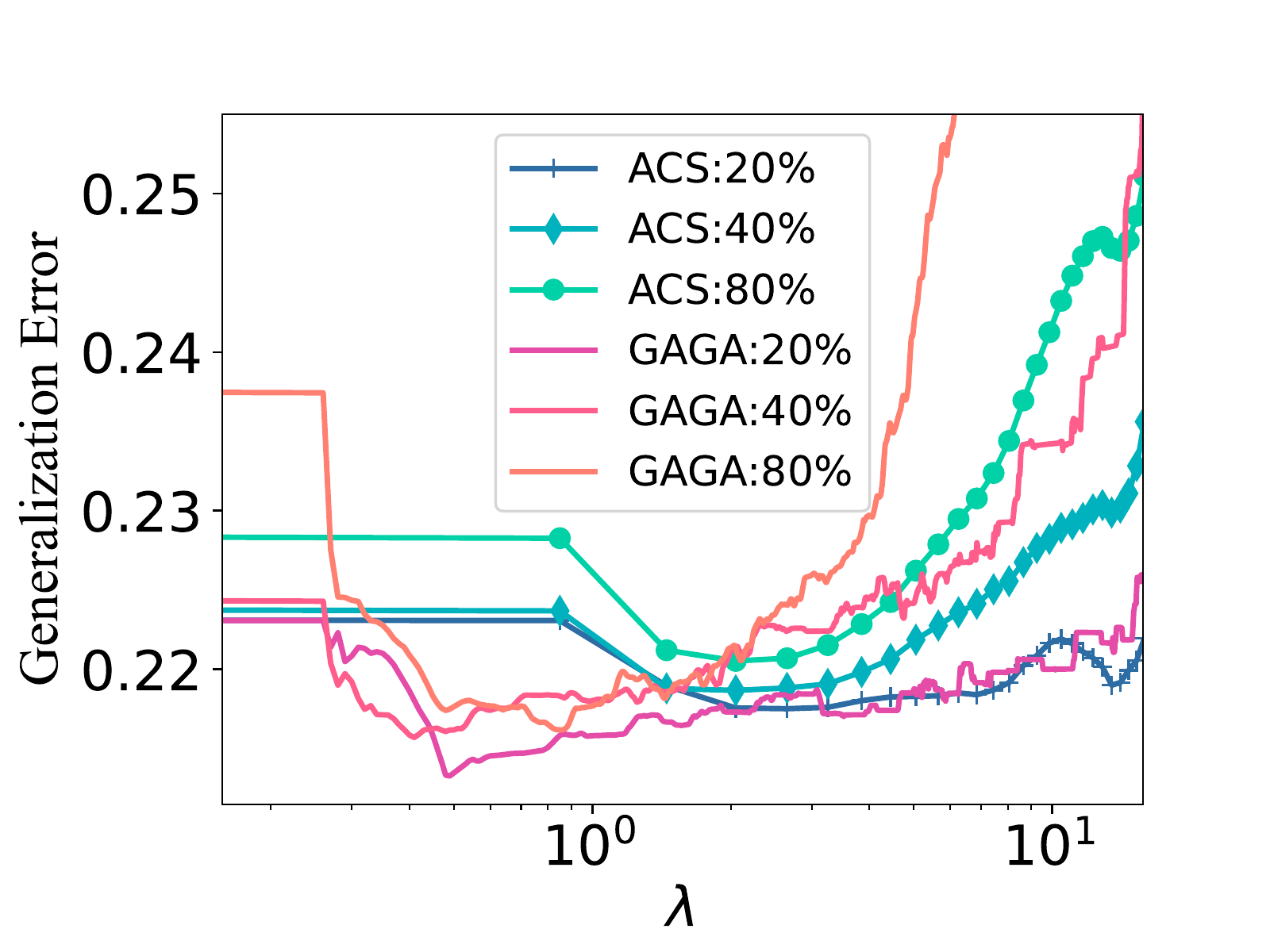}
			\end{adjustbox}
			\vspace{-2em}
			\caption{Robustness to noise. This figure shows the learning curve under different noise ratios, which confirms that the \method is more robust when in the setting of relatively high noise.}
			\label{fig:noisy}
		\end{figure}
	\end{minipage}
\end{figure}
We present the empirical results of the proposed \method on two tasks: SVM for binary classification and Lasso for robust regression in the noisy environment. The results demonstrate that our approach outperforms existing SPL implementations on both tasks.

\begin{table}[t]
	\centering\tiny
	\setlength{\tabcolsep}{7pt}
	\resizebox{\textwidth}{!}{
	\begin{tabular}{l cc c     ccc c c}
		\toprule
		\bf \multirow{2}{*}{\vspace{-5pt}Dataset} &  \multicolumn{2}{c}{\bf Parameter} && \multicolumn{3}{c}{\bf Competing Methods} & \bf Ours &\bf \multirow{2}{*}{\vspace{-5pt}Restarting times}\\
		\cmidrule{2-3}  \cmidrule {5-7}
		& $C,\gamma_{\kappa}$ & $\alpha$ &&  \it Original & \it ACS & \it MOSPL &  ~\method~   \\
		\midrule
		mfeat-pixel\dag    & 1.00, 0.50 & --         %
		&& ~0.959\stderr{0.037}~ & ~0.976\stderr{0.015}~ & ~0.978\stderr{0.021}~ & ~\textbf{0.986}\stderr{0.016}~ & 23\\
		mfeat-pixel\ddag     & 1.00, 0.50 & --         %
		&& ~0.945\stderr{0.025}~ & ~0.947\stderr{0.031}~ & ~0.960\stderr{0.027}~ & ~\textbf{0.983}\stderr{0.013}~ & 25\\
		hiva agnostic\dag  & 1.00, 0.50 &   --                  %
		&& ~0.868\stderr{0.027}~ & 0.941\stderr{0.009} & ~0.946\stderr{0.0137}~ & ~\textbf{0.960}\stderr{0.004}~ & 8\\
		pendigits\dag   & 1.00, 1.00 & --%
		&& ~0.924\stderr{0.069}~ & ~0.960\stderr{0.005}~ & ~0.962\stderr{0.048}~ & ~\textbf{0.971}\stderr{0.046}~ & 10\\
		pendigits\ddag     & 1.00, 0.20            & --      %
		&& ~0.931\stderr{0.045}~  & ~0.942\stderr{0.089}~ & ~0.940\stderr{0.088}~ & ~\textbf{0.944}\stderr{0.089}~ & 8\\ \midrule
		elevator\dag & --                & 2e-3                         %
		&& ~0.146\stderr{0.011}~ & ~0.144\stderr{0.012}~ & ~0.144\stderr{0.020}~ & ~\textbf{0.143}\stderr{0.012}~ & 3\\
		ailerons\dag       & --                & 6e-3                         %
		&& ~0.674\stderr{0.071}~    & ~0.492\stderr{0.006}~ & ~0.491\stderr{0.041}~ & ~\textbf{0.489}\stderr{0.009}~ & 16\\
		music\dag    & --              & 5e-3                      %
		&& ~0.325\stderr{0.009}~    & ~0.219\stderr{0.018}~ & ~0.215\stderr{0.012}~ & ~\textbf{0.206}\stderr{0.013}~ & 123\\
		delta elevators\dag    & --              & 5e-3                      %
		&& ~0.783\stderr{0.153}~    & ~0.724\stderr{0.138}~ & ~0.679\stderr{0.057}~ & ~\textbf{0.634}\stderr{0.184}~ & 4\\
		houses\dag      & --  & 5e-3 
		&& ~0.213\stderr{0.013}~ & ~0.209\stderr{0.010}~ & ~0.205\stderr{0.231}~ & ~\textbf{0.201}\stderr{0.146}~ & 4\\
		\bottomrule \\
	\end{tabular}
	}
	\caption{Average results with the standard deviation in 20 runs on different datasets using the \emph{linear SP-regularizer}. The top results in each row are in boldface. }
	\label{tab:acc1}
\end{table}
\begin{table}[!ht]
	\centering\tiny
	\setlength{\tabcolsep}{7pt}
	\resizebox{\textwidth}{!}{
	\begin{tabular}{l ccc c     ccc c c}
		\toprule
		\bf \multirow{2}{*}{\vspace{-5pt}Dataset} &  \multicolumn{3}{c}{\bf Parameter} && \multicolumn{3}{c}{\bf Competing Methods} & \bf Ours &\bf \multirow{2}{*}{\vspace{-5pt}Restarting times}\\
		\cmidrule{2-4}  \cmidrule {6-8}
		&$\gamma$& $C$,$\gamma_{\kappa}$ & $\alpha$ &&  \it Original & \it ACS & \it MOSPL &  ~\method~   \\
		\midrule
		mfeat-pixel\dag   & 0.20 & 1.00, 1.00 & --%
		&& ~0.959\stderr{0.037}~ & ~0.963\stderr{0.038}~ & ~0.968\stderr{0.037}~ & ~\textbf{0.973}\stderr{0.040}~ & 12\\
		mfeat-pixel\ddag    & 0.50 & 0.20, 1.00      & --      %
		&& ~0.945\stderr{0.025}~    & ~0.962\stderr{0.024}~ & ~0.970\stderr{0.027}~ & ~\textbf{0.977}\stderr{0.015}~ & 10\\
		hiva agnostic\dag  & 0.50 & 1.00, 1.00 &   --                  %
		&& ~0.868\stderr{0.027}~ & ~0.946\stderr{0.004}~ & ~0.949\stderr{0.019}~ & ~\textbf{0.957}\stderr{0.007}~ & 10\\
		pendigits\dag     & 0.50 & 2.00, 1.00 & --         %
		&& ~0.924\stderr{0.069}~ & ~0.956\stderr{0.062}~ & ~0.957\stderr{0.071}~ & ~\textbf{0.962}\stderr{0.083}~ & 32\\
		pendigits\ddag     & 0.20 & 1.00, 1.00 & --         %
		&& ~0.931\stderr{0.045}~ & ~0.940\stderr{0.088} ~ & ~0.942\stderr{0.089}~ & ~\textbf{0.944}\stderr{0.088}~ & 30\\\midrule
		
		cadata\dag    & 1.00 & --,--             & 5e-3                       %
		&& ~0.798\stderr{0.039}~    & ~0.782\stderr{0.042}~ & ~0.754\stderr{0.084}~ & ~\textbf{0.748}\stderr{0.010}~ & 13\\
		ailerons\dag    & 0.50 & --,--              & 5e-3                       %
		&& ~0.674\stderr{0.071}~    & ~0.452\stderr{0.057}~ & ~0.433\stderr{0.083}~ & ~\textbf{0.422}\stderr{0.090}~ & 14\\
		music\dag       & 0.50 & --,--                & 6e-3                         %
		&& ~0.325\stderr{0.009}~    & ~0.218\stderr{0.009}~ & ~0.216\stderr{0.021}~ & ~\textbf{0.213}\stderr{0.027}~ & 110\\
		delta elevators\dag & 0.50 & --,--             & 5e-3                         %
		&& ~0.783\stderr{0.153}~ & ~0.663\stderr{0.074}~ & ~0.650\stderr{0.029}~ & ~\textbf{0.595}\stderr{0.132}~ & 12\\
		houses\dag      & 0.50 & --,--  & 5e-3 
		&& ~0.213\stderr{0.013}~ & ~0.146\stderr{0.012}~ & ~0.144\stderr{0.027}~ & ~\textbf{0.142}\stderr{0.012}~ & 8\\
		\bottomrule \\
	\end{tabular}
	}
	\caption{Average results with the standard deviation in 20 runs on different datasets using the \emph{mixture SP-regularizer}. The top results in each row are in boldface.}
	\label{tab:acc2}
\end{table}
\begin{figure}[!t]
    \vspace{-1em}
	\centering
	\begin{adjustbox}{width=0.95\textwidth}
		\includegraphics{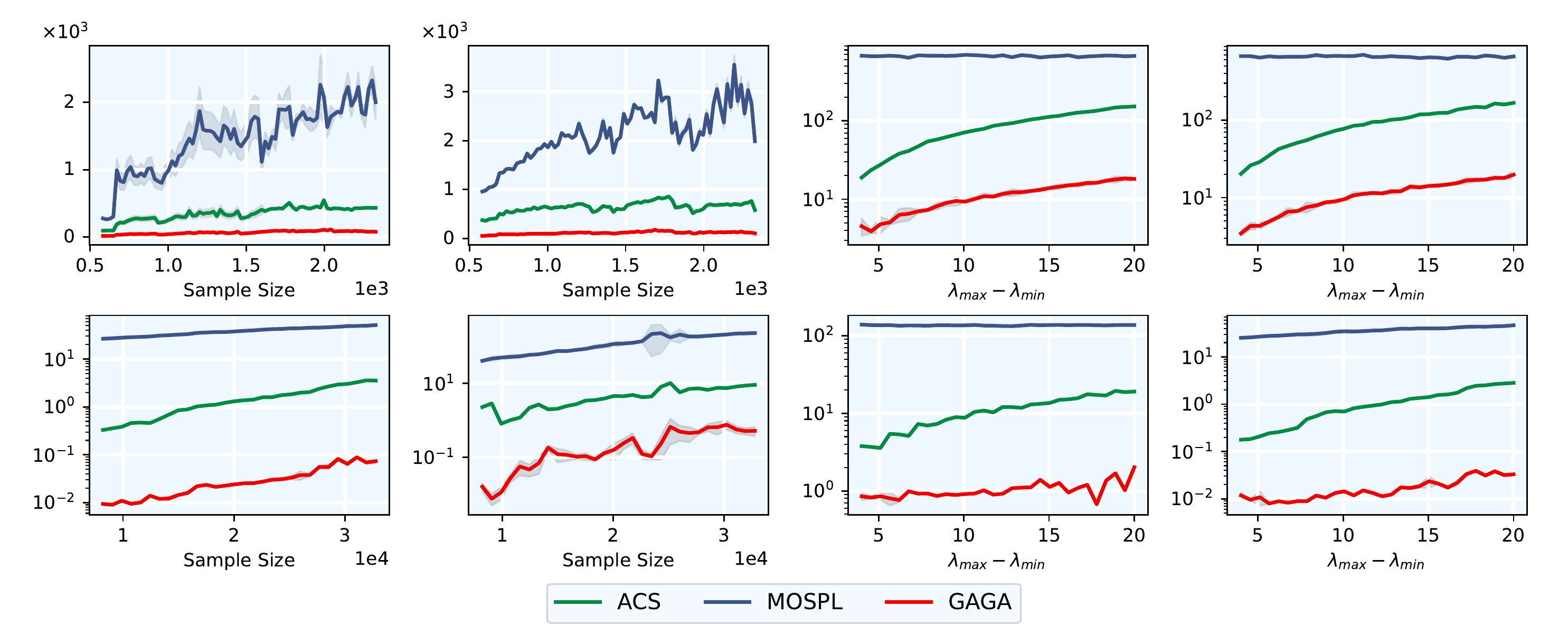}
	\end{adjustbox}
	\caption{The study of efficiency comparison. $y$-axis denotes the average running time (in seconds) with 20 runs. The interval $[\lambda_{min},\lambda_{max}]$ refers to the predefined searching area.}
	\label{fig:time}
\end{figure}
\paragraph{Baselines. }  \method is compared against three baseline methods: 1) \textbf{Original} learning model without SPL regime.  
2) \textbf{ACS} \cite{wendell1976minimization} performs sequential search of $\lambda$, which is the most commonly used algorithm in SPL implementations.
3) \textbf{MOSPL} \cite{li2016multi} is a state-of-the-art age-path approach that using the multi objective optimization, in which the solution is derived with the age parameter $\lambda$ implicitly. 

\paragraph{Datasets.}  The Table \ref{dataset} summarizes the datasets information. As universally known that SPL enjoys robustness in noisy environments, we impose 30\% of noises into the real-world datasets. In particular, we generate noises by turning normal samples into poisoning ones by flipping their labels \cite{Frenay2014classification, Ghosh2017robust} for classification tasks. For regression problem, noises are generated by the similar distribution of the training data as performed in \cite{Jagielski2018manipulating}.

\paragraph{Experiment Setting.}
\begin{wrapfigure}[]{R}{0.3\textwidth}
	\vspace{-2.2em}
	\begin{minipage}{0.3\textwidth}
		\begin{figure}[H]
			\centering
			\begin{adjustbox}{width=1\textwidth}
				\includegraphics{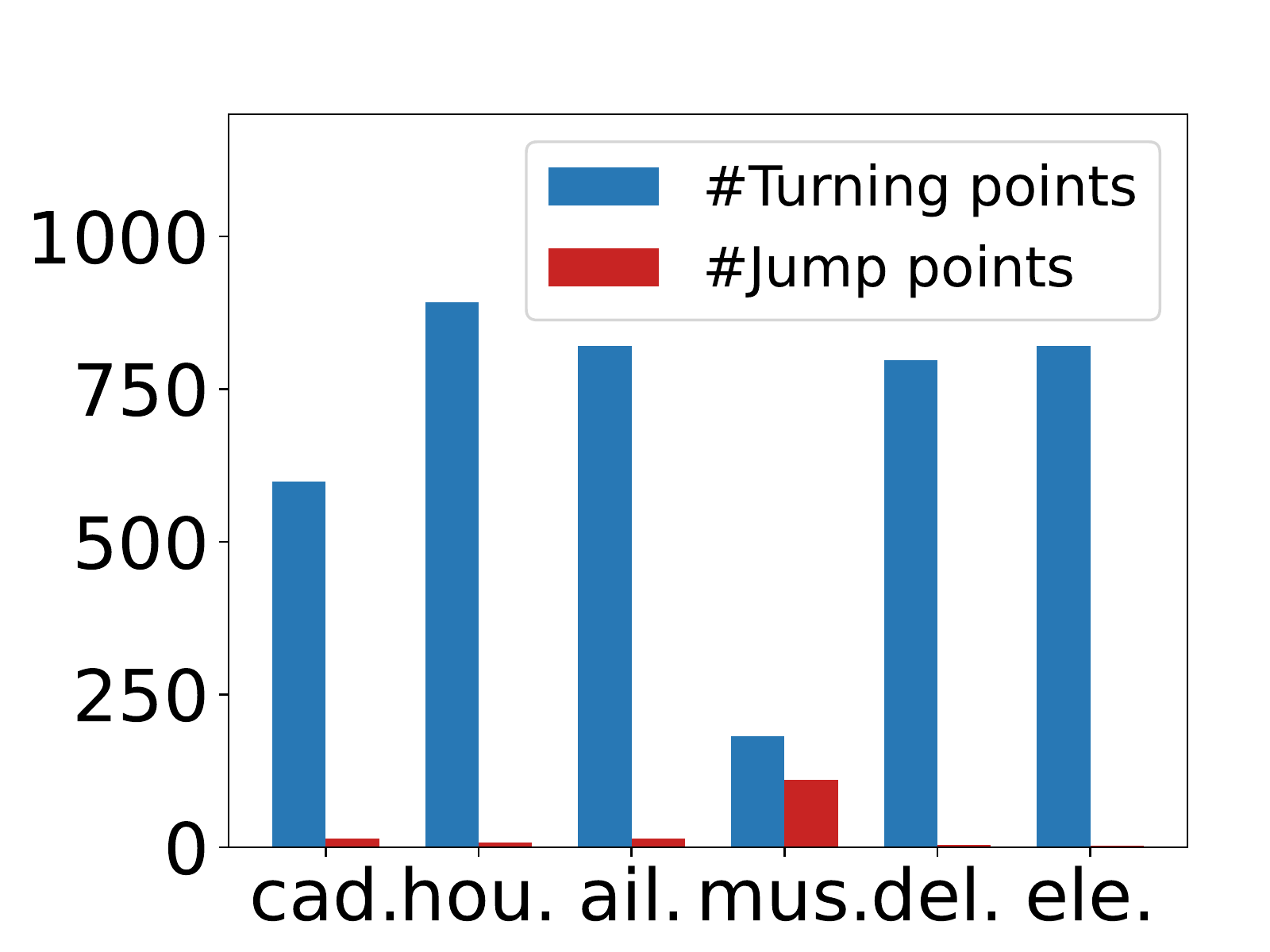}
			\end{adjustbox}
			\vspace{-10pt}
			\caption{Histogram illustrating the number of different types of critical points. Names of datasets are shortened into the first 3 letters.}
			\label{fig:points}
		\end{figure}
	\end{minipage}

\end{wrapfigure}
In experiments, we first verify the performance of \method and traditional ACS algorithm under different noise intensity to reflect the robustness of \method. We further study the generalization performance of \method with competing methods, so as to show its ability to select optimal model during the learning process. Meanwhile, we also evaluate the running efficiency between \method and existing SPL implementations in different settings, which examines the speedup of \method as well as its practicability. Finally, we count the number of restarts and different types of critical points when using \method, to investigate its ability to address critical points. For SVM, we use the Gaussian kernel $K(x_{1},x_{2})=\operatorname{exp}(-\gamma_{\kappa}\|x_{1}-x_{2}\|^{2})$. More details can be found in Appendix \ref{additional}.

\paragraph{Results.}
Figure \ref{fig:noisy} illustrates that conventional ACS fails to reduce the generalization error due to heavy noises fed to the model at a large age (\emph{i.e.}, overfits the dataset), while \method makes up for this shortcoming by selecting the optimal model that merely learns the trust-worthy samples during the continuous learning process. Table \ref{tab:acc1} and \ref{tab:acc2} demonstrate an overall performance enhancing in \method than competing algorithms. The `\dag' in tables denotes 30\% of artificial noise, while `\ddag' represents 20\%. Note that performances are measured by accuracy and generalization error for classification and regression, respectively. The results guarantee that \method outperforms the state-of-the-art approaches in SPL under different circumstances. Figure \ref{fig:time} shows that \method also enjoys a high computational efficiency by changing the sample size as well as the predefined age sequence, emphasizing the potentials of utilizing \method in practice. The number of different types of critical points on some datasets is given in Figure \ref{fig:points}. Corresponding restarting times can be found in Table \ref{tab:acc1} and \ref{tab:acc2}, hence indicate that \method is capable of identifying different types of critical points and uses the heuristic trick to avoids restarts at massive turning points.

\paragraph{Additional Experiments in Appendix \ref{additional}.} {We further demonstrate the ability of \method to address the  relatively large sample size and present more histograms. In addition, we apply \method to the logistic regression \cite{vapnik1999nature} for classification. 
We also verify that conventional ACS indeed tracks an approximation path of partial optimum in experiments, which provides a more in-depth understanding towards SPL and the performance promotion brought by \method. We also conduct a comparative study to the state-of-the-art robust model for SVMs \cite{chen2015robust,yang2014robust,biggio2011support} and Lasso \cite{nasrabadi2011robust} besides the SPL domain.}

\section{Conclusion}\label{conclu}
In this paper, we connect the SPL paradigm to the partial optimum for the first time. Using this idea, we propose the first \emph{exact} age-path algorithm able to tackle \emph{general} SP-regularizer, namely \method. 
Experimental results demonstrate \method outperforms traditional SPL paradigm and the state-of-the-art age-path approach in many aspects, especially in the highly noisy environment. We further build the relationship between our framework and existing theoretical analyses on SPL regime, which provides more in-depth understanding towards the principle behind SPL.

\bibliography{Ref}
\bibliographystyle{unsrt}\clearpage
\section*{Checklist}

\begin{enumerate}
	
	\item For all authors...
	\begin{enumerate}
		\item Do the main claims made in the abstract and introduction accurately reflect the paper's contributions and scope?
		\answerYes{}
		\item Did you describe the limitations of your work?
		\answerYes{See Appendix \ref{limit_impact}.}
		\item Did you discuss any potential negative societal impacts of your work?
		\answerNA{}
		\item Have you read the ethics review guidelines and ensured that your paper conforms to them?
		\answerYes{}
	\end{enumerate}

	\item If you are including theoretical results...
	\begin{enumerate}
		\item Did you state the full set of assumptions of all theoretical results?
		\answerYes{}
		\item Did you include complete proofs of all theoretical results?
		\answerYes{}
	\end{enumerate}

	\item If you ran experiments...
	\begin{enumerate}
		\item Did you include the code, data, and instructions needed to reproduce the main experimental results (either in the supplemental material or as a URL)?
		\answerYes{}
		\item Did you specify all the training details (e.g., data splits, hyperparameters, how they were chosen)?
		\answerYes{}
		\item Did you report error bars (e.g., with respect to the random seed after running experiments multiple times)?
		\answerYes{}
		\item Did you include the total amount of compute and the type of resources used (e.g., type of GPUs, internal cluster, or cloud provider)?
		\answerYes{ See Appendix \ref{detail}.}
	\end{enumerate}

	\item If you are using existing assets (e.g., code, data, models) or curating/releasing new assets...
	\begin{enumerate}
  \item If your work uses existing assets, did you cite the creators?
    \answerYes{See Table \ref{dataset}.}
  \item Did you mention the license of the assets?
    \answerYes{Licenses are available referring to the provided links.}
		\item Did you include any new assets either in the supplemental material or as a URL?
		\answerYes{}
		\item Did you discuss whether and how consent was obtained from people whose data you're using/curating?
		\answerNA{}
		\item Did you discuss whether the data you are using/curating contains personally identifiable information or offensive content?
		\answerNA{}
	\end{enumerate}

\item If you used crowdsourcing or conducted research with human subjects...
\begin{enumerate}
  \item Did you include the full text of instructions given to participants and screenshots, if applicable?
    \answerNA{This paper does not use crowdsourcing.}
  \item Did you describe any potential participant risks, with links to Institutional Review Board (IRB) approvals, if applicable?
    \answerNA{This paper does not use crowdsourcing.}
  \item Did you include the estimated hourly wage paid to participants and the total amount spent on participant compensation?
    \answerNA{This paper does not use crowdsourcing.}
\end{enumerate}

\end{enumerate}

\clearpage
\appendix
{\Large \textbf{Appendix}}

\section{Supplementary Notes for Section \ref{tracking}}
\label{details}
In this section, we present some additional discussions and theoretical results for Section \ref{tracking}.
\subsection{Explicit Expression}
The goal of this section is to obtain precise and explicit expressions of the Jacobian matrices in (\ref{mainODE}).

Firstly, the $\boldsymbol{J}_{\boldsymbol{w}, \boldsymbol{t}_{\mathcal{R}}, \boldsymbol{t}_{\ell}}^{-1}$ has the form of
\begin{equation}\label{jacobian}
   \boldsymbol{J}_{\boldsymbol{w}, \boldsymbol{t}_{\mathcal{R}}, \boldsymbol{t}_{\ell}}^{-1}=  \left(\begin{array}{cc}
     \tilde{\boldsymbol{F}}_{\boldsymbol{w}}& \tilde{\boldsymbol{F}}_{\boldsymbol{t}} \\
    \left(\begin{array}{c}
        \tilde{\boldsymbol{D}}    \\
         \boldsymbol{O}
    \end{array}\right) & 
    \left(\begin{array}{c}
        \boldsymbol{O}     \\
         \tilde{\boldsymbol{I}}
    \end{array}\right)
\end{array}\right)^{-1},
\end{equation}
in which
\begin{align*}
    &\tilde{\boldsymbol{F}}_{\boldsymbol{w}}=	\sum_{j=1}^{m} \sum_{k \in I_{\mathcal{R}_{j}}^{a}(\hat{\boldsymbol{w}})} \alpha_{j} \hat{t}_{\mathcal{R}_{j}}^{k}(\hat{\boldsymbol{w}}) \nabla^{2} D_{\mathcal{R}_{j}}^{k}(\hat{\boldsymbol{w}})+\sum_{J \in \mathcal{P}} \sum_{i \in J} \sum_{k \in I_{\ell_{i}}^{a}(\hat{\boldsymbol{w}})} v^{*}_i\left(\ell_{i}(\hat{\boldsymbol{w}}), \lambda\right) \hat{t}_{\ell_{i}}^{k}(\hat{\boldsymbol{w}}) \nabla^{2} D_{\ell_{i}}^{k}(\hat{\boldsymbol{w}}),\\
    &\tilde{\boldsymbol{F}}_{\boldsymbol{t}}=\left(
     \begin{array}{cc}
     \mdoubleplus_{1\leq j\leq m,k\in I_{\mathcal{R}_{j}}^{a}}\nabla D_{\mathcal{R}_{j}}^{k}(\hat{\boldsymbol{w}})& 
      \mdoubleplus_{\mathcal{J}\in\mathcal{P},j\in\mathcal{J},k \in I_{\ell_{i}}^{a}}\alpha_{j} \nabla D_{\ell_{j}}^{k}(\hat{\boldsymbol{w}})
     \end{array}
     \right),\\
     &\tilde{\boldsymbol{D}}=
      \left(\begin{array}{c}
      \left(\mdoubleplus_{1\leq j\leq m,k\in I_{\mathcal{R}_{j}}^{a}\backslash\{r_{j}\}}\nabla D_{\mathcal{R}_{j}}^{k}(\hat{\boldsymbol{w}})-\nabla D_{\mathcal{R}_{j}}^{r_{j}}(\hat{\boldsymbol{w}})\right)^{T} \\
      \left(\mdoubleplus_{\mathcal{J}\in\mathcal{P},j\in\mathcal{J},k \in I_{\ell_{i}}^{a}\backslash\{l_{i}\}}\nabla D_{\ell_{i}}^{k}(\hat{\boldsymbol{w}})-\nabla
      D_{\ell_{i}}^{l_{i}}(\hat{\boldsymbol{w}})\right)^{T}
      \end{array}
      \right),\\
     & \tilde{\boldsymbol{I}}=\left(
    \begin{array}{c}
    \left(\mdoubleplus_{i=1}^{n}\left(\begin{array}{c}
        \boldsymbol{1}_{I_{\mathcal{R}_{j}}^{a}}   \\
        \boldsymbol{0}_{\bar{I}_{\mathcal{R}_{j}}^{a}}
    \end{array}\right)\right)^{T}  
    \\
    \left(\mdoubleplus_{\mathcal{J}\in\mathcal{P},j\in\mathcal{J}}\left(\begin{array}{c}
        \boldsymbol{1}_{I_{\ell_{i}}^{a}}   \\
        \boldsymbol{0}_{\bar{I}_{\ell_{i}}^{a}}
    \end{array}\right)
    \right)^{T}
\end{array}
\right).
\end{align*}
The symbol $\mdoubleplus$ denotes column matrix concatenation, following the convention in Haskell Language.

Similarly, $\boldsymbol{J}_{\lambda}$ is given by
\begin{equation}
    \boldsymbol{J}_{\lambda} = \left(\begin{array}{c}
         \tilde{\boldsymbol{F}}_{\lambda}  \\
         \boldsymbol{0} 
    \end{array}\right),
\end{equation}
where $$\tilde{\boldsymbol{F}}_{\lambda} =  \sum_{J \in \mathcal{P}} \sum_{i \in J} \sum_{k \in I_{\ell_{i}}^{a}(\hat{\boldsymbol{w}})} \dfrac{\partial v^{*}_i\left(\ell_{i}(\hat{\boldsymbol{w}}), \lambda\right)}{\partial \lambda} \hat{t}_{\ell_{i}}^{k}(\hat{\boldsymbol{w}}) \nabla D_{\ell_{i}}^{k}(\hat{\boldsymbol{w}}).
$$

As an outcome, (\ref{mainODE}) can be explicitly expressed via 
\begin{equation}\label{explicitMainOde}
	\frac{d\left(\begin{array}{c}
		\hat{\boldsymbol{w}} \\
		\hat{\boldsymbol{t}_{\mathcal{R}}} \\
		\hat{\boldsymbol{t}_{\ell}}
	\end{array}\right)}{\rule{0pt}{10pt}d\lambda} = \left(\begin{array}{cc}
     \tilde{\boldsymbol{F}}_{\boldsymbol{w}}& \tilde{\boldsymbol{F}}_{\boldsymbol{t}} \\
    \left(\begin{array}{c}
        \tilde{\boldsymbol{D}}    \\
         \boldsymbol{O}
    \end{array}\right) & 
    \left(\begin{array}{c}
        \boldsymbol{O}     \\
         \tilde{\boldsymbol{I}}
    \end{array}\right)
\end{array}\right)^{-1}
\left(\begin{array}{c}
         \tilde{\boldsymbol{F}}_{\lambda}  \\
         \boldsymbol{0} 
    \end{array}\right).
\end{equation}
\subsection{On the Numerical ODEs Solving}
\paragraph{Matrix Inversion.} Here we give some relative discussions to support the assumption in  Section \ref{path} that $\boldsymbol{J}_{\boldsymbol{w},\boldsymbol{t}_{\mathcal{R}},\boldsymbol{t}_{\boldsymbol{\ell}}}$ ($\boldsymbol{J}$ for short) is non-singular. First 

note that $\boldsymbol{J}$ is singular if and only if the value of its determinant $|\boldsymbol{J}|$ is zero, where $|\boldsymbol{J}|$ is indeed a polynomial w.r.t. the uncertain elements in $\boldsymbol{J}$. Denote the number of these unknown components as $q$, then the probability that $\boldsymbol{J}$ is singular can be somewhat equivalently seen as the measure of the hypersurface $\mathcal{S}=\{\boldsymbol{x} \in \mathbb{R}^{q}: |\boldsymbol{J}(\boldsymbol{x})|=0\}$ in $\mathbb{R}^{q}.$ For the polynomial function $|\boldsymbol{J}|,$ it's easy to prove that $\mathcal{S}$ is a zero-measured set in $\mathbb{R}^{q}$, which indicates the probability of $\boldsymbol{J}$ being non-invertible is zero. This result shows the non-singularity assumption fits the common situation in practice. Secondly, there are some general cases that $\boldsymbol{J}$ is guaranteed to be 
invertible. We refer some of these claims to \cite{gebken2021structure}. Moreover, during the extensive empirical studies, none of the singular $\boldsymbol{J}$ is observed, which 
again validates the rationality of the assumption. 
\paragraph{Robustness.}
In practice, we avoid directly computing the inverse of $\boldsymbol{J}$ by the consideration of robustness. Instead, we adopt the Moore-Penrose inverse \cite{Ben1974inverse} in implementation. The Moore-Penrose inverse exists for any matrix $X$ even if the matrix owns singularity, which guarantees our algorithm to be robust.
\paragraph{Complexity.} Our approach utilizes the singular value decomposition (SVD) \cite{Ben1974inverse} when solving the Moore-Penrose inverse in (\ref{explicitMainOde}), which is demonstrated to be a state-of-the-art technique via a computationally simple and precise way \cite{Horn1990matrix}. Consequently, the computational efficiency as well as the accuracy of our algorithm is guaranteed.
\paragraph{Stability.} Our algorithm uses ODE solvers from the LSODE package \cite{Radhakrishnan1994LSODE} when solving the initial value problem (\ref{mainODE}). The solver will automatically select a proper method to solve different initial value problems which guarantees the general performance of our algorithm. Especially, when the problem tends to be unstable, the solver adopts the backward differentiation formula (BDF) method \cite{Watt1972ode} to avoid extremely small step sizes while preserving the stability and accuracy of output solutions.

\section{Proofs}
\label{proofs}
In this section, we give complete proofs to all the theorems and properties stated in the main article.
\subsection{Proof of \textbf{Theorem} \ref{conver2}}

Prior to the proof, we first review the relevant background about latent SPL loss \cite{ma2018convergence}. Regarding the unconstrained learning problem (\ref{obj}), its latent SPL objective is defined as\footnote{Note that the objective function in (\ref{obj}) is indeed the same as what in \cite{ma2018convergence} despite minor differences in notations, so we keep the manner in this paper for the sake of consistency.}
\begin{equation}\label{implicit}
    G_{\lambda}(\boldsymbol{w}):= \sum_{j=1}^{m} \alpha_{j} \mathcal{R}_{j}(\boldsymbol{w})+\sum_{i=1}^{n} F_{\lambda}\left(\ell_{i}(\boldsymbol{w})\right),
\end{equation}
where $F_{\lambda}(\ell)=\int_{0}^{\ell} v_{\lambda}^{*}(\tau) ~d \tau$. 
\begin{theorem}\label{conver3}\cite{ma2018convergence}
	In the SPL objective (\ref{spl}), suppose $\ell$ is bounded below, $\boldsymbol{w} \mapsto$ $\ell(\cdot)$ is continuously differentiable, $v_{\lambda}^{*}(\cdot)$ is continuous, and $\sum_{j=1}^{m} \alpha_{j} \mathcal{R}_{j}$ is coercive and lower semi-continuous. Then for any initial parameter $\boldsymbol{w}^{0}$, every cluster point of the produced sequence $\left\{\boldsymbol{w}^{k}\right\}$, obtained by the ACS algorithm on solving (\ref{spl}), is a critical point of the implicit objective $G_{\lambda}$ (\ref{implicit}).
\end{theorem}

In Theorem \ref{conver2}, the relationship between the partial optimum of original SPL objective $\mathcal{L}$ and the critical point of implicit SPL objective $G_{\lambda}$ is constructed. Its proof is given as follows.
\begin{proof}
	On one hand,
	$$
	\begin{aligned}
		\boldsymbol{w}_{0} \text { is a critical point of } G_{\lambda} \Longleftrightarrow 0 \in \partial G_{\lambda}\left(\boldsymbol{w}_{0}\right) &=\partial \sum_{j=1}^{m} \alpha_{j} \mathcal{R}_{j}(\boldsymbol{w}_{0})+\sum_{i=1}^{N} \nabla F_{\lambda}\left(l_{i}(\boldsymbol{w}_{0})\right) \\
		&=\sum_{j=1}^{m}\alpha_{j}\partial \mathcal{R}_{j}(\boldsymbol{w}_{0}) +\sum_{i=1}^{N} v_{\lambda}^{*}\left(l_{i}(\boldsymbol{w}_{0})\right) \cdot \nabla l_{i}(\boldsymbol{w}_{0}).
	\end{aligned}
	$$
On the other hand, assuming $\left(\boldsymbol{w}_{0}, \boldsymbol{v}\right)$ is a partial optimum of the SPL objective $\mathcal{L}$, it's obvious that $v_{i}=v_{\lambda}^{*}\left(l_{i}(\boldsymbol{w}_{0})\right)$ (for $i=1 \ldots N$), where we write $\boldsymbol{v} = \boldsymbol{v}_{\lambda}^{*}(\boldsymbol{w}_{0})$ in short. Then
	$$
	\begin{aligned}
		\left(\boldsymbol{w}_{0}, \boldsymbol{v}_{\lambda}^{*}\left(\boldsymbol{w}_{0}\right)\right) \text { is a partial optimum of } \mathcal{L} \Longleftrightarrow 0& \in \partial_{\boldsymbol{w}} \mathcal{L}\left(\boldsymbol{w}_{0}, \boldsymbol{v}_{\lambda}^{*}\left(\boldsymbol{w}_{0}\right) ; \lambda\right) \\
		&=\partial \sum_{j=1}^{m} \alpha_{j} \mathcal{R}_{j}(\boldsymbol{w}_{0})+\sum_{i=1}^{N} v_{\lambda}^{*}\left(l_{i}(\boldsymbol{w}_{0})\right) \cdot \nabla l_{i}(\boldsymbol{w}_{0}).
	\end{aligned}
	$$
	Combine the above two results we can conclude Theorem \ref{conver2}.
\end{proof}

\subsection{Proof of \textbf{Theorem} \ref{theorem}}
\begin{proof}
In Section \ref{path}, we have shown that any $(\boldsymbol{w},\boldsymbol{v}^{*}(\boldsymbol{\ell}(\boldsymbol{w}),\lambda))$ is a partial optimum iff there exists $\boldsymbol{t}_{\mathcal{R}},\boldsymbol{t}_{\mathcal{\ell}}$ such that (\ref{KKT_}) holds. Given a certain partial optimum $(\hat{\boldsymbol{w}},\boldsymbol{v}^{*}(\boldsymbol{\ell}(\hat{\boldsymbol{w}}),\lambda))$, solving the corresponding $\hat{\boldsymbol{t}}_{\mathcal{R}},\hat{\boldsymbol{t}}_{\ell}$ is indeed calculating the linear equations
\begin{equation}\label{linearEq}
	\begin{aligned}
		\sum_{j=1}^{m} \sum_{k \in I_{\mathcal{R}_{j}}^{a}(\hat{\boldsymbol{w}})} \alpha_{j} \hat{t}_{\mathcal{R}_{j}}^{k}(\hat{\boldsymbol{w}}) \nabla D_{\mathcal{R}_{j}}^{k}(\hat{\boldsymbol{w}})&+\sum_{J \in \mathcal{P}} \sum_{i \in J} \sum_{k \in I_{\ell_{i}}^{a}(\hat{\boldsymbol{w}})} v^{*}_i\left(\ell_{i}(\hat{\boldsymbol{w}}), \lambda\right) \hat{t}_{\ell_{i}}^{k}(\hat{\boldsymbol{w}}) \nabla D_{\ell_{i}}^{k}(\hat{\boldsymbol{w}})=\boldsymbol{0}, \\
		\quad \sum_{k \in I_{\mathcal{R}_{j}}^{a}(\hat{\boldsymbol{w}})} \hat{t}_{\mathcal{R}_{j}}^{k}(\hat{\boldsymbol{w}})-1&=0,  \quad \hat{t}_{\mathcal{R}_{j}}^{k}(\hat{\boldsymbol{w}}) \geq 0, 
		\quad 1 \leq j \leq m \\
		\sum_{k \in I_{\ell_{i}}^{a}(\hat{\boldsymbol{w}})} \hat{t}_{\ell_{i}}^{k}(\hat{\boldsymbol{w}})-1&=0,  \quad \hat{t}_{\ell_{i}}^{k}(\hat{\boldsymbol{w}}) \geq 0, 
		\quad 1\leq i \leq n.
	\end{aligned}
\end{equation}
On a non-critical point, suppose we've obtained a partial optimum $(\hat{\boldsymbol{w}},\boldsymbol{v}^{*}(\boldsymbol{\ell}(\hat{\boldsymbol{w}}),\lambda))$ at $\lambda$. 
Now a critical point is triggered by either of two conditions: \textbf{1)} Partition $\mathcal{P}$ changes. This means the value of some $\ell_{i}$ lie on the boundary between two distinct sets in $\mathcal{P}$. \textbf{2)} One of $I_{\mathcal{R}},I_{\ell}$ changes. This indicates the existence of some $\mathcal{R}_{i}$ ($i \in \bar{I}_{\mathcal{R}}(I_{\mathcal{R}}) $) becomes (non-)differentiable at $\hat{\boldsymbol{w}}$, or some $\ell_{j}$ ($j \in \bar{I}_{\ell}(I_{\ell}) $) becomes (non-)differentiable at $\hat{\boldsymbol{w}}.$ The latter  can be detected by the value of $t_{i}$. For example, assume that $i\in I_{\ell}$ holds along a segment of the path, i.e., there exists $k$ such that $t_{\ell_{i}}^{k}=1$, while $t_{\ell_{i}}^{\tilde{k}}=1$ holds for all $\tilde{k}\neq k.$ At the kink, $\ell_{i}$ changes into non-differentiable. As a result, the value of $t_{\ell_{i}}^{k}$ will decrease from 1, since some other selection functions turn to essentially active status. Altogether, at the optimal $\hat{\boldsymbol{w}}$, all the inequalities in (\ref{KKT_}) are \emph{strict}. In case that $\mathcal{F}(\boldsymbol{w},\lambda,\boldsymbol{t}_{\mathcal{R}},\boldsymbol{t}_{\ell})=\boldsymbol{0}$ deduces a continuous solution path passing $(\hat{\boldsymbol{w}},\hat{\boldsymbol{t}}_{\mathcal{R}},\hat{\boldsymbol{t}}_{\ell})$, the (\ref{KKT_}) will be maintained along the path until the next critical point occurs.

Denote the Jacobin of $\mathcal{F}$ w.r.t. $(\boldsymbol{w},\boldsymbol{t}_{\mathcal{R}},\boldsymbol{t}_{\ell})$, $\lambda$ as $\boldsymbol{J}_{\boldsymbol{w},\boldsymbol{t}_{\mathcal{R}},\boldsymbol{t}_{\ell}}$,  $\boldsymbol{J}_{\lambda}$, respectively. Following our setting in Section \ref{path}, $\mathcal{F}$ is $\mathbb{C}^{1}$ and $\boldsymbol{J}_{\boldsymbol{w},\boldsymbol{t}_{\mathcal{R}},\boldsymbol{t}_{\ell}}$ is invertible at the initial point. In this condition, the implicit function theorem directly indicates the existence and uniqueness of a local $\mathbb{C}^{1}$ solution path of optimal $(\hat{\boldsymbol{w}},\hat{\boldsymbol{t}}_{\mathcal{R}},\hat{\boldsymbol{t}}_{\ell})$ w.r.t. $\lambda$, started from the initial point. Furthermore, the theorem also guarantees that (\ref{mainODE}) is valid along the path. Owing to the $\mathbb{C}^{1}$ function $\mathcal{F}$, the right side of (\ref{mainODE}) is continuous w.r.t. $\lambda$. Therefore, the Picard–Lindelöf theorem \cite{Adkins2012ODE} straightforwardly proves that the solution of (\ref{mainODE}) is unique and can be extended to the nearest boundary of $\lambda$ (\emph{i.e.}, a new critical point appears). 

\end{proof}
\paragraph{Geometric Intuition.}
There exists a geometric understanding towards (\ref{KKT_}) either. Rewriting the first equation in (\ref{KKT_}) by the differentiability of each function gives
\begin{equation}\label{geometry}
    -\sum_{j\in I_{\mathcal{R}}}\alpha_{j}\nabla{\mathcal{R}_{j}}-\sum_{i\in I_{\ell}}v^{*}(\ell_{i},\lambda)\nabla{\ell_{i}} = \sum_{j\in \bar{I}_{\mathcal{R}}}\sum_{k\in I_{\mathcal{R}_{j}}^{a}}\hat{t}_{\mathcal{R}_{j}}^{k}\alpha_{j}\nabla{D_{\mathcal{R}_{j}}^{k}}+\sum_{i\in \bar{I}_{\ell}}\sum_{k\in I_{\ell_{i}}^{a}}\hat{t}_{\ell_{i}}^{k}v^{*}(\ell_{i},\lambda)\nabla{D_{\ell_{i}}^{k}},
\end{equation}
where $\hat{t}_{\mathcal{R}_{j}}^{k}, \hat{t}_{\ell_{i}}^{k}$ meet restrictions in (\ref{KKT_}) and all symbols are in terms of $\hat{\boldsymbol{w}}$. As $\lambda$ varies, the left side in (\ref{geometry}) describes a smooth curve  using the standard frame in $\mathbb{R}^{d}$, while the right side is actually sum of vectors chosen from a convex hull $\operatorname{conv}\left(\bigcup_{k \in I_{\mathcal{R}_{j}}^{a}}\{\alpha_{j}\nabla\left\{D_{\mathcal{R}_{j}}^{k}\right\}\}\right)$ or $\operatorname{conv}\left(\bigcup_{k \in I_{\ell_{i}}^{a}}\left\{v^{*}(\ell_{i},\lambda)\nabla{D_{\ell_{i}}^{k}}\right\}\right)$ and can be viewed as the vector $\boldsymbol{1}$ under an analogue moving frame made up of these selected vectors, from the perspective of differential geometry. In other words, (\ref{geometry}) indeed uses an analogue moving frame to re-depict a smooth curve.

It is worth adding that a recent work \cite{gebken2021structure} shows a similar intuition. Specifically, under some assumptions the solution surface near a given point is a projection of certain smooth manifold even without the convexity assumption.

\subsection{Proof of \textbf{Corollary} \ref{cor}}
\begin{proof}
 Suppose that $\mathcal{R}_{j},\ell_{i}$ in (\ref{obj}) are all differentiable at the given $\hat{\boldsymbol{w}}$, then it yields $I_{\mathcal{R}}=I_{\ell}=\emptyset$. Consequently, (\ref{KKT_}) is degenerate into its first equation
 \begin{equation}
     \mathcal{F}(\hat{\boldsymbol{w}},\lambda)=\sum_{j=1}^{m} \alpha_{j} \nabla \mathcal{R}_{j}(\hat{\boldsymbol{w}})+\sum_{J \in \mathcal{P}} \sum_{i \in J}  v^{*}_i\left(\ell_{i}(\hat{\boldsymbol{w}}), \lambda\right)  \nabla\ell_{i}(\hat{\boldsymbol{w}})=\boldsymbol{0}.
 \end{equation}
  Hence applying Theorem \ref{theorem} straightly gives our Corollary \ref{cor}.
\end{proof}

\subsection{Proof of Prop. \ref{propSVM}}
Here we only present the detailed proof for SVM with linear SP-regularizer, due to the fact that the proof for mixture SP-regularizer is almost the same as the former except some partition difference.

\begin{proof}
Given a partial optimum $\left({(}\hat{\boldsymbol{w}},{b)}\boldsymbol{v}^{*}\left(\hat{\boldsymbol{w}}{,b}\right)\right)$ at $\lambda$, (\ref{KKT_}) can be directly applied here with obvious simplifications.  Mathematically, there exits $\hat{\boldsymbol{t}}=\left(\hat{t}_{i}\right)_{i=1}^{n}$ such that 
\begin{equation}\label{svmKkt1}
    \begin{aligned}
        \hat{\boldsymbol{w}}-\sum_{i=1}^{n}Cv_{i}^{*}\hat{t}_{i}y_{i}\phi(x_{i})&=\boldsymbol{0},\\
        \sum_{i=1}^{n}Cv_{i}^{*}\hat{t}_{i}y_{i}&=\boldsymbol{0},\\
        1-y_{i}\left(\langle \phi(x_{i}),\hat{\boldsymbol{w}}\rangle +b\right)&=0,\quad
        i\in \mathcal{E}_{Z},\\
        \hat{\boldsymbol{t}}_{\mathcal{E}_{N}}=\boldsymbol{0}_{\mathcal{E}_{N}},
         \quad \hat{\boldsymbol{t}}_{\mathcal{E}_{\mathcal{P}}\cup \mathcal{D}}=\boldsymbol{1}_{\mathcal{E}_{\mathcal{P}}\cup \mathcal{D}},&
         \quad \boldsymbol{0} \preceq \hat{\boldsymbol{t}}_{\mathcal{E}_{Z}} \preceq \boldsymbol{1}_{\mathcal{E}_{Z}},
    \end{aligned}
\end{equation}
where $\preceq$ denotes the element-wise comparison between vectors. Denote $\hat{\boldsymbol{\alpha}}=C\boldsymbol{v}^{*}\odot \hat{\boldsymbol{t}}$, then (\ref{svmKkt1}) can be equivalently converted to equations w.r.t. $(\hat{\boldsymbol{\alpha}},b)$ as
    \begin{equation}\label{svmKkt2}
        \begin{aligned}
            \boldsymbol{y}^{T}\hat{\boldsymbol{\alpha}} &= \boldsymbol{0},\\
            \boldsymbol{1}_{\mathcal{E}_{Z}}-Q_{\mathcal{E}_{Z}}\hat{\boldsymbol{\alpha}}-\boldsymbol{y}_{\mathcal{E}_{Z}}\hat{b}&=\boldsymbol{0},\\
            \hat{\boldsymbol{\alpha}}_{\mathcal{E}_{P}} - C\boldsymbol{v}_{{\mathcal{E}_{P}}}^{*}&=
            \boldsymbol{0}_{{\mathcal{E}_{P}}},\\
            \hat{\boldsymbol{\alpha}}_{\mathcal{E}_{Z}} - C\boldsymbol{v}_{\mathcal{E}_{Z}}^{*} \odot
            \hat{\boldsymbol{t}}_{\mathcal{E}_{Z}}&=\boldsymbol{0}_{{\mathcal{E}_{Z}}},\\
            \hat{\boldsymbol{\alpha}}_{\mathcal{E}_{N} \cup \mathcal{D}}& = \boldsymbol{0}_{\mathcal{E}_{N} \cup \mathcal{D}},\\
        \end{aligned}
    \end{equation}
where $Q=\boldsymbol{y}^{T}K\boldsymbol{y}$, $K=\left(k(x_{i},x_{j})\right)_{1\leq i,j \leq n}$ is the kernel matrix and $k$ is the kernel function. Then the optimal $\hat{\boldsymbol{w}}=\sum_{i=1}^{n}\hat{\alpha}_{i}y_{i}\phi(x_{i})$, hence problem (\ref{svm}) is transformed into solving equations merely related to $(\hat{\boldsymbol{ \alpha}},b).$ Specifically, the decision function can be rewritten as $d(x)=\sum_{i=1}^{n}y_{i}\hat{\alpha}_{i}k(x_{i},x)+b.$

Now supposed that $(\hat{\boldsymbol{\alpha}},\hat{b})$ is not a critical point, then the fourth equation in (\ref{svmKkt2}) is actually an inequality constraint by changing the value of $\hat{\boldsymbol{t}}_{\mathcal{E}_{Z}}$ in $\left[0,1\right]^{|\mathcal{E}_{Z}|}.$ As a result, (\ref{svmKkt2}) is only related to $(\hat{\boldsymbol{\alpha}},\hat{b})$ and the left side of the first three equations accords with the function $\mathcal{F}$ in Theorem \ref{theorem}. Consequently, (\ref{SVMode}) can be derived using Theorem \ref{theorem}. {In detail, denote
\begin{equation*}
\begin{aligned}
\mathcal{F}(\hat{\boldsymbol{\alpha}}_{\mathcal{E}_{Z}},\hat{\boldsymbol{\alpha}}_{\mathcal{E}_{P}},\hat{b})&= \left(\begin{array}{c}
     \boldsymbol{y}^{T}_{\mathcal{E}_{Z}\cup\mathcal{E}_{P}}\hat{\boldsymbol{\alpha}}_{\mathcal{E}_{Z}\cup\mathcal{E}_{P}}  \\
     \boldsymbol{1}_{\mathcal{E}_{Z}}-Q_{\mathcal{E}_{Z}}\hat{\boldsymbol{\alpha}}_{\mathcal{E}_{Z}\cup\mathcal{E}_{P}}-\boldsymbol{y}_{\mathcal{E}_{Z}}\hat{b}   \\
     \hat{\boldsymbol{\alpha}}_{\mathcal{E}_{P}}-C\boldsymbol{v}_{\mathcal{E}_{P}}^{*}
\end{array}\right) \\
&= \left(\begin{array}{c}
     \boldsymbol{y}^{T}_{\mathcal{E}_{Z}\cup\mathcal{E}_{P}}\hat{\boldsymbol{\alpha}}_{\mathcal{E}_{Z}\cup\mathcal{E}_{P}}  \\
     \boldsymbol{1}_{\mathcal{E}_{Z}}-Q_{\mathcal{E}_{Z}}\hat{\boldsymbol{\alpha}}_{\mathcal{E}_{Z}\cup\mathcal{E}_{P}}-\boldsymbol{y}_{\mathcal{E}_{Z}}\hat{b}   \\
     \hat{\boldsymbol{\alpha}}_{\mathcal{E}_{P}}-C(\boldsymbol{1}_{\mathcal{E}_{P}}-C\frac{\boldsymbol{1}_{\mathcal{E}_{P}}-Q_{\mathcal{E}_{P}}\hat{\boldsymbol{\alpha}}_{\mathcal{E}_{P}\cup\mathcal{E}_{Z}}-\boldsymbol{y}_{\mathcal{E}_{P}}\hat{b}}{\lambda}),
\end{array}\right)
\end{aligned}
\end{equation*}
then the jacobian can be calculated as 
\begin{equation*}
\begin{aligned}
    \boldsymbol{J}_{\mathcal{F},\left(\hat{\boldsymbol{\alpha}}_{\mathcal{E}_{Z}},\hat{\boldsymbol{\alpha}}_{\mathcal{E}_{P}},b\right)}&= -\left(\begin{array}{ccc}
		-\boldsymbol{y}_{\mathcal{E}_{Z}}^{T} & -\boldsymbol{y}_{\mathcal{E}_{P}}^{T} & 0\\
			Q_{\mathcal{E}_{Z}\mathcal{E}_{Z}} & Q_{\mathcal{E}_{Z}\mathcal{E}_{P}} & \boldsymbol{y}_{\mathcal{E}_{Z}} \\
			\frac{C^{2}}{\lambda} Q_{\mathcal{E}_{P}\mathcal{E}_{Z}} & \frac{C^{2}}{\lambda} Q_{\mathcal{E}_{P}\mathcal{E}_{P}}-I_{\mathcal{E}_{P}\mathcal{E}_{P}} & \frac{C^{2}}{\lambda} \boldsymbol{y}_{\mathcal{E}_{P}} 
		\end{array}\right), \\
		\boldsymbol{J}_{\mathcal{F},\lambda}&=\left(\begin{array}{c}
		    0\\
			\boldsymbol{0}_{\mathcal{E}_{Z}} \\
			-\frac{C}{\lambda^{2}}\boldsymbol{\ell}_{\mathcal{E}_{P}} 
		\end{array}\right).
\end{aligned}
\end{equation*}
The implicit function theorem immediately indicates the following ODEs hold
\begin{equation*}
\begin{aligned}
		\frac{d\left(\begin{array}{c}
		\boldsymbol{\alpha}_{\mathcal{E}_{Z}} \\
		\boldsymbol{\alpha}_{\mathcal{E}_{P}} \\
				b
			\end{array}\right)}{\rule{0pt}{10pt}d\lambda} &= -\boldsymbol{J}_{\mathcal{F},\left(\hat{\boldsymbol{\alpha}}_{\mathcal{E}_{Z}},\hat{\boldsymbol{\alpha}}_{\mathcal{E}_{P}},b\right)}^{-1}\cdot \boldsymbol{J}_{\mathcal{F},\lambda} \\
		&=\left(\begin{array}{ccc}
		-\boldsymbol{y}_{\mathcal{E}_{Z}}^{T} & -\boldsymbol{y}_{\mathcal{E}_{P}}^{T} & 0\\
			Q_{\mathcal{E}_{Z}\mathcal{E}_{Z}} & Q_{\mathcal{E}_{Z}\mathcal{E}_{P}} & \boldsymbol{y}_{\mathcal{E}_{Z}} \\
			\frac{C^{2}}{\lambda} Q_{\mathcal{E}_{P}\mathcal{E}_{Z}} & \frac{C^{2}}{\lambda} Q_{\mathcal{E}_{P}\mathcal{E}_{P}}-I_{\mathcal{E}_{P}\mathcal{E}_{P}} & \frac{C^{2}}{\lambda} \boldsymbol{y}_{\mathcal{E}_{P}} 
		\end{array}\right)^{-1}
		\left(\begin{array}{c}
		    0\\
			\boldsymbol{0}_{\mathcal{E}_{Z}} \\
			-\frac{C}{\lambda^{2}}\boldsymbol{\ell}_{\mathcal{E}_{P}} 
		\end{array}\right).
\end{aligned}
\end{equation*}}

\end{proof}

\subsection{Proof of Prop. \ref{proplasso}}
\begin{proof}
 Following discussions in Section \ref{lassoGuide}, the objective (\ref{lasso}) of Lasso is reformulated under the self-paced paradigm  as
\begin{equation}\label{reSvm}
  \text { Compute } \hat{\boldsymbol{w}} \text {,~ s.t.~} \hat{\boldsymbol{w}} \in \arg
  \min_{\boldsymbol{w}} \alpha \|\boldsymbol{w}\|_{1} + \frac{1}{2n}\|\sqrt{V^{*}}(X\boldsymbol{w}-\boldsymbol{y})\|^{2},
\end{equation}
where $V^{*}$ denotes $Diag\{\boldsymbol{v}^{*}\}$ and $\sqrt{V^{*}}$ denotes $Diag\left\{\sqrt{\boldsymbol{v}^{*}}\right\}.$ Given a partial optimum $\hat{\boldsymbol{w}}$ at $\lambda$,  
applying (\ref{KKT_}) to the objective (\ref{lasso}) deduces
\begin{equation}\label{lassoKkt1}
    \begin{aligned}
        \frac{1}{n} X_{\mathcal{A}}^{T}V^{*}(X\hat{\boldsymbol{w}}-\boldsymbol{y})+\alpha \cdot \mathrm{\textbf{sgn}}(\hat{\boldsymbol{w}}_{\mathcal{A}})&=\boldsymbol{0},\\
        \frac{1}{n} X_{\bar{\mathcal{A}}}^{T}V^{*}(X\hat{\boldsymbol{w}}-\boldsymbol{y})+\alpha \cdot \hat{\boldsymbol{t}}_{\bar{\mathcal{A}}}&=\boldsymbol{0},
    \end{aligned}
\end{equation}
where $-\boldsymbol{1}_{\bar{\mathcal{A}}} \preceq \hat{\boldsymbol{t}}_{\bar{\mathcal{A}}} \preceq
        \boldsymbol{1}_{\bar{\mathcal{A}}}$ and $\boldsymbol{w}_{\bar{\mathcal{A}}}=\boldsymbol{0}_{\bar{\mathcal{A}}}.$ Suppose that $\hat{\boldsymbol{w}}$ is not a critical point, the second equation is indeed converted to an inequality constraint via varying $\hat{\boldsymbol{t}}_{\bar{\mathcal{A}}}$ in $\left[-1,1\right]^{|\bar{\mathcal{A}}|}.$ As a result, (\ref{lassoKkt1}) is merely in connection with $\boldsymbol{w}_{\mathcal{A}}$ and the left side in the first equation consists with the function $\mathcal{F}$ in Theorem \ref{theorem}. Take the \emph{mixture} SP-regularizer as an example, the optimality of estimation when using mixture SP-regularizer is described as
	\begin{equation*}
		\frac{1}{n} \sum_{i \in \mathcal{E}}\left(x_{i} \hat{\boldsymbol w}_\mathcal{A}-y_{i}\right) x_{i}^{T}+\frac{\gamma}{n} \sum_{k \in \mathcal{M}}\underbrace{\left(\frac{1}{2\sqrt{l_k}}-\frac{1}{\lambda} \right)\left(x_{k} \hat{\boldsymbol w}_\mathcal{A}-y_{k}\right) x_{k}^{T}}_{\mathcal{Z}\left(\hat{\boldsymbol w}_\mathcal{A}, \lambda\right)}+\alpha \cdot \mathrm{\textbf{sgn}}(\hat{\boldsymbol{w}}_{\mathcal{A}})=\boldsymbol{0}.
	\end{equation*}
	For the sake of simplicity, we derive the result of $\dfrac{d \mathcal{Z}\left(\hat{\boldsymbol w}_\mathcal{A}, \lambda\right)}{d \lambda}$ first. Note that $l_k= \left(x_{k} \hat{\boldsymbol w}_\mathcal{A}-y_{k}\right)^2$.
	\begin{equation*}
		\begin{aligned}
			\frac{d \mathcal{Z}\left(\hat{\boldsymbol w}_\mathcal{A}, \lambda\right)}{d \lambda} &=\left[-\frac{1}{4l_k\sqrt{l_k}}\cdot2 \left(x_{k} \hat{\boldsymbol w}_\mathcal{A}-y_{k}\right) \cdot x_k \frac{d \hat{\boldsymbol w}_\mathcal{A}}{d \lambda}+\frac{1}{\lambda^2}  \right]  \left(x_{k} \hat{\boldsymbol w}_\mathcal{A}-y_{k}\right) x_{k}^{T}+\left(\frac{1}{2\sqrt{l_k}}-\frac{1}{\lambda} \right) x_k^T x_{k} \frac{d \hat{\boldsymbol w}_\mathcal{A}}{d \lambda} \\
			&=\frac{1}{\lambda^2}  \left(x_{k} \hat{\boldsymbol w}_\mathcal{A}-y_{k}\right)x_k^T-\frac{1}{\lambda} x_k^T x_{k} \frac{d \hat{\boldsymbol w}_\mathcal{A}}{d \lambda}.
		\end{aligned}
	\end{equation*}
	Similar to the proof of {Theorem} \ref{theorem}, we have
	\begin{equation*}
		\frac{1}{n} \sum_{i \in \mathcal{E}}x_i^T x_{i} \frac{d \hat{\boldsymbol w}_\mathcal{A}}{d \lambda}+\frac{\gamma}{n} \sum_{k \in \mathcal{M}}\left[\frac{1}{\lambda^2}  \left(x_{k} \hat{\boldsymbol w}_\mathcal{A}-y_{k}\right)x_k^T-\frac{1}{\lambda} x_k^T x_{k} \frac{d \hat{\boldsymbol w}_\mathcal{A}}{d \lambda} \right]  =\boldsymbol{0}.
	\end{equation*}
	And final result comes from combining and vectoring the terms w.r.t. $\dfrac{d \hat{\boldsymbol w}_\mathcal{A}}{d \lambda}$, which can be utilized to derive the Prop. \ref{proplasso}.

\end{proof}

\section{Detailed Algorithms}\label{algs}
In this section, we present more details of the concrete algorithms derived for the SVM and Lasso.
\subsection{Support Vector Machines}
The goal of \method is to calculate the age-path on the interval $[\lambda_{min},\lambda_{max}]$. 
As mentioned in Section \ref{tracking}, started from an initial point, the algorithm solves the derived ODEs (\ref{SVMode}) 
while examining all the partitions along the way of $\lambda$. In SVM, partition violating is merely caused by the change in  $\boldsymbol{g}$ (\emph{c.f.} Section \ref{svmGuide}). For example, some $g_{i}$ varying from a negative value to zero will lead the violation of $\mathcal{E}_{N},$ resulting in a critical point. In case that the point belongs to a turning point, the only need is to resign $i$ from $\mathcal{E}_{N}$ to $\mathcal{E}_{Z}.$ Otherwise the point is a jump point and we have to perform the warm start to re-calculate the next solution, which could be time consuming. Since it's 
non-trivial to identify the type of the critical point as a prior, we adopt a heuristic operation to avoid excessive warm starts. When the partition violation occurs, we directly resign all the violated indexes into the right status by the  partition rule. Suppose that  a turning point is encountered, the solutions return by numerical ODEs  solver with the updated index sets will keep the KKT condition, allowing our algorithm to proceed. The algorithmic steps are given in Algorithm \ref{alg:svm}.
\begin{algorithm}[h]
	\caption{~~\method for SVM}
	\label{alg:svm}
	\textbf{Input}: Initial solution $(\boldsymbol{\alpha},b)|_{\lambda_{t}=\lambda_{min}}$, $X$, $y$, $\lambda_{min}$ and $\lambda_{max}$\\
	\textbf{Parameter}: Cost parameter $C$\\
	\textbf{Output}: Age-path $(\boldsymbol{\alpha},b)$ on $[\lambda_{min}, \lambda_{max}]$
	\begin{algorithmic}[1] 
        \STATE $\lambda_t\gets\lambda_{min}$, set $\mathcal{E}_{N},\mathcal{E}_{Z},\mathcal{E}_{P},\mathcal{D}(,\mathcal{M})$ in Proposition \ref{propSVM} according to $\boldsymbol w|_{\lambda_{t}=\lambda_{min}}$
        \WHILE{$\lambda_{t}\leq\lambda_{max}$ } 
        \STATE Solve (\ref{SVMode}) or (\ref{SVMode1}) and partition samples in $X$ and components of $\boldsymbol{\alpha}$  simultaneously.
        \IF{Partition $\mathcal{E}_{N},\mathcal{E}_{Z},\mathcal{E}_{P},\mathcal{D}(,\mathcal{M})$ was not met}
        \STATE Resign violated indexes in $\mathcal{P}$ by $g_{i}$. 
        \ENDIF
        \STATE $\boldsymbol{\alpha}_{\mathcal{E}_{N}}=\boldsymbol{0}_{\mathcal{E}_{N}},\boldsymbol{\alpha}_{\mathcal{D}}=\boldsymbol{0}_{\mathcal{D}}.$ For the mixture regularizer, $\boldsymbol{\alpha}_{\mathcal{E}_{P}}=\boldsymbol{1}_{\mathcal{E}_{P}}.$
        \STATE Solve (\ref{SVMode}) or (\ref{SVMode1}) with updated $\mathcal{E}_{N},\mathcal{E}_{Z},\mathcal{E}_{P},\mathcal{D}(,\mathcal{M})$
        \IF {KKT conditions are not met}
		\STATE Warm start at $\lambda_t+\delta$ (for a small $\delta>0$).
		\ENDIF
        \ENDWHILE
	\end{algorithmic}
\end{algorithm}
\subsection{Lasso}
In Lasso, the main routine of \method is similar with that in SVM. The only difference is that here we need to examine the partition $\mathcal{A}$ additionally.

Due to the property of the $\ell_{1}$ norm, monitoring $\mathcal{A}$ is operated by observing if the value of $\boldsymbol w_{i}$ equals to zero or conversely, whether the subgradient of inactive component is reached to $1$. 

The details of the algorithms are shown in Algorithm \ref{alg:lasso}.

\begin{algorithm}[h]
	\caption{~~\method for Lasso}
	\label{alg:lasso}
	\textbf{Input}: $\boldsymbol{w}|_{\lambda_{t}=\lambda_{min}}$, $X$, $y$, $\lambda_{min}$ and $\lambda_{max}$\\
	\textbf{Parameter}: Regularization strength $\alpha$\\
	\textbf{Output}: Age-path $\boldsymbol w$ on $[\lambda_{min}, \lambda_{max}]$
	\begin{algorithmic}[1] 
		\STATE $\lambda_t\gets\lambda_{min}$, set $\mathcal{A},\mathcal{P}$ in Proposition \ref{proplasso} according to $\boldsymbol w|_{\lambda_{t}=\lambda_{min}}$ 
		\WHILE{$\lambda_{t}\leq\lambda_{max}$  }
		\STATE Solve (\ref{lassoode}) or (\ref{lassoode2}) and partition $\boldsymbol w_{\mathcal{A}}$, $X_{\mathcal{A}}$ simultaneously.
		\IF{Partition $\mathcal{A},\mathcal{P}$ was not met}
		\IF{$\mathcal{A}$ was not met}
		\IF {$k$-th element turns to inactive}
		\STATE $\boldsymbol w_k=\boldsymbol0$.
		\STATE Remove $k$ from $\mathcal{A}$.
		\ELSIF {$k$-th element becomes active}
		\STATE Put $k$ into $\mathcal{A}$.
		\ENDIF
		\ENDIF
		\IF{$\mathcal{P}$ was not met}
		\STATE Resign violated indexes in $\mathcal{P}$ by $\ell_{i}$. 
		\ENDIF
		\STATE Solve (\ref{lassoode}) or (\ref{lassoode2}) with updated $\mathcal{A},\mathcal{P}.$
		\IF {KKT conditions are not met}
		\STATE Warm start at $\lambda_t+\delta$ (for a small $\delta>0$).
		\ENDIF
		\ENDIF
		
		\ENDWHILE
	\end{algorithmic}
\end{algorithm}

\section{Additional Results}
\label{additional}
In this section, we present additional experimental results on the logistic regression and path consistency to obtain a more comprehensive evaluation of proposed algorithm.
\subsection{Logistic Regression}
Given the dataset $X$ and label $\boldsymbol{y}$, the logistic regression gives the optimization problem as (\ref{logisticObj})
\begin{equation}\label{logisticObj}
    \min_{\boldsymbol{w}\in \mathbb{R}^{d},~b} ~~\frac{1}{2}\|\boldsymbol{w}\|^{2} + \sum_{i=1}^{n}C\ln\left(1+e^{-y_{i}(X_{i}\boldsymbol{w}+b)}\right),
\end{equation}
where $C>0$ is the trade-off parameter. Note the objective (\ref{logisticObj}) is smooth on the entire domain, hence we apply Corollary \ref{cor} to derive the ODEs and the only $\mathcal{P}$ is needde to be tracked and reset along the path. To start with, let $\ell_{i}=C\ln(1+e^{-y_{i}(X_{i}\boldsymbol{w}+b)})$, then for the linear SP-regularizer, the partition $\mathcal{P} = \{\mathcal{E},\mathcal{D}\}$, where
$\mathcal{E}=\{1\leq i \leq n: \ell_{i}<\lambda \},~\mathcal{D}=\{1\leq i \leq n: \ell_{i}\geq\lambda \}.$ For the mixture SP-regularizer, $\mathcal{P} = \{\mathcal{E},\mathcal{M},\mathcal{D}\}$, where
$\mathcal{E}=\{1\leq i \leq n: \ell_{i}<\left(\frac{\lambda\gamma}{\lambda+\gamma}\right)^{2} \},~\mathcal{M}=\{1\leq i \leq n: \left(\frac{\lambda\gamma}{\lambda+\gamma}\right)^{2} \leq \ell_{i} \leq \lambda^{2} \},$ and $\mathcal{D}=\{1\leq i \leq n: \ell_{i}>\lambda^{2} \}.$ Applying Corollary \ref{cor} obtains Theorem \ref{logisticKkt}.
\begin{theorem}\label{logisticKkt}
    When $\boldsymbol{w},b$ indicate a partial optimum, the dynamics of optimal $\boldsymbol{w},b$ in (\ref{logisticObj}) w.r.t. $\lambda$ for the linear and mixture SP-regularizer are shown as
    \begin{equation}\label{logisticOde1}
		\frac{d\left(\begin{array}{c}
				\boldsymbol{w} \\
				b
			\end{array}\right)}{\rule{0pt}{10pt}d\lambda} =
		\left(\begin{array}{cc}
		I+CX_{\mathcal{E}}^{T}U_{\mathcal{E}} X_{\mathcal{E}}  & CX_{\mathcal{E}}^{T}U_{\mathcal{E}} \\\rule{0pt}{19pt}
			C\boldsymbol{1}_{\mathcal{E}}^{T}U_{\mathcal{E}} X_{\mathcal{E}} & C\boldsymbol{1}_{\mathcal{E}}^{T}U_{\mathcal{E}}\\
			
		\end{array}\right)^{-1}
		\left(\begin{array}{c}
		    CX_{\mathcal{E}}^{T}\left[\dfrac{ \boldsymbol{y}_{\mathcal{E}} \odot \boldsymbol{\ell}_{\mathcal{E}} }{\lambda^{2}} \odot \left(e^{-\frac{\boldsymbol{\ell}_{\mathcal{E}}}{C}}-1 \right) \right]\\
			\rule{0pt}{22pt}C\boldsymbol{1}_{\mathcal{E}}^{T}\left[\dfrac{ \boldsymbol{y}_{\mathcal{E}} \odot \boldsymbol{\ell}_{\mathcal{E}} }{\lambda^{2}} \odot \left(e^{-\frac{\boldsymbol{\ell}_{\mathcal{E}}}{C}}-1 \right) \right]
			
		\end{array}\right),
	\end{equation}
	where 
	$U_{\mathcal{E}} = 
	    Diag\left\{\boldsymbol{y}_{\mathcal{E}}^{2} \odot \boldsymbol{u}_{\mathcal{E}} \right\},
	\boldsymbol{u}_{\mathcal{E}} = 
	    \left(\dfrac{\boldsymbol{\ell}_{\mathcal{E}}-C}{\lambda}-1\right) \odot e^{-\frac{2\boldsymbol{\ell}_{\mathcal{E}}}{C}}+\left(\dfrac{2C-\boldsymbol{\ell}_{\mathcal{E}}}{\lambda}+1\right)\odot e^{-\frac{\boldsymbol{\ell}_{\mathcal{E}}}{C}}-\dfrac{C}{\lambda}.
	$
	    \begin{equation}\label{logisticOde2}
		\frac{d\left(\begin{array}{c}
				\boldsymbol{w} \\
				b
			\end{array}\right)}{\rule{0pt}{10pt}d\lambda} =
		\left(\begin{array}{cc}
		I+CX_{\mathcal{A}}^{T}U_{\mathcal{A}}  & CX_{\mathcal{A}} \\\rule{0pt}{19pt}
			C\boldsymbol{1}_{\mathcal{A}}^{T}U_{\mathcal{A}} X_{\mathcal{A}} & C\boldsymbol{1}_{\mathcal{A}}^{T}U_{\mathcal{A}}\\
			
		\end{array}\right)^{-1}
		\left(\begin{array}{c}
		    CX_{\mathcal{A}}^{T}\left[\dfrac{ \boldsymbol{y}_{\mathcal{A}} \odot \boldsymbol{\ell}_{\mathcal{A}} }{\lambda^{2}} \odot \left(e^{-\frac{\boldsymbol{\ell}_{\mathcal{A}}}{C}}-1 \right) \right]\\
			C\boldsymbol{1}_{\mathcal{A}}^{T}\rule{0pt}{22pt}\left[\dfrac{ \boldsymbol{y}_{\mathcal{A}} \odot \boldsymbol{\ell}_{\mathcal{A}} }{\lambda^{2}} \odot \left(e^{-\frac{\boldsymbol{\ell}_{\mathcal{A}}}{C}}-1 \right) \right]
		\end{array}\right),
	\end{equation}
	where $\mathcal{A}=\mathcal{E}\cup\mathcal{M},U_{\mathcal{A}} = 
	    Diag\left\{\boldsymbol{y}_{\mathcal{A}}^{2} \odot \boldsymbol{u}_{\mathcal{A}} \right\},
	\boldsymbol{u}_{\mathcal{E}} = e^{-\frac{\boldsymbol{\ell}_{\mathcal{E}}}{C}} \odot \left(1-e^{-\frac{\boldsymbol{\ell}_{\mathcal{E}}}{C}}\right),
	\boldsymbol{u}_{\mathcal{M}} = 
	\left(\dfrac{C}{2}\boldsymbol{\ell}_{\mathcal{M}}^{-\frac{3}{2}}+\boldsymbol{\ell}_{\mathcal{M}}^{-\frac{1}{2}}-\dfrac{1}{\lambda}\right) \odot e^{-\frac{2\boldsymbol{\ell}_{\mathcal{M}}}{C}}-
	\left(C\boldsymbol{\ell}_{\mathcal{M}}^{-\frac{3}{2}}+\boldsymbol{\ell}_{\mathcal{M}}^{-\frac{1}{2}}-\dfrac{1}{\lambda}\right)\odot
	e^{-\frac{\boldsymbol{\ell}_{\mathcal{M}}}{C}}+
	\dfrac{C}{2}\boldsymbol{\ell}_{\mathcal{M}}^{-\frac{3}{2}}.
	$
	
\end{theorem}

\begin{proof}
The proof is nearly the same as that of SVM and Lasso, hence we merely present the main structure in the following. The linear SP-regularizer is utilized during the derivation, while the proof of mixture SP-regularizer is quite similar.

Given a partial optimum $(\boldsymbol{w},b)$ at $\lambda,$ (\ref{KKT_}) is rewritten in detail with the form of 
\begin{equation*}
    \begin{aligned}
        \boldsymbol{w}+C\sum_{i\in\mathcal{E}}\left(\frac{1}{1+e^{-y_{i}(\boldsymbol{w}^{T}X_{i}+b)}}-1\right)X^{T}_{i}&=\boldsymbol{0}\\
        C\sum_{i\in\mathcal{E}}\left(\frac{1}{1+e^{-y_{i}(\boldsymbol{w}^{T}X_{i}+b)}}-1\right)&=0.
    \end{aligned}
\end{equation*}
Similarly, we set the $\mathcal{F}$ as 
\begin{equation*}
    \begin{aligned}
        \mathcal{F}&=\left(\begin{array}{c}
            \boldsymbol{w}+C\sum\limits_{i\in\mathcal{E}}\left(\dfrac{1}{1+e^{-y_{i}(\boldsymbol{w}^{T}X_{i}+b)}}-1\right)X^{T}_{i} \\
            C\sum\limits_{i\in\mathcal{E}}\left(\dfrac{1}{1+e^{-y_{i}(\boldsymbol{w}^{T}X_{i}+b)}}-1\right)
        \end{array}\right)\\
        &=\left(\begin{array}{c}
            \boldsymbol{w}+CX^{T}_{\mathcal{E}}\left(\boldsymbol{y}_{\mathcal{E}}\odot\left(\boldsymbol{1}-\frac{\boldsymbol{\ell_{\mathcal{E}}}}{\lambda}\right)\odot\left(e^{-\dfrac{\boldsymbol{\ell}_{\mathcal{E}}}{c}}-1\right)\right)   \\
            C\boldsymbol{1}^{T}_{\mathcal{E}}\left(\boldsymbol{y}_{\mathcal{E}}\odot\left(\boldsymbol{1}-\frac{\boldsymbol{\ell_{\mathcal{E}}}}{\lambda}\right)\odot\left(e^{-\dfrac{\boldsymbol{\ell}_{\mathcal{E}}}{c}}-1\right)\right)  
        \end{array}\right).
    \end{aligned}
\end{equation*}
Afterwards, the corresponding jaconbian is derived as 
\begin{equation*}
    \begin{aligned}
        \boldsymbol{J}_{\mathcal{F},\left(\boldsymbol{w},b\right)}&=\left(\begin{array}{cc}
		I+CX_{\mathcal{E}}^{T}U_{\mathcal{E}} X_{\mathcal{E}}  & CX_{\mathcal{E}}^{T}U_{\mathcal{E}} \\\rule{0pt}{19pt}
			C\boldsymbol{1}_{\mathcal{E}}^{T}U_{\mathcal{E}} X_{\mathcal{E}} & C\boldsymbol{1}_{\mathcal{E}}^{T}U_{\mathcal{E}}\\
		\end{array}\right)\\
        \boldsymbol{J}_{\mathcal{F},\lambda}&=-\left(\begin{array}{c}
		    CX_{\mathcal{E}}^{T}\left[\dfrac{ \boldsymbol{y}_{\mathcal{E}} \odot \boldsymbol{\ell}_{\mathcal{E}} }{\lambda^{2}} \odot \left(e^{-\frac{\boldsymbol{\ell}_{\mathcal{E}}}{C}}-1 \right) \right]\\
			\rule{0pt}{22pt}C\boldsymbol{1}_{\mathcal{E}}^{T}\left[\dfrac{ \boldsymbol{y}_{\mathcal{E}} \odot \boldsymbol{\ell}_{\mathcal{E}} }{\lambda^{2}} \odot \left(e^{-\frac{\boldsymbol{\ell}_{\mathcal{E}}}{C}}-1 \right) \right]
		\end{array}\right),
    \end{aligned}
\end{equation*}
where $U_{\mathcal{E}} = 
	    Diag\left\{\boldsymbol{y}_{\mathcal{E}}^{2} \odot \boldsymbol{u}_{\mathcal{E}} \right\},
	\boldsymbol{u}_{\mathcal{E}} = 
	    \left(\dfrac{\boldsymbol{\ell}_{\mathcal{E}}-C}{\lambda}-1\right) \odot e^{-\frac{2\boldsymbol{\ell}_{\mathcal{E}}}{C}}+\left(\dfrac{2C-\boldsymbol{\ell}_{\mathcal{E}}}{\lambda}+1\right)\odot e^{-\frac{\boldsymbol{\ell}_{\mathcal{E}}}{C}}-\dfrac{C}{\lambda}
	$.
Therefore, the implicit function theorem implies that
\begin{equation*}
    \begin{aligned}
		\frac{d\left(\begin{array}{c}
				\boldsymbol{w} \\
				b
			\end{array}\right)}{\rule{0pt}{10pt}d\lambda} &= -\boldsymbol{J}_{\mathcal{F},\left(\boldsymbol{w},b\right)}^{-1}\cdot
			\boldsymbol{J}_{\mathcal{F},\lambda}\\
		&=\left(\begin{array}{cc}
		I+CX_{\mathcal{E}}^{T}U_{\mathcal{E}} X_{\mathcal{E}}  & CX_{\mathcal{E}}^{T}U_{\mathcal{E}} \\\rule{0pt}{19pt}
			C\boldsymbol{1}_{\mathcal{E}}^{T}U_{\mathcal{E}} X_{\mathcal{E}} & C\boldsymbol{1}_{\mathcal{E}}^{T}U_{\mathcal{E}}\\
			
		\end{array}\right)^{-1}
		\left(\begin{array}{c}
		    CX_{\mathcal{E}}^{T}\left[\dfrac{ \boldsymbol{y}_{\mathcal{E}} \odot \boldsymbol{\ell}_{\mathcal{E}} }{\lambda^{2}} \odot \left(e^{-\frac{\boldsymbol{\ell}_{\mathcal{E}}}{C}}-1 \right) \right]\\
			\rule{0pt}{22pt}C\boldsymbol{1}_{\mathcal{E}}^{T}\left[\dfrac{ \boldsymbol{y}_{\mathcal{E}} \odot \boldsymbol{\ell}_{\mathcal{E}} }{\lambda^{2}} \odot \left(e^{-\frac{\boldsymbol{\ell}_{\mathcal{E}}}{C}}-1 \right) \right]
			
		\end{array}\right).
		\end{aligned}
	\end{equation*}
\end{proof}
\color{black}
\subsection{Detailed Experimental Setting}\label{detail}

We use the Scikit-learn  package \cite{scikitlearn} to optimize the subproblems of SVM, logistic regression and Lasso. The  MOSPL method is implemented using the toolbox geatpy \cite{Jazzbin2020geatpy}. All codes were implemented in Python and all experiments were conducted on a machine with 48 2.2GHz cores, 80GB of RAM and 4 Nvidia 1080ti GPUs.

In all experiments of performance comparison, we evaluate the average performance in 20 runs. To maintain the reproducibility, the random seed is fixed with $40$. In each trail, the each dataset is randomly divided into a training set and a testing set by the ratio of $3:1$. When carrying out \method and ACS, the predefined interval of $\lambda$ is set to $\left[0.1,20\right],$ and the step size in ACS equals to $0.5$. 

We utilize the NSGA-\uppercase\expandafter{\romannumeral3} as the framework of MOSPL, in which $N_p$ is set to 150 and $Gen=800$\footnote{$N_p$ and $Gen$ represent the number of populations and the utmost generations in evolutionary algorithm.}. When applying the mixture regularizer, we utilize the polynomials loss in \cite{gong2018decomposition} as $\ell_{i}$ to transform the original problem into a multi-objective problem. Afterwards, the polynomial order $t$ is fixed at 1.2 and 1.35, respectively. 
\subsection{Simulation Study on Logistic Regression}
We present additional experimental results on the logistic regression (\ref{logisticObj}) to validate our ODEs. 
The utilized datasets are listed in Table \ref{tab:logisData}.
\begin{table}[h]
    \centering
    \begin{tabular}{lllll}\toprule
  		Dataset & Source & Samples & Dimensions & Task\\ \midrule
		mfeat-pixel & UCI  & 2000 & 240 & \multirow{5}{*}{C}\\ 
		pendigits & UCI & 3498 & 16 \\
		hiva agnostic & OpenML & 4230 & 1620  \\ 
		nomao & OpenML & 34465 & 118 \\
		MagicTelescope & OpenML & 19020 & 11 \\
		\bottomrule \\
    \end{tabular}
    \caption{Datasets description in experiments on logistic regression. The C = Classification.}
    \label{tab:logisData}
\end{table}
The averaged results using the linear and mixture SP-regularizer are illustrated in Table \ref{tab:acc3} and \ref{tab:acc4}, respectively, in which the performance is measured by the classification accuracy. Meanwhile, Figure \ref{fig:time_log} confirms the computational efficiency of \method on large-scale dataset. Taking all results into consideration, \method outperforms than the baseline methods on all datasets and 
parameter settings, hence demonstrates the performance of \method in classification tasks with large data size.

\begin{table}[t]
	\caption{Average results with the standard deviation in 20 runs on different datasets using the \emph{linear SP-regularizer}. The top results in each row are in boldface. The $\dag$ and $\ddag$ share the same meaning as in the main body.}
	\label{tab:acc3}
	\centering\tiny
	\setlength{\tabcolsep}{7pt}
	\resizebox{\textwidth}{!}{
	\begin{tabular}{l cc c     ccc c c}
		\toprule
		\bf \multirow{2}{*}{\vspace{-5pt}Dataset} &  \multicolumn{2}{c}{\bf Parameter} && \multicolumn{3}{c}{\bf Competing Methods} & \bf Ours &\bf \multirow{2}{*}{\vspace{-5pt}Restarting times}\\
		\cmidrule{2-3}  \cmidrule {5-7}
		& $C$ & $\gamma$ &&  \it Original & \it ACS & \it MOSPL &  ~\method~   \\
		\midrule
		mfeat-pixel\dag   & 0.50 & --         %
		&& ~0.827\stderr{0.028}~ & ~0.937\stderr{0.054}~ & ~0.962\stderr{0.0281}~ & ~\textbf{0.980}\stderr{0.015}~ & 189\\
		pendigits\ddag    & 0.50 & --         %
		&& ~0.982\stderr{0.007}~ & ~0.989\stderr{0.006}~ & ~0.986\stderr{0.012}~ & ~\textbf{0.992}\stderr{0.006}~ & 86\\
		hiva agnostic\dag  & 1.00 &   --                  %
		&& ~0.681\stderr{0.020}~ & 0.700\stderr{0.021} & ~0.958\stderr{0.010}~ & ~\textbf{0.965}\stderr{0.004}~ & 328\\
		MagicTelescope\dag   & 0.50 & --%
		&& ~0.974\stderr{0.004}~ & ~0.981\stderr{0.001}~ & ~0.977\stderr{0.005}~ & ~\textbf{0.991}\stderr{0.001}~ & 61\\
		nomao\ddag     & 0.50          & --      %
		&& ~0.939\stderr{0.002}~  & ~0.940\stderr{0.003}~ & ~0.944\stderr{0.001}~ & ~\textbf{0.944}\stderr{0.001}~ & 32\\
		\bottomrule \\
	\end{tabular}
	}
	\vspace{-10pt}
\end{table}

\begin{table}[t]
	\caption{Average results with the standard deviation in 20 runs on different datasets using the \emph{mixture SP-regularizer}. The top results in each row are in boldface. The $\dag$ and $\ddag$ share the same meaning as in the main body.}
	\label{tab:acc4}
	\centering\tiny
	\setlength{\tabcolsep}{7pt}
	\resizebox{\textwidth}{!}{
	\begin{tabular}{l cc c     ccc c c}
		\toprule
		\bf \multirow{2}{*}{\vspace{-5pt}Dataset} &  \multicolumn{2}{c}{\bf Parameter} && \multicolumn{3}{c}{\bf Competing Methods} & \bf Ours &\bf \multirow{2}{*}{\vspace{-5pt}Restarting times}\\
		\cmidrule{2-3}  \cmidrule {5-7}
		& $C $ & $\gamma$ &&  \it Original & \it ACS & \it MOSPL &  ~\method~   \\
		\midrule
		mfeat-pixel\dag    & 0.50 & 0.20         %
		&& ~0.827\stderr{0.028}~ & ~0.980\stderr{0.011}~ & ~0.981\stderr{0.014}~ & ~\textbf{0.981}\stderr{0.018}~ & 166\\
		pendigits\ddag     & 0.50 & 0.20         %
		&& ~0.982\stderr{0.007}~ & ~0.989\stderr{0.007}~ & ~0.988\stderr{0.006}~ & ~\textbf{0.993}\stderr{0.005}~ & 178\\
		hiva agnostic\dag  & 0.50 &   0.20                %
		&& ~0.681\stderr{0.020}~ & 0.713\stderr{0.021} & ~0.944\stderr{0.008}~ & ~\textbf{0.973}\stderr{0.014}~ & 195\\
		MagicTelescope\dag   & 0.50 &  0.20%
		&& ~0.974\stderr{0.004}~ & ~0.976\stderr{0.003}~ & ~0.991\stderr{0.003}~ & ~\textbf{0.991}\stderr{0.001}~ & 73\\
		nomao\ddag     & 0.50          & 0.20      %
		&& ~0.939\stderr{0.002}~  & ~0.941\stderr{0.001}~ & ~0.941\stderr{0.002}~ & ~\textbf{0.946}\stderr{0.002}~ & 44\\
		
		\bottomrule \\
	\end{tabular}
	}
	\vspace{-10pt}
\end{table}

\begin{figure}[H]
    \vspace{-1em}
	\centering
	\begin{adjustbox}{width=0.95\textwidth}
		\includegraphics{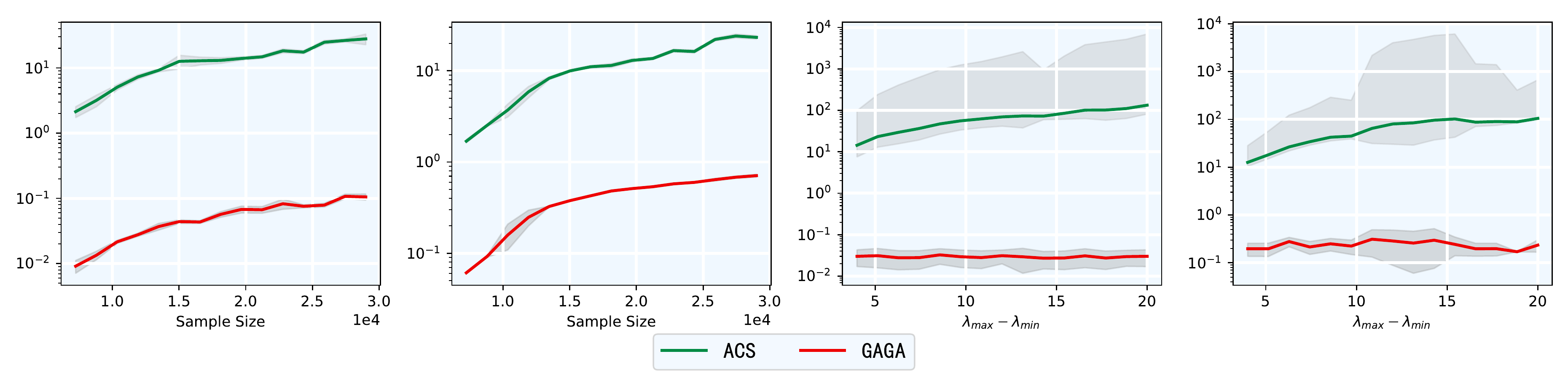}
	\end{adjustbox}
	\vspace{-1em}
	\caption{The study of efficiency comparison. $y$-axis denotes the average running time (in seconds) in 20 runs. The interval $[\lambda_{min},\lambda_{max}]$ refers to the predefined search space.}
	\label{fig:time_log}
	\vspace{-1em}
\end{figure}

\subsection{Path consistency}

In this section we illustrate that the age-path tracked by \method exactly consists with the path of real partial optimum, which is produced by the ACS algorithm. Due to the expensive computational cost of finding the partial optimum using ACS (over 30 loops on average), we choose the toy datasets from the Scikit-learn package to trace and plot the age-path. In detail, we use the Boston house and breast cancer datasets for regression tasks, and the classification is performed on the handwritten digits dataset. In order to track the exact path of partial optimum, we set the step size to be 1e-4 and 3e-1 in ACS and \method, respectively. The graphs of the tracked path are illustrated in Figure \ref{path_1}, Figure \ref{path_2} and Figure \ref{path_3}. The path of partial optimum is plotted in blue solid lines while the age-path traced by \method is marked with red dashed lines. 

This result empirically validates the path consistency between the computed age-path by \method and the ground truth age-path (\emph{i.e.} path of the partial optimum), which is stated in Theorem (\ref{theorem}).

\begin{figure}
\begin{tabular}{lll}
\toprule
      \includegraphics[width=0.33\linewidth]{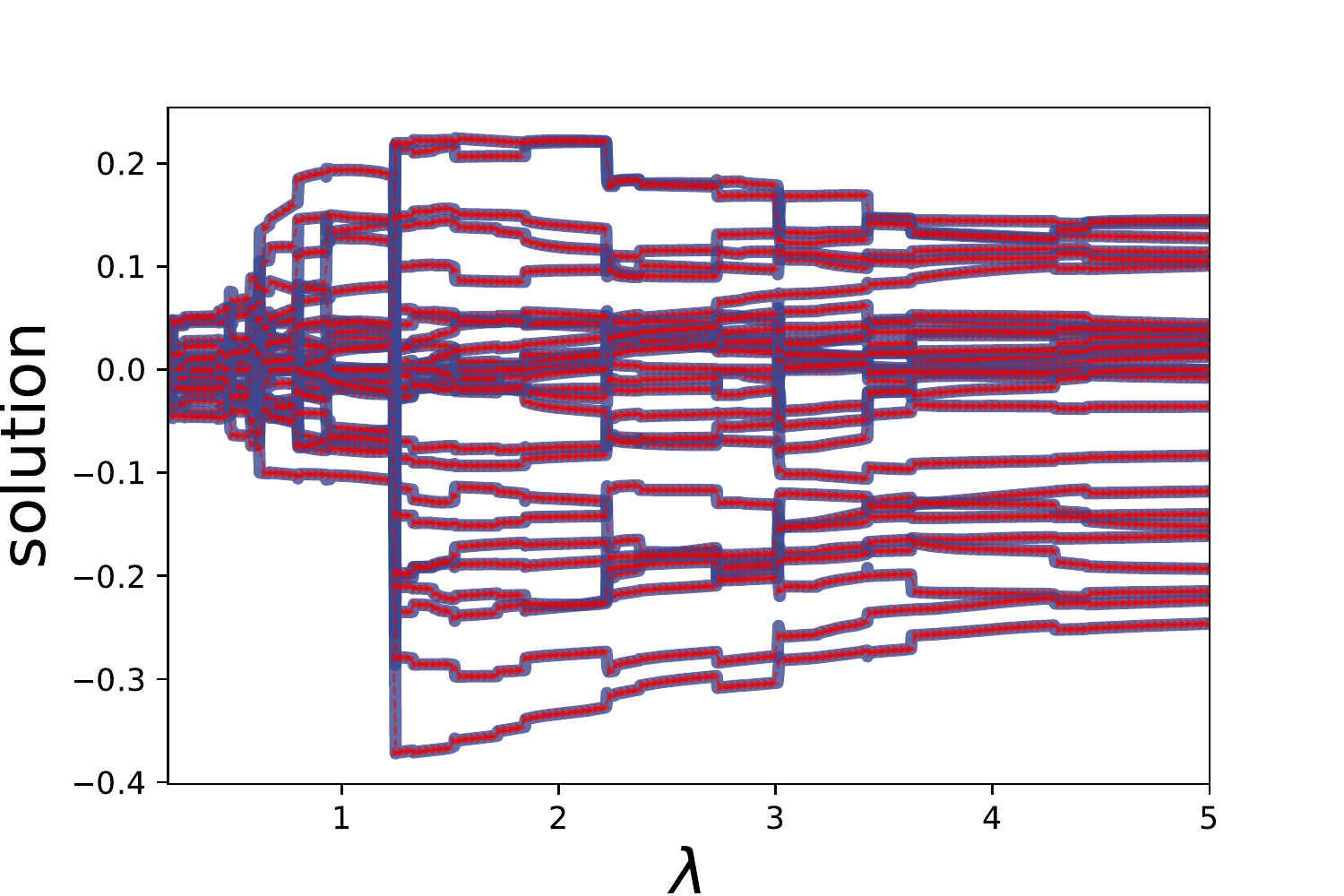} &      \includegraphics[width=0.33\linewidth]{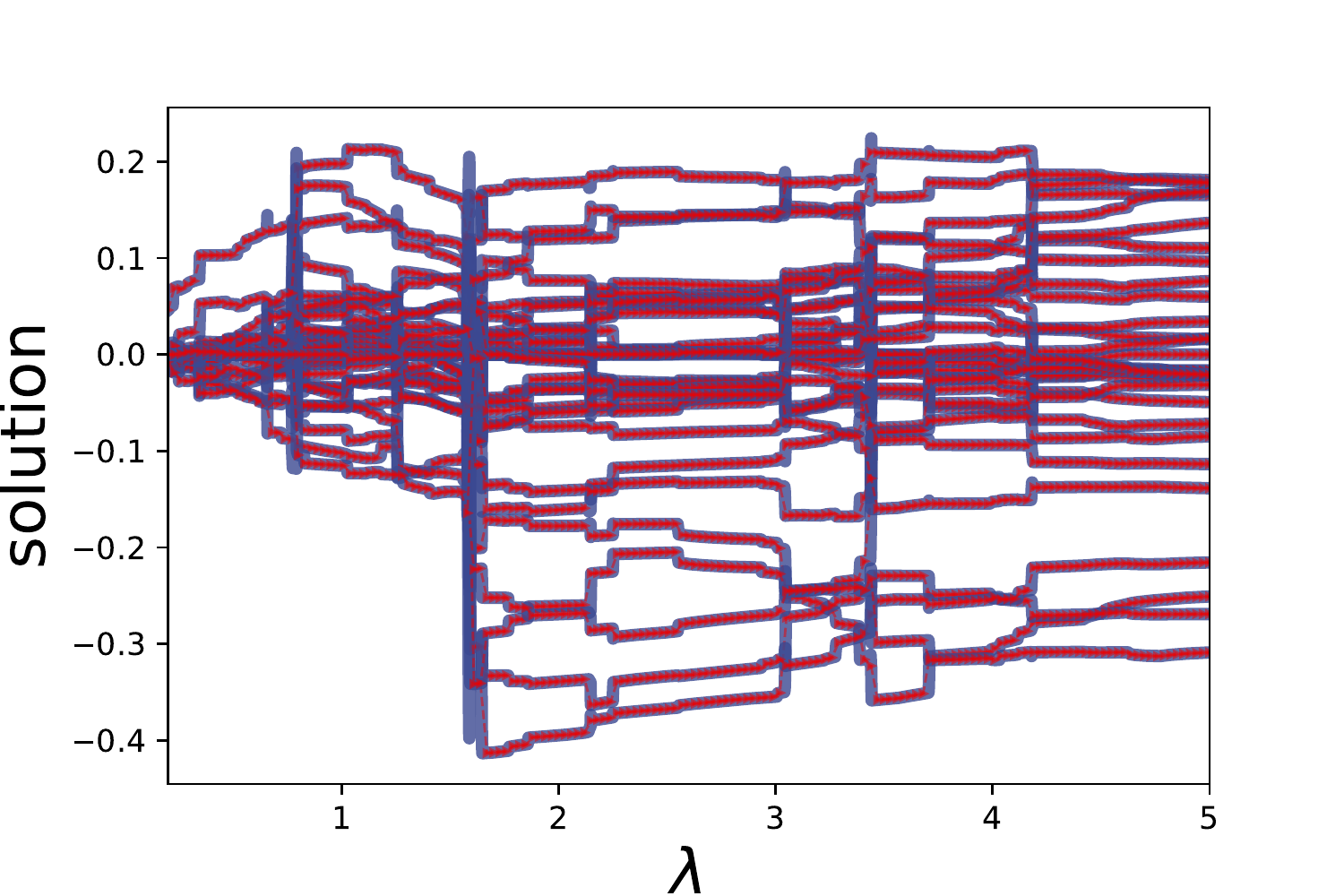} &        \includegraphics[width=0.33\linewidth]{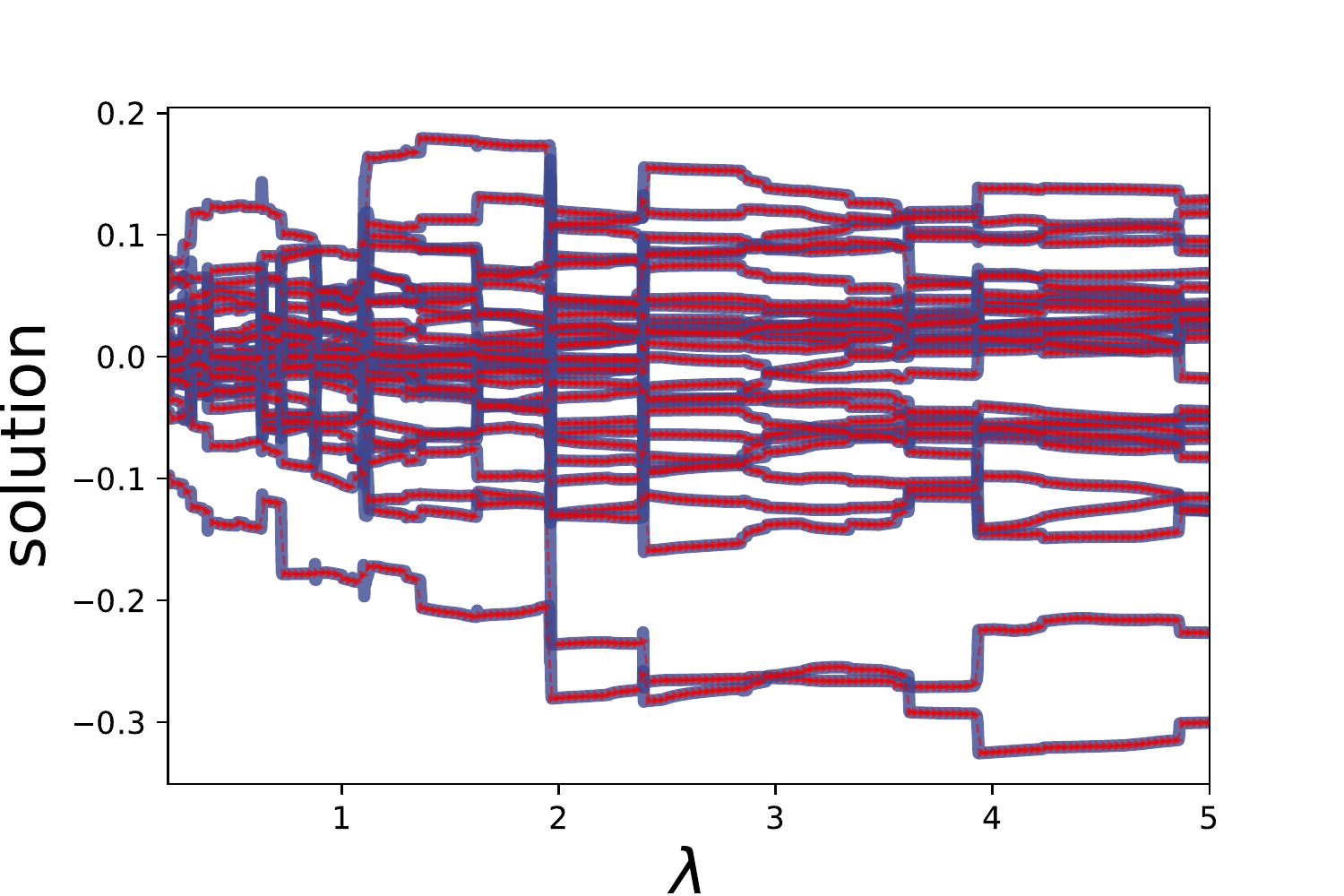} \\
       \includegraphics[width=0.33\linewidth]{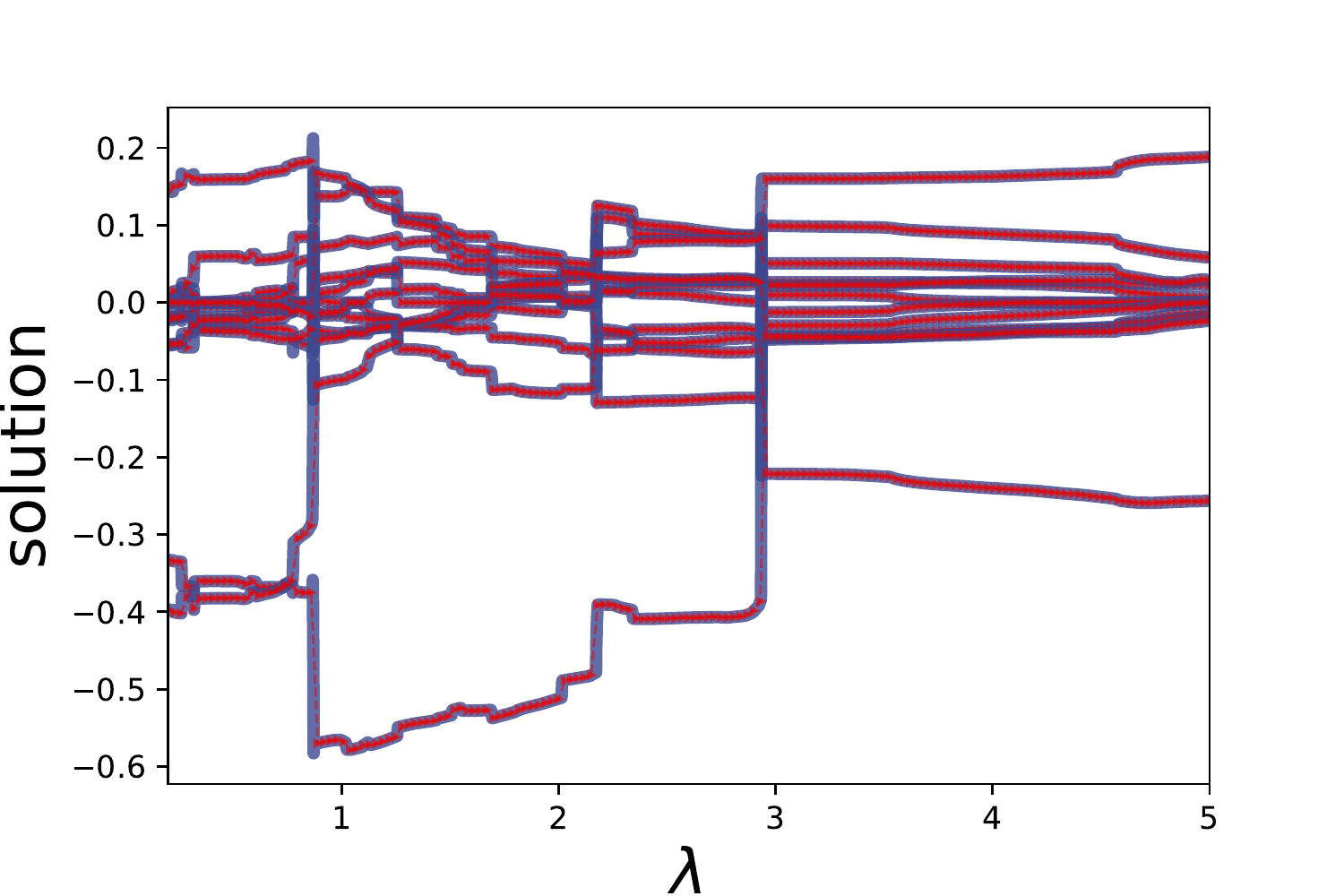} &         \includegraphics[width=0.33\linewidth]{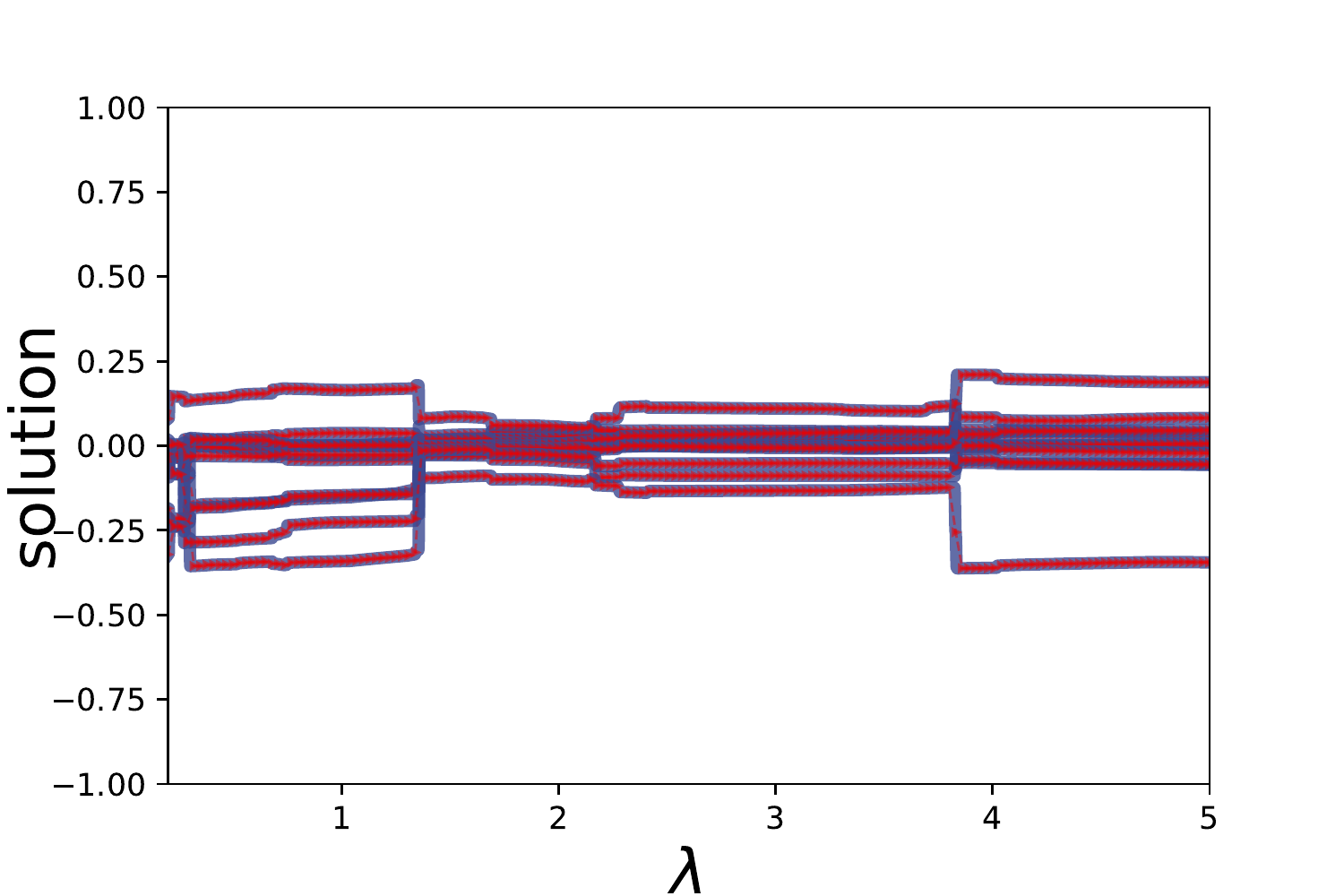} &          \includegraphics[width=0.33\linewidth]{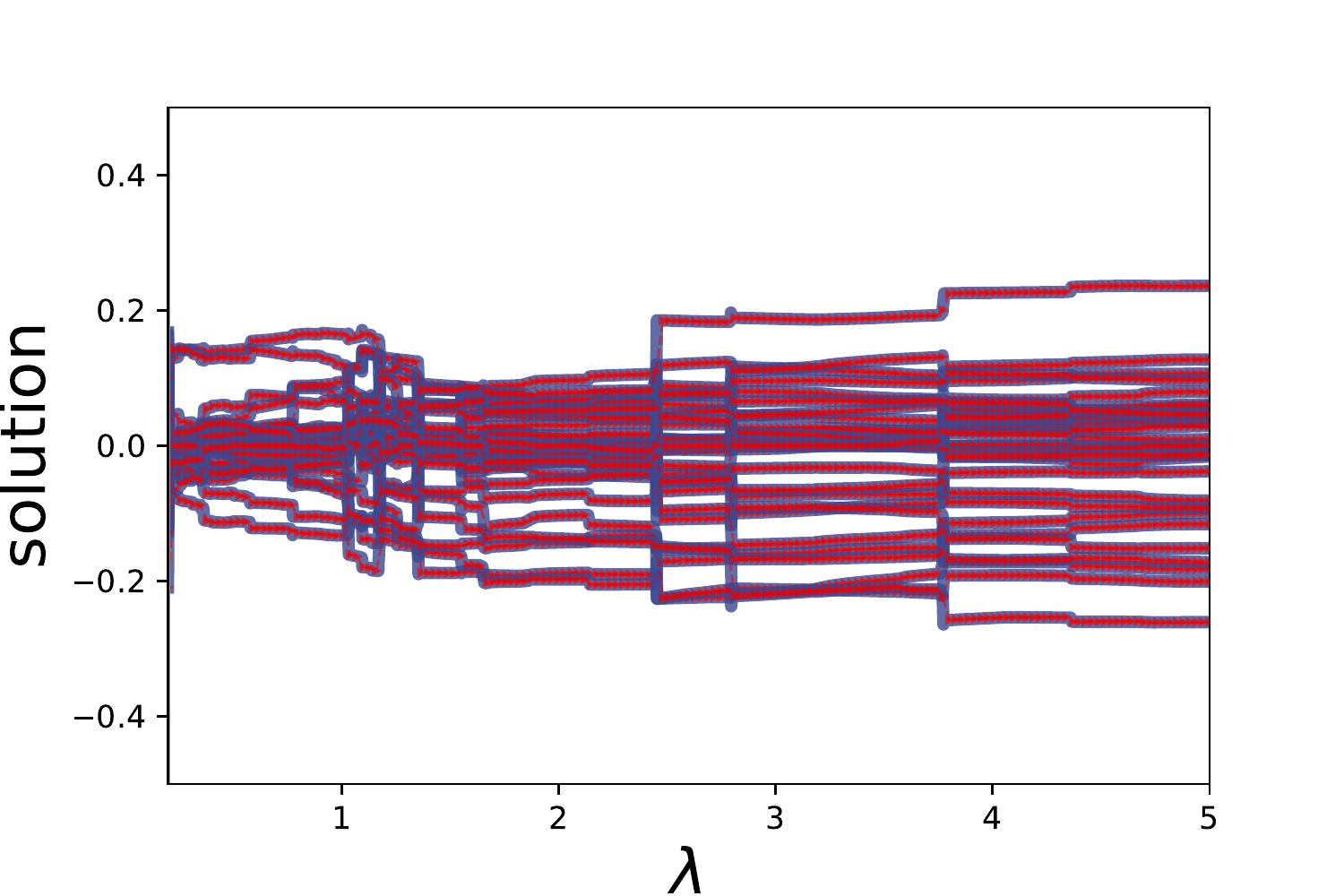} \\
       \includegraphics[width=0.33\linewidth]{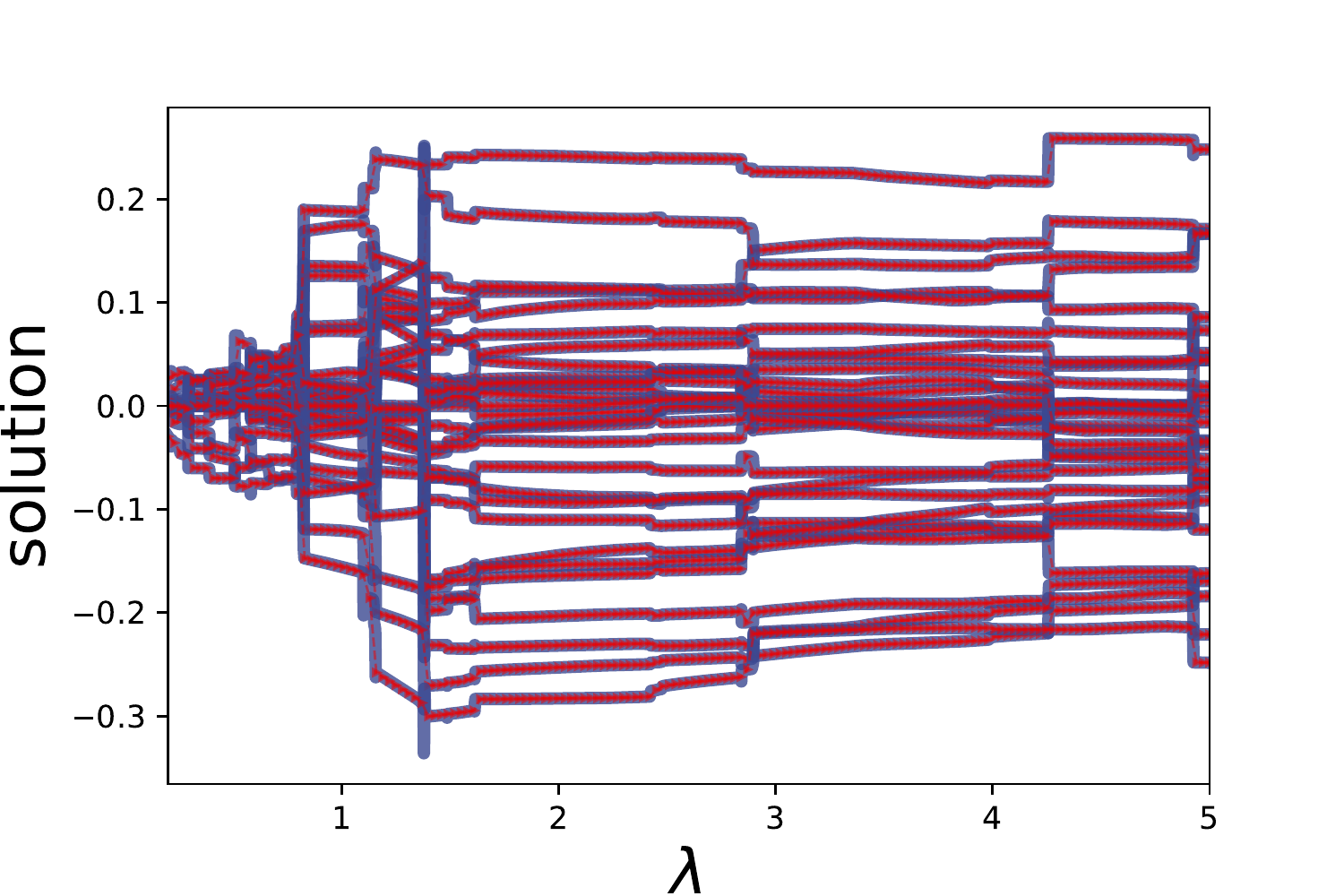} &      \includegraphics[width=0.33\linewidth]{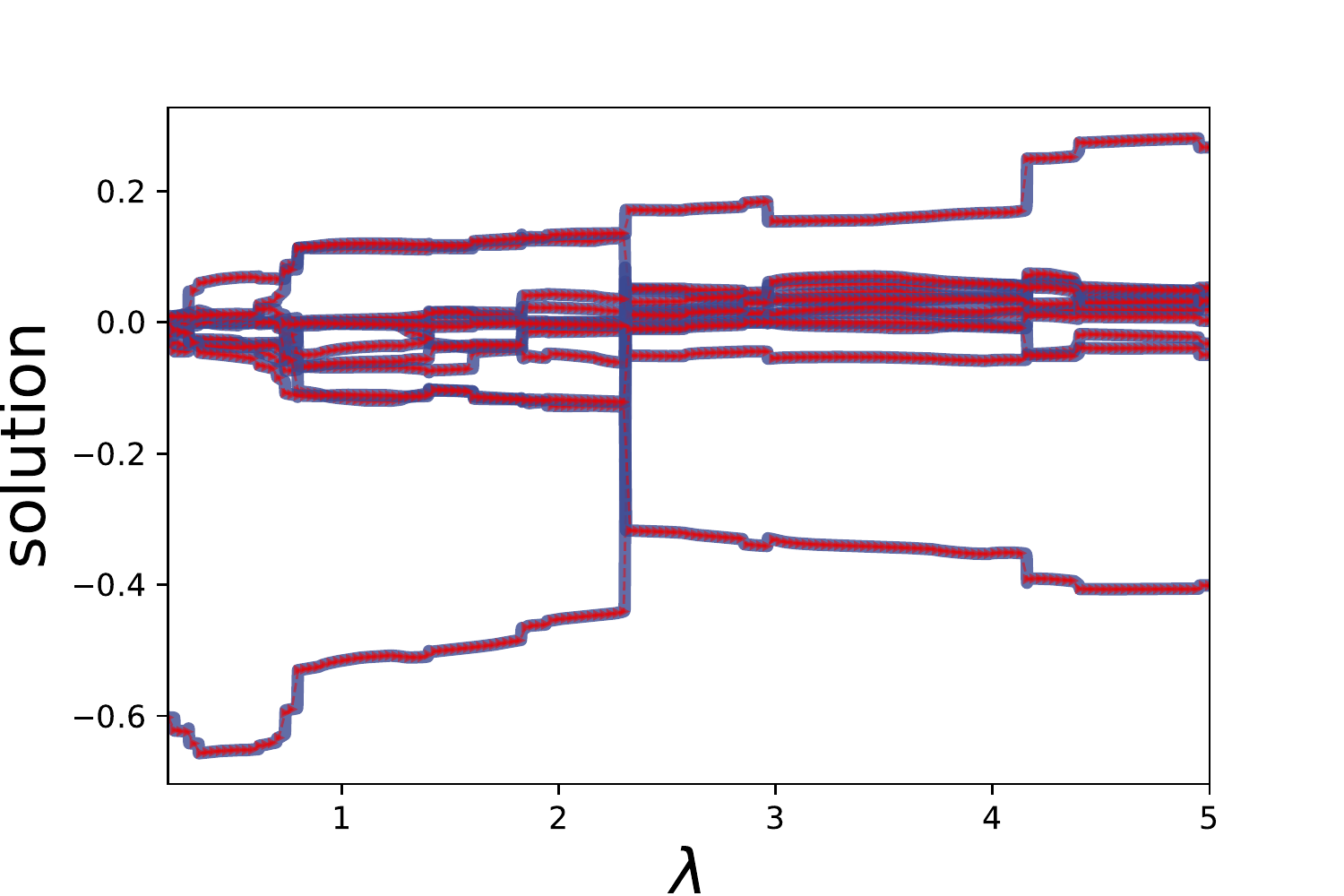} &        \includegraphics[width=0.33\linewidth]{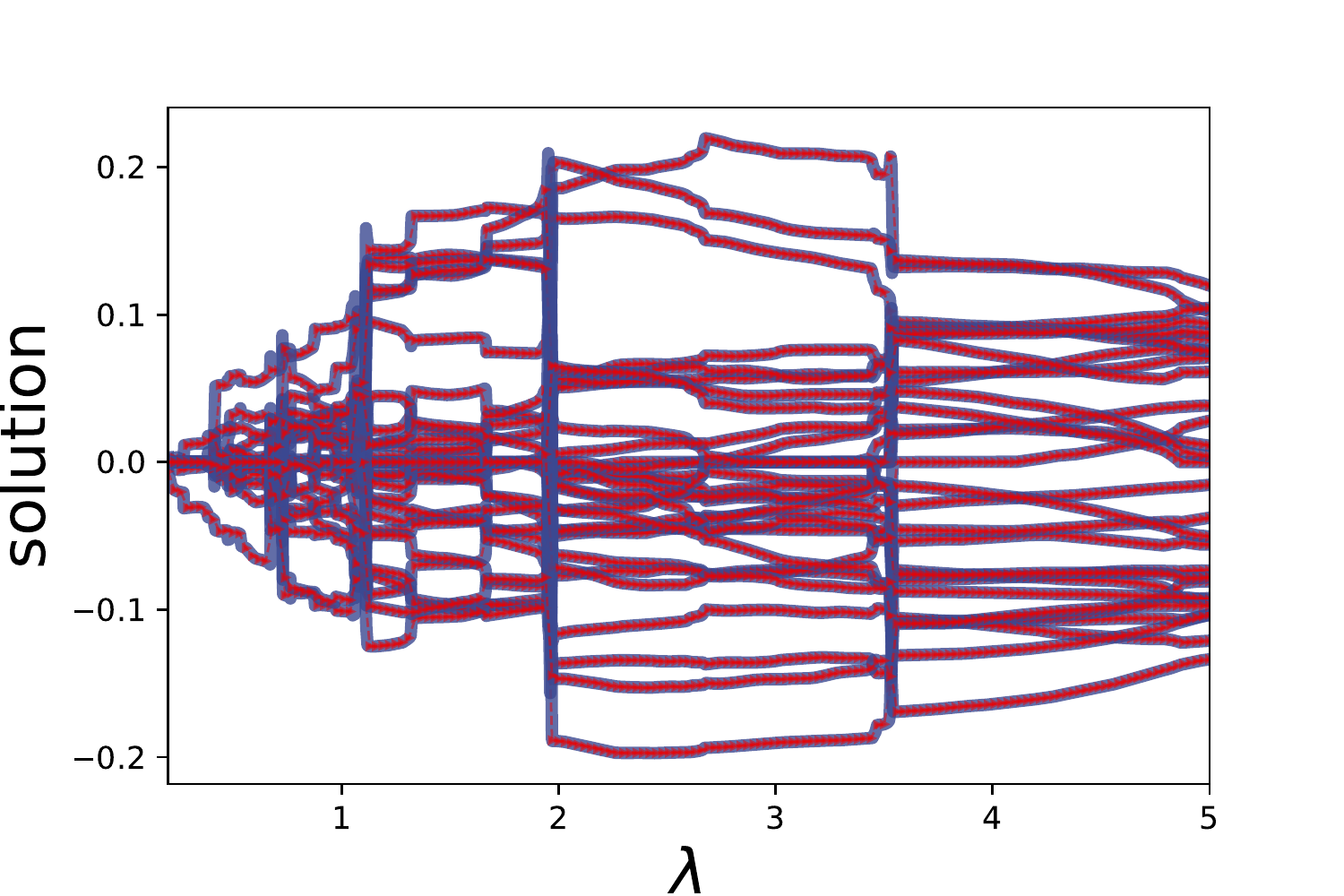} \\
       \includegraphics[width=0.33\linewidth]{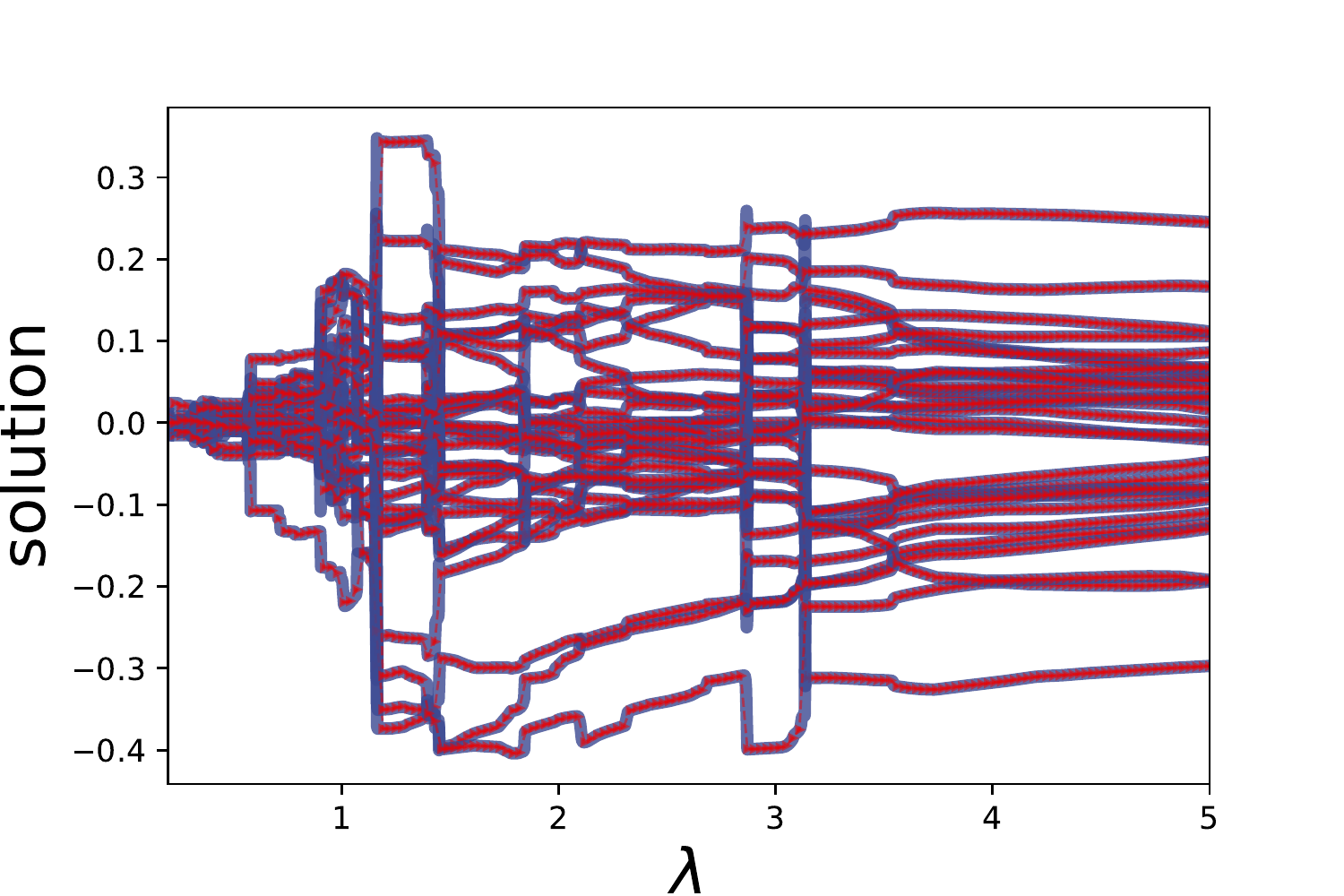} &         \includegraphics[width=0.33\linewidth]{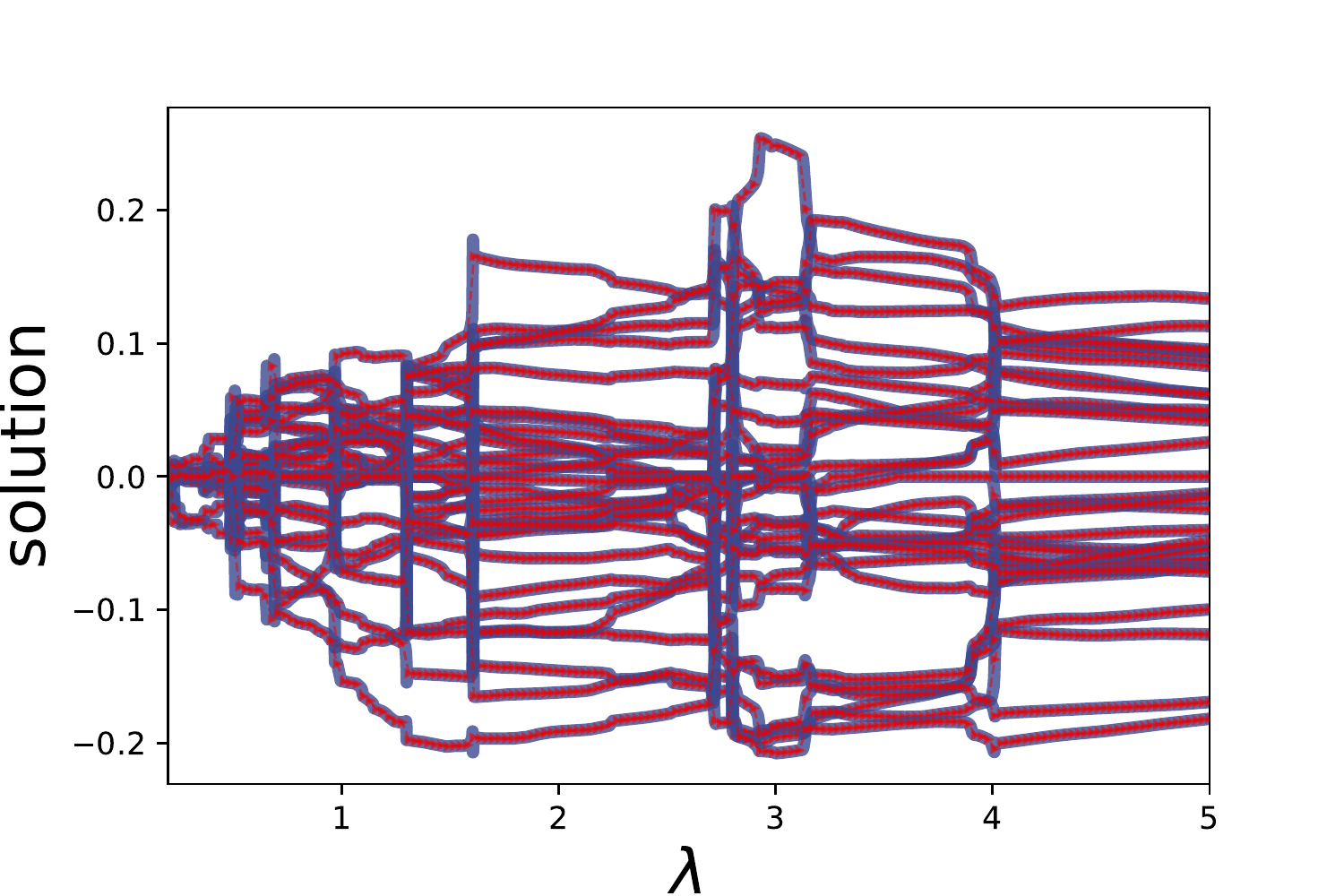} &          \includegraphics[width=0.33\linewidth]{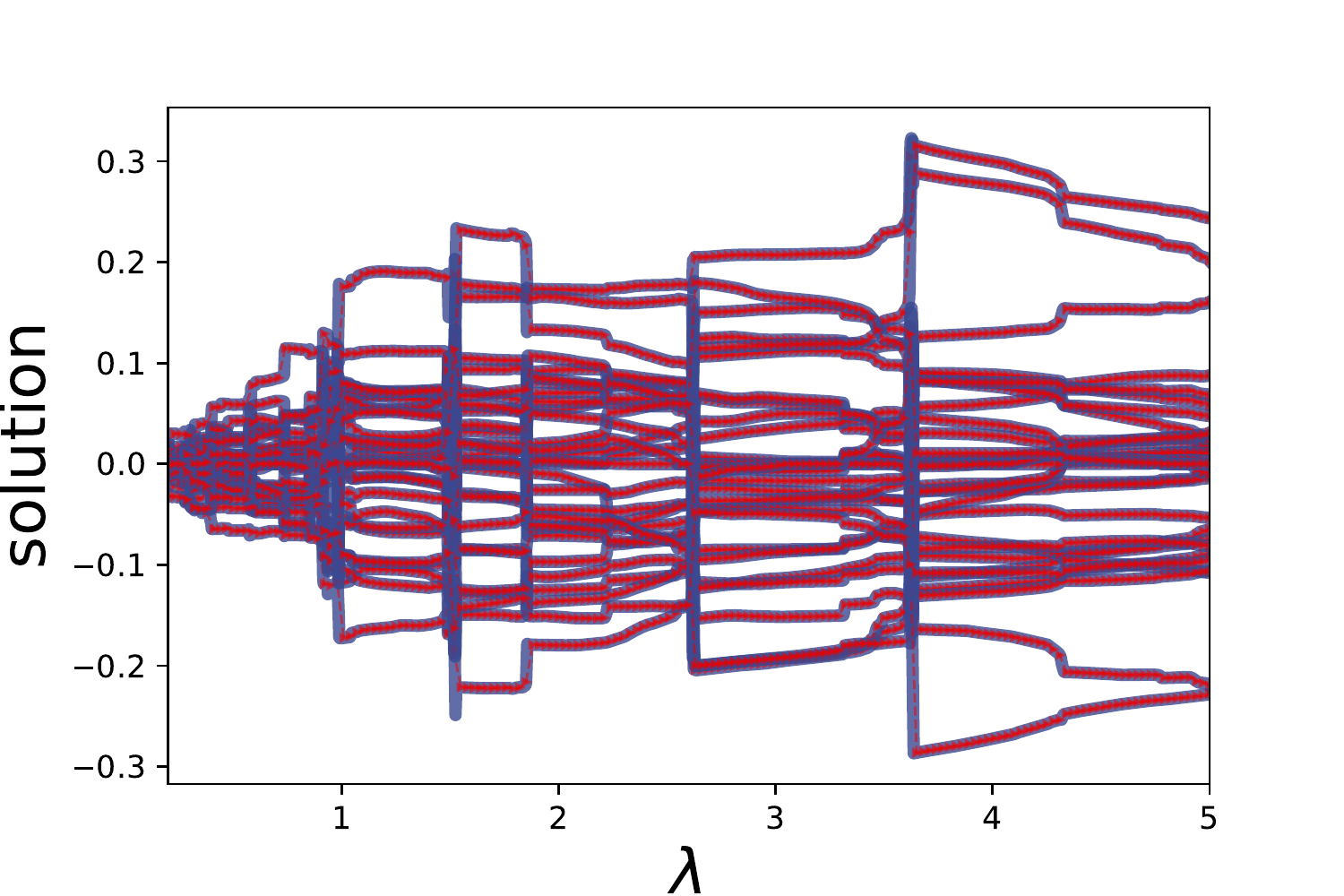} \\
\bottomrule
\end{tabular}
\caption{Age-path of SVM with different parameters and datasets. The first two rows of subfigures illustrate age-path using the linear SP-regularizer while the last two rows of subfigures show age-path using the mixture SP-regularizer. For experiments in the first and third row, the $\alpha=0.02$. For experiments in the second and fourth row, the $\alpha=0.04$.  }
\label{path_1}
\end{figure}
\begin{figure}
\begin{tabular}{lll}
\toprule
      \includegraphics[width=0.33\linewidth]{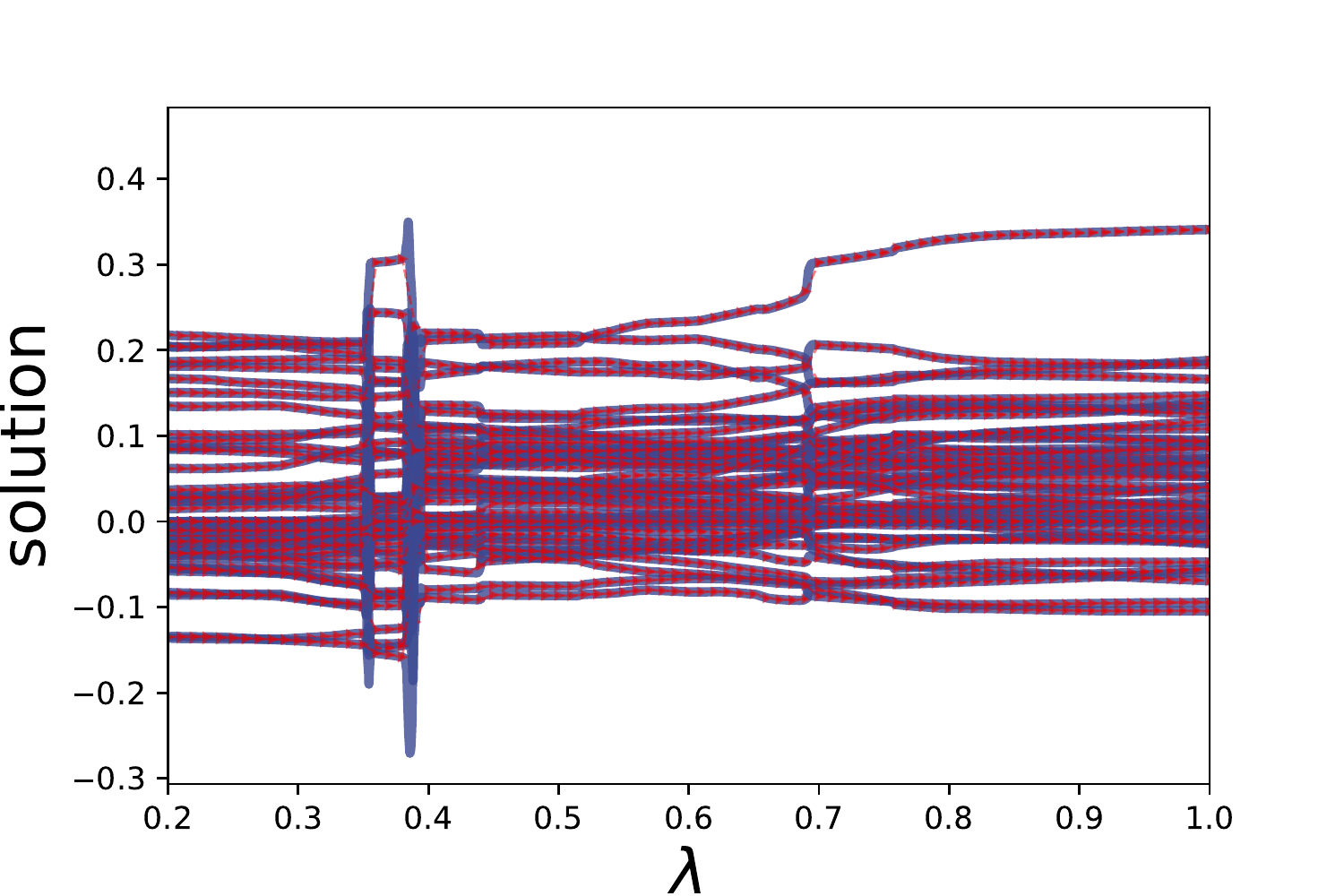} &      \includegraphics[width=0.33\linewidth]{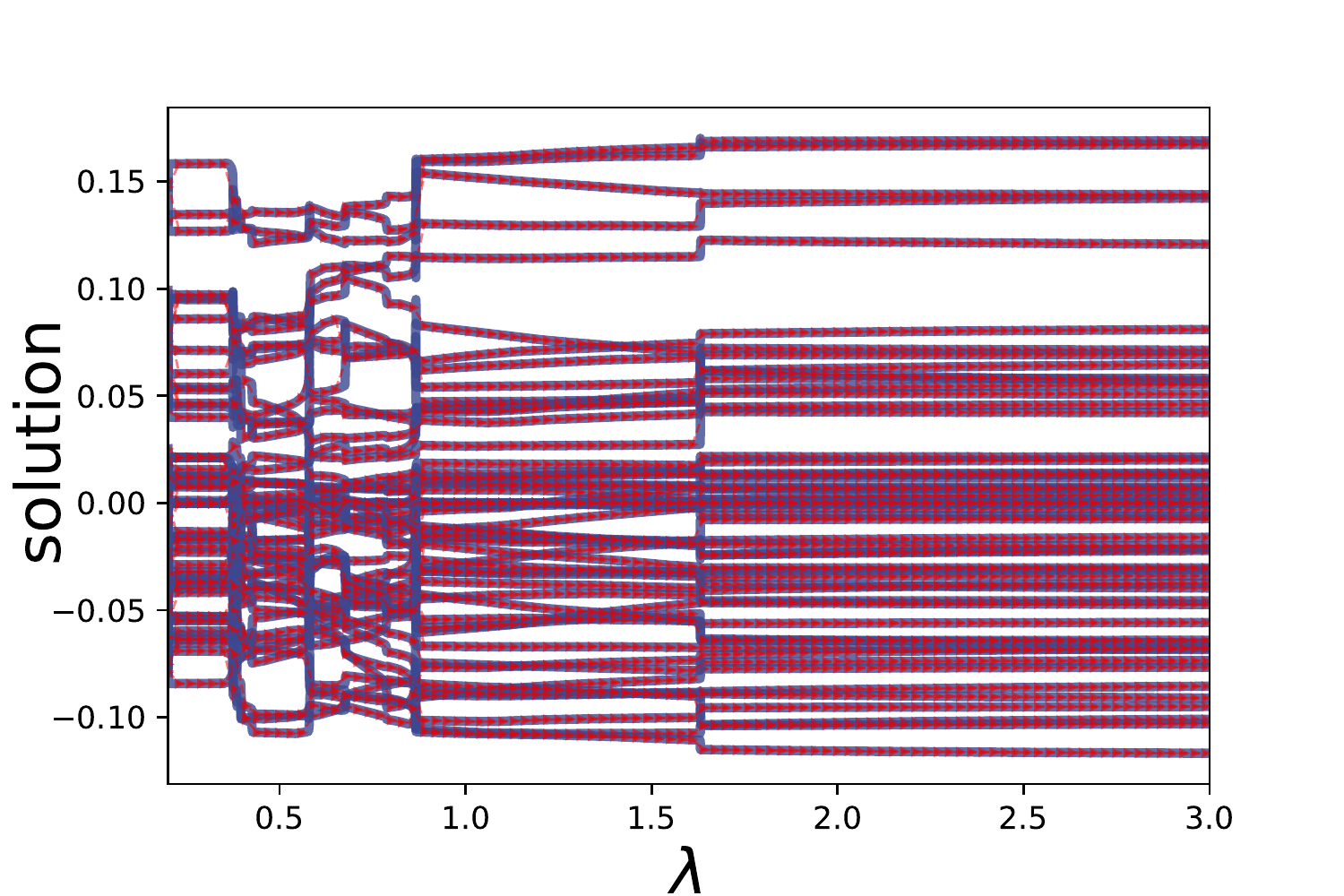} &        \includegraphics[width=0.33\linewidth]{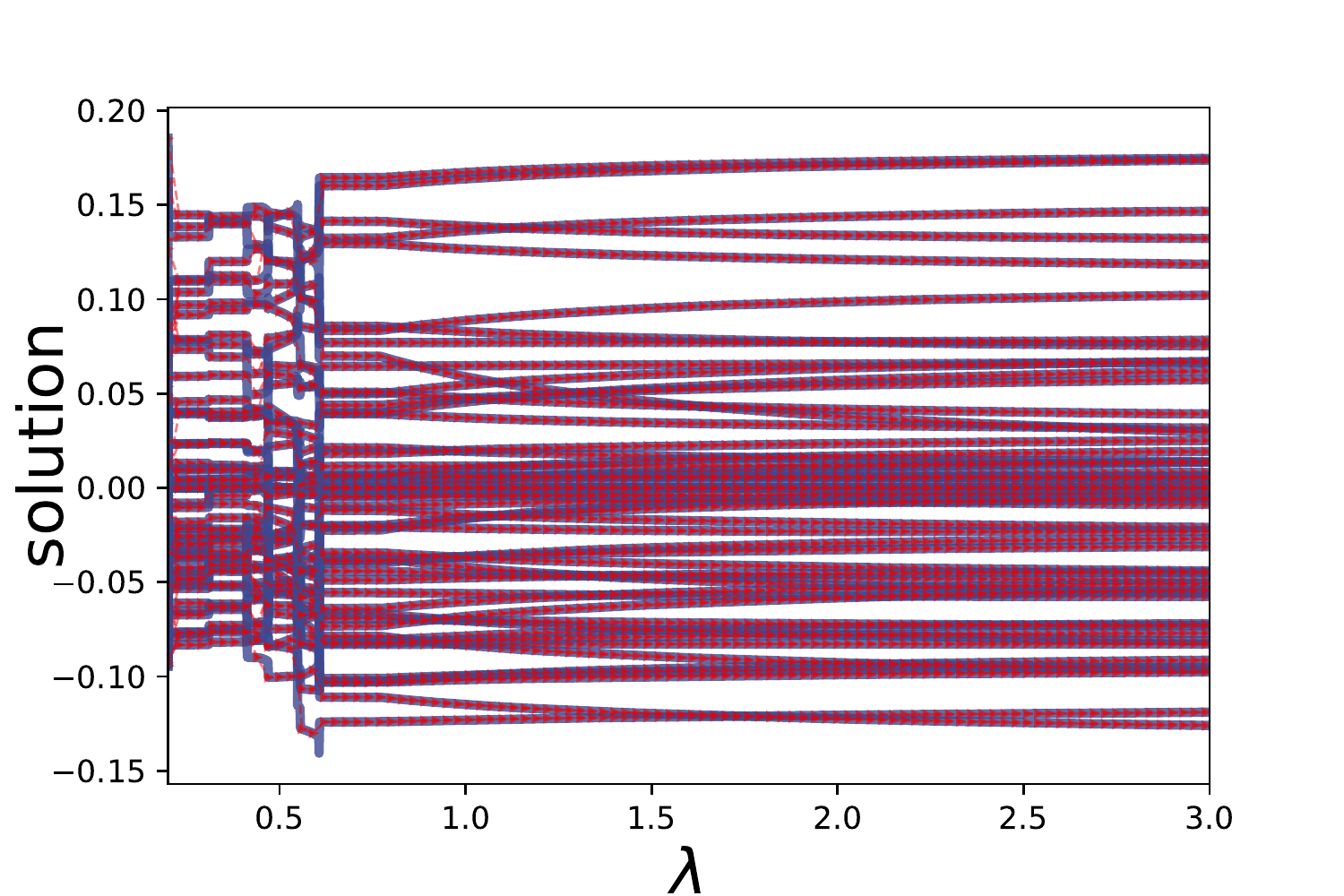} \\
       \includegraphics[width=0.33\linewidth]{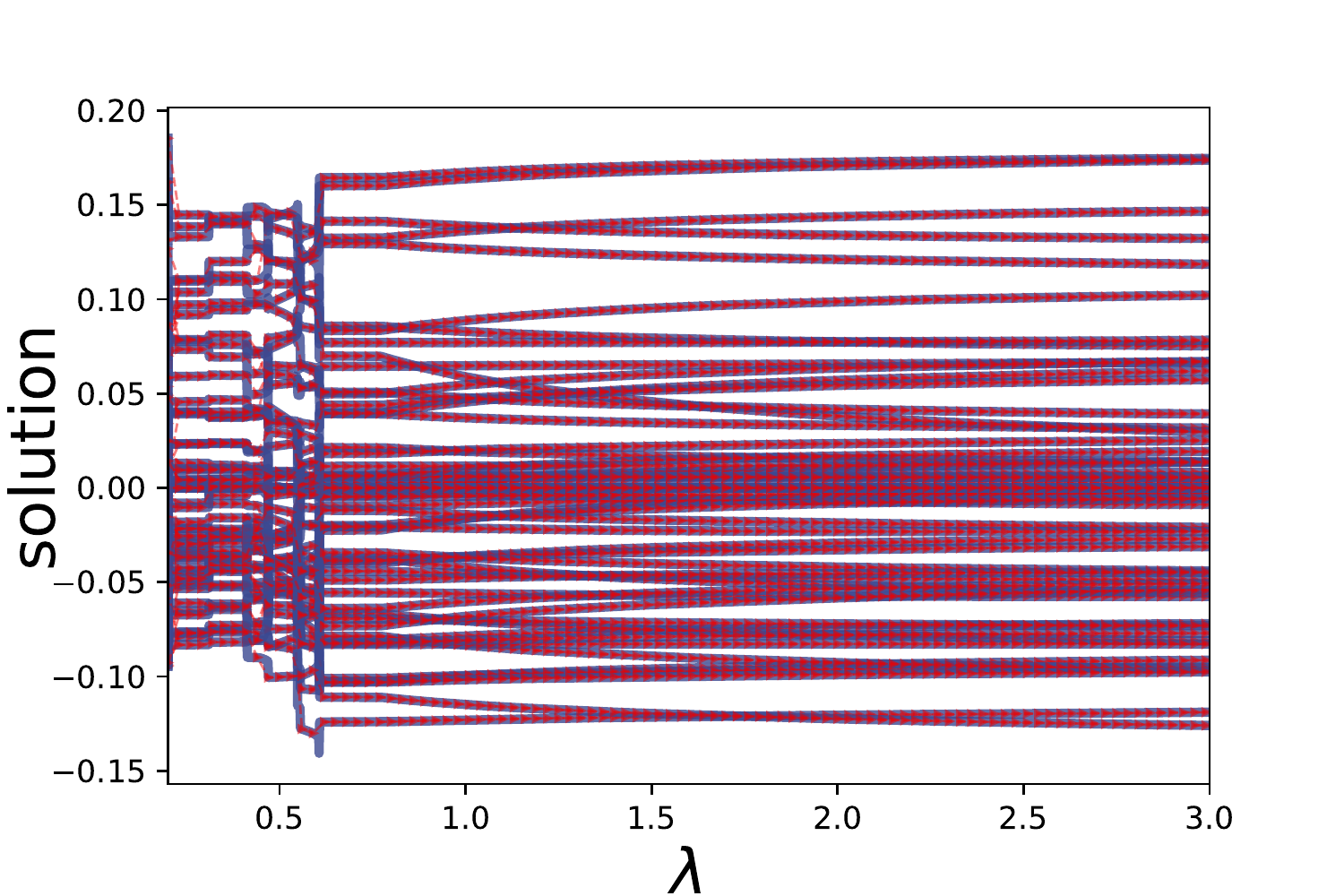} &         \includegraphics[width=0.33\linewidth]{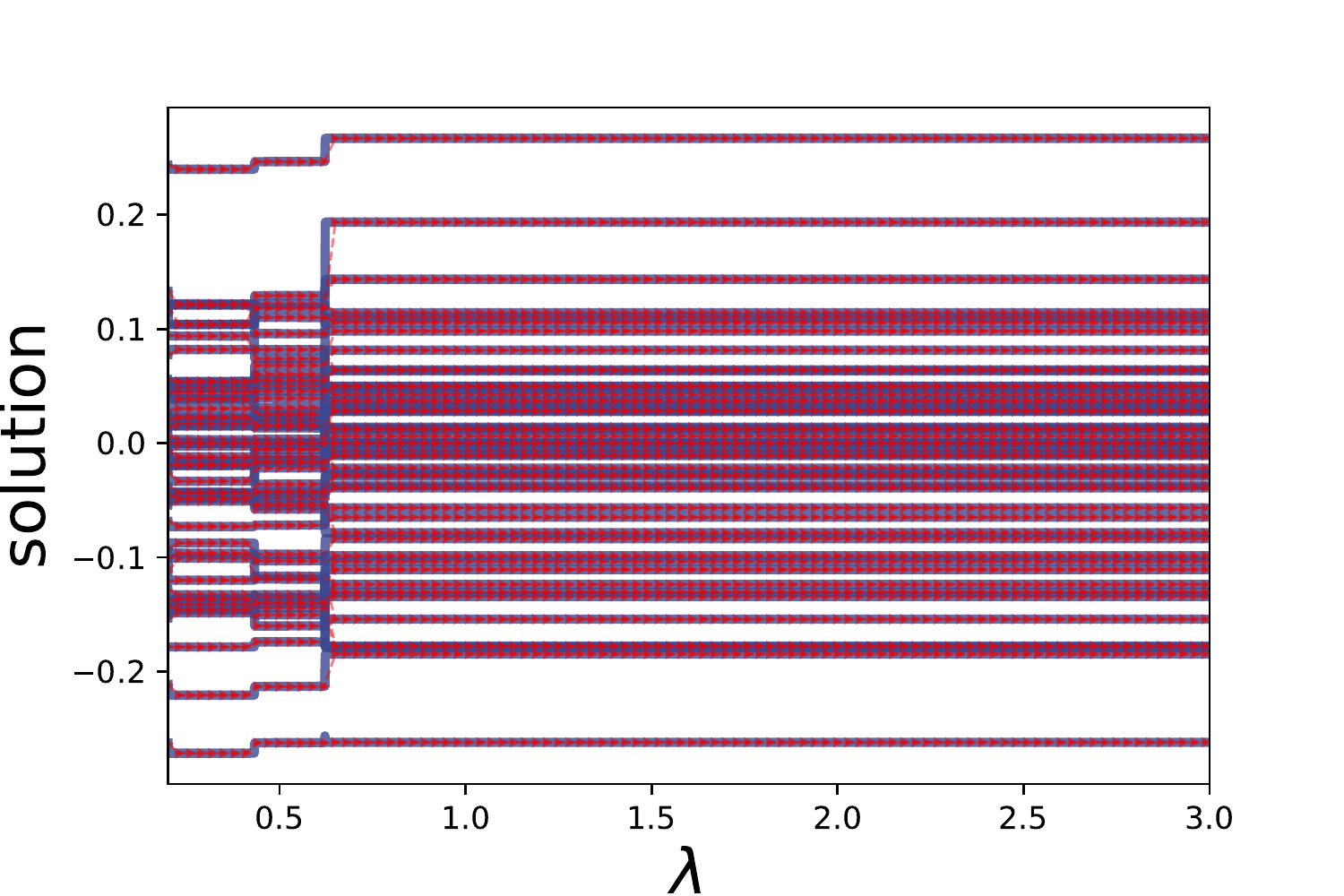} &          \includegraphics[width=0.33\linewidth]{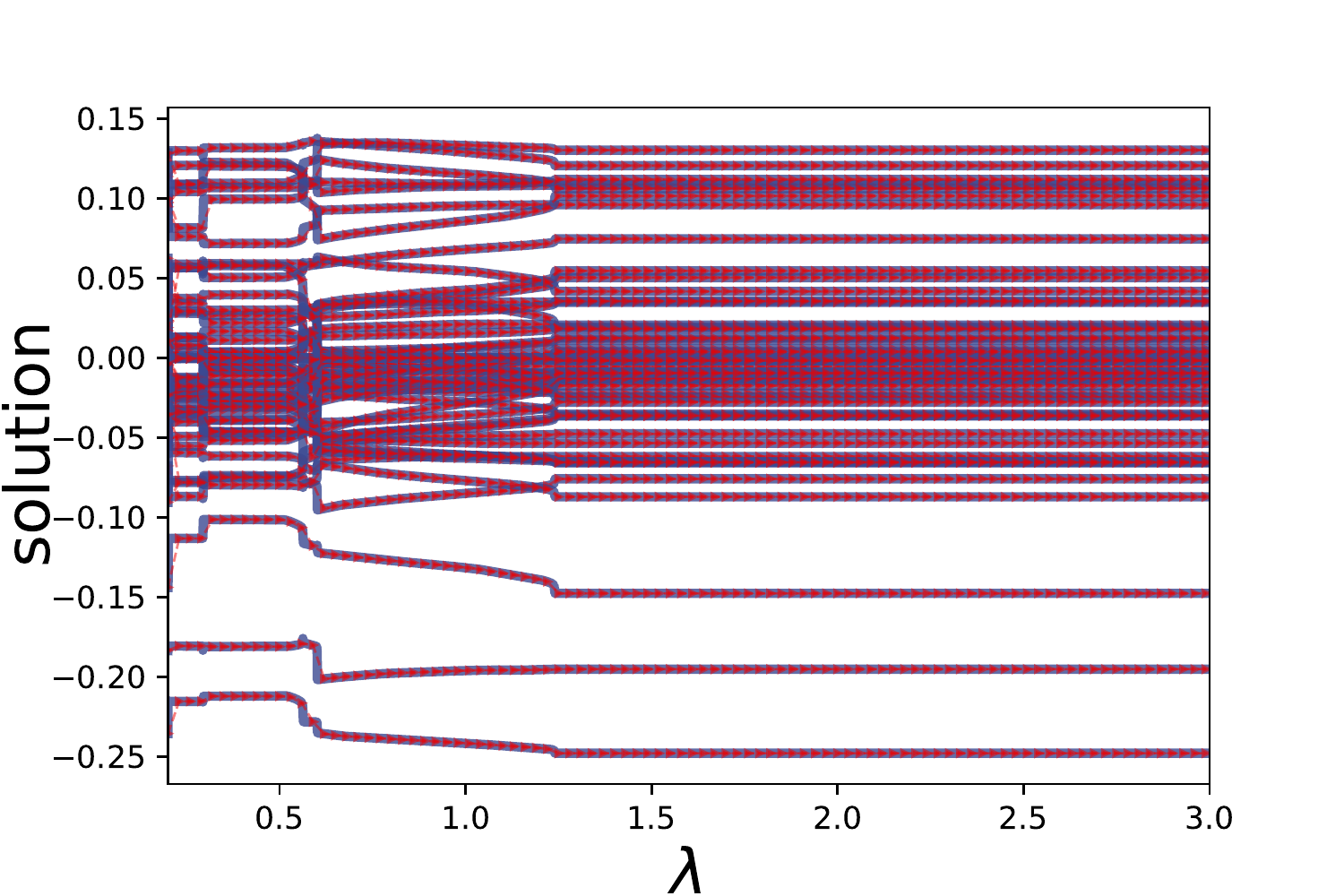} \\
       \includegraphics[width=0.33\linewidth]{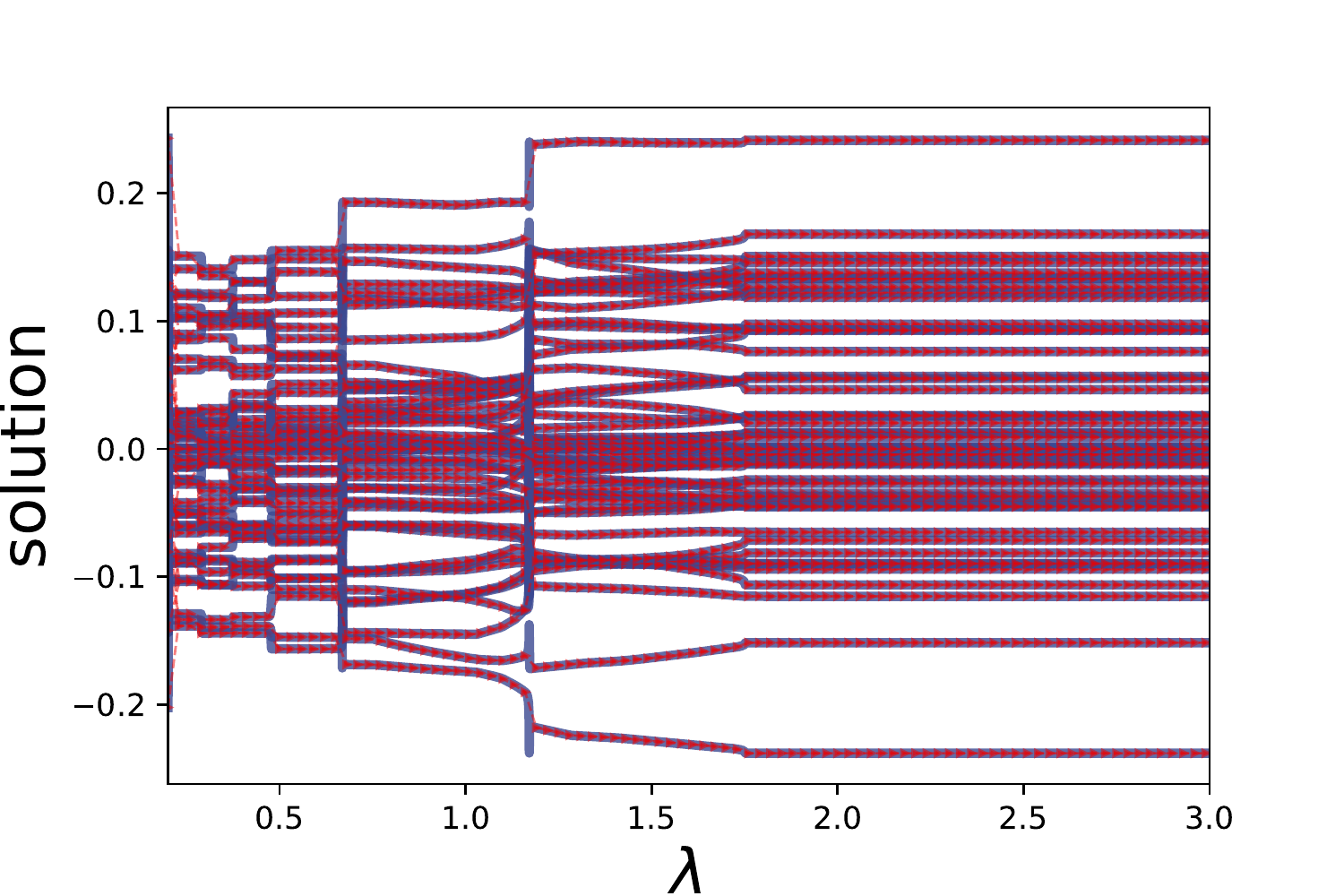} &      \includegraphics[width=0.33\linewidth]{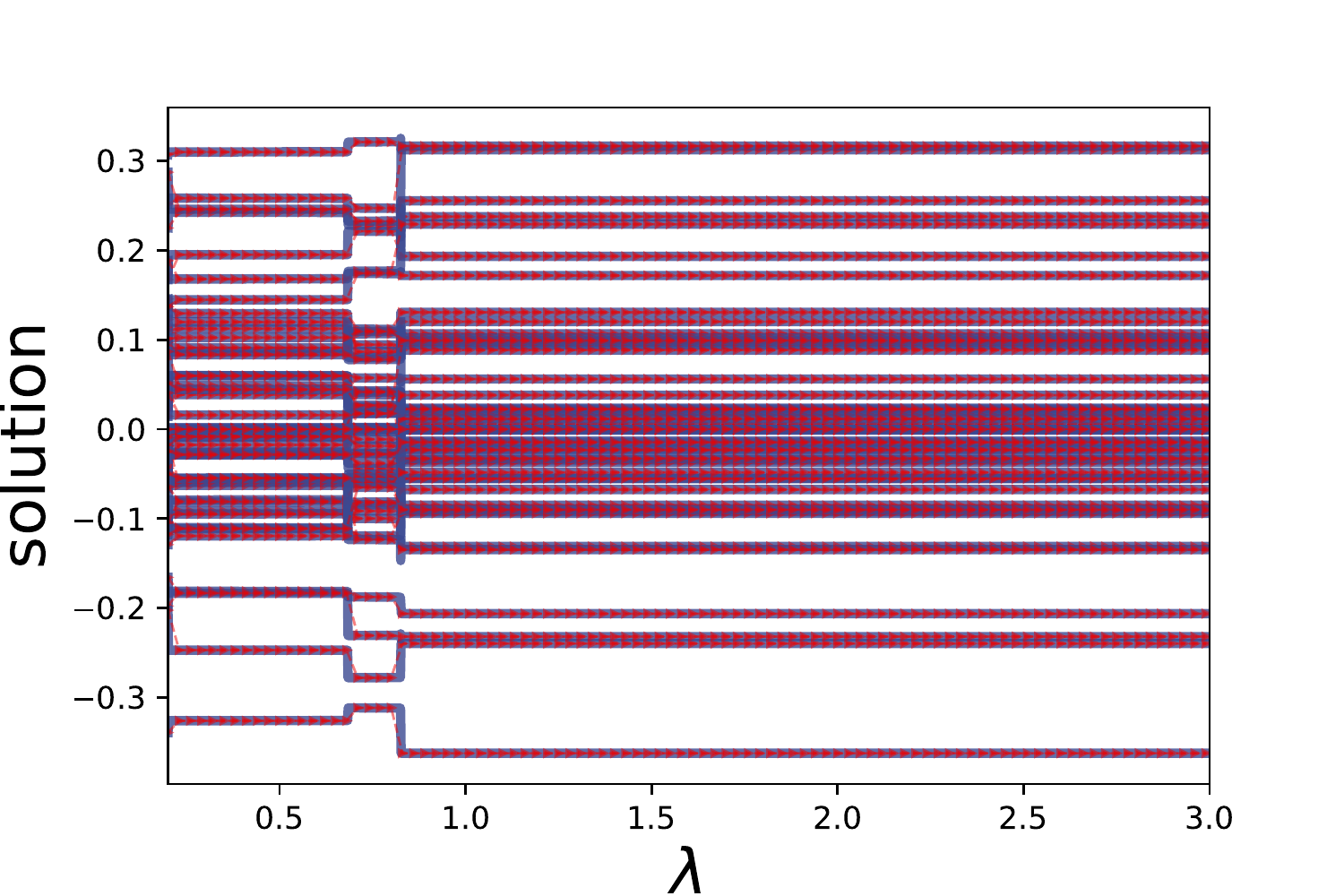} &        \includegraphics[width=0.33\linewidth]{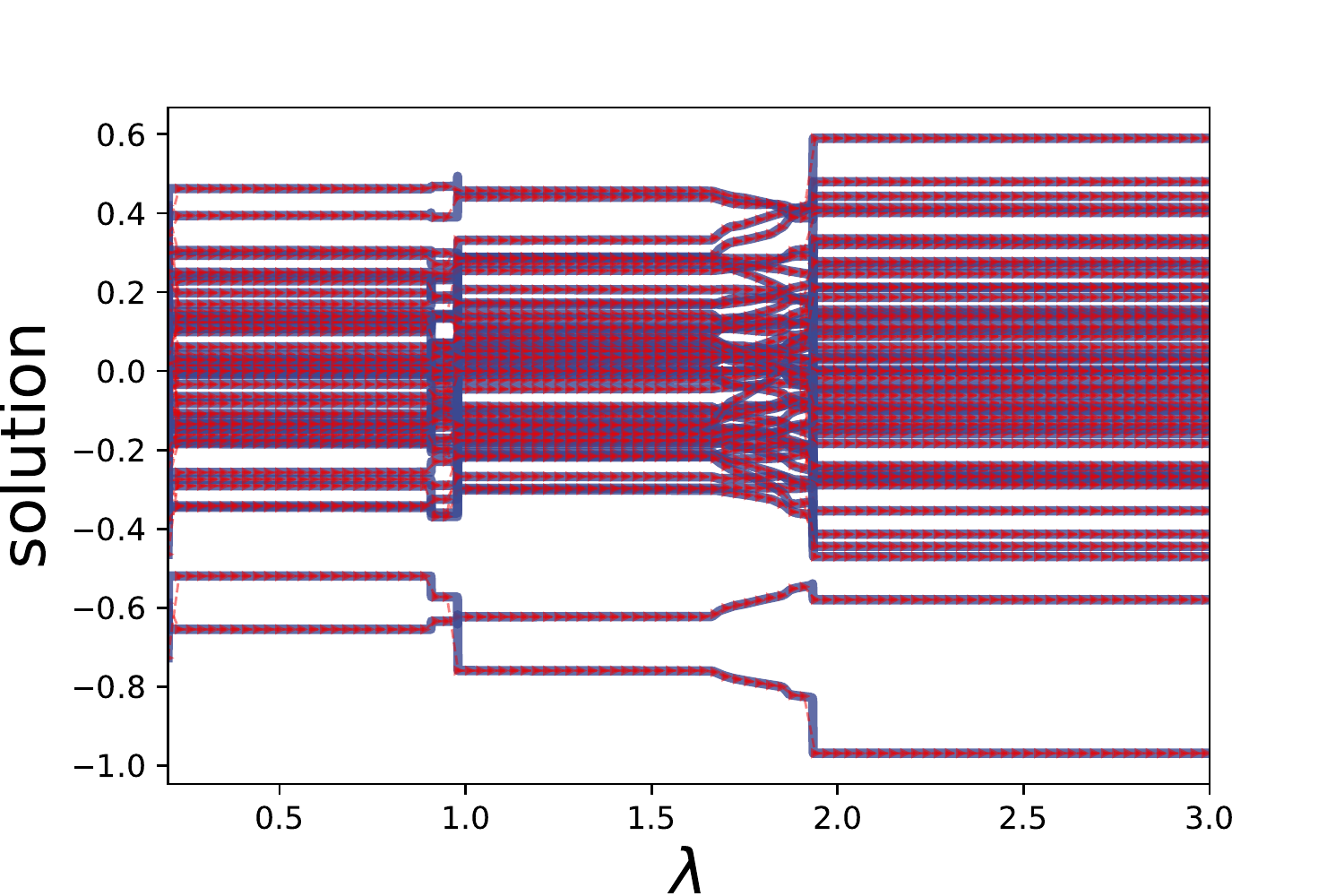} \\
       \includegraphics[width=0.33\linewidth]{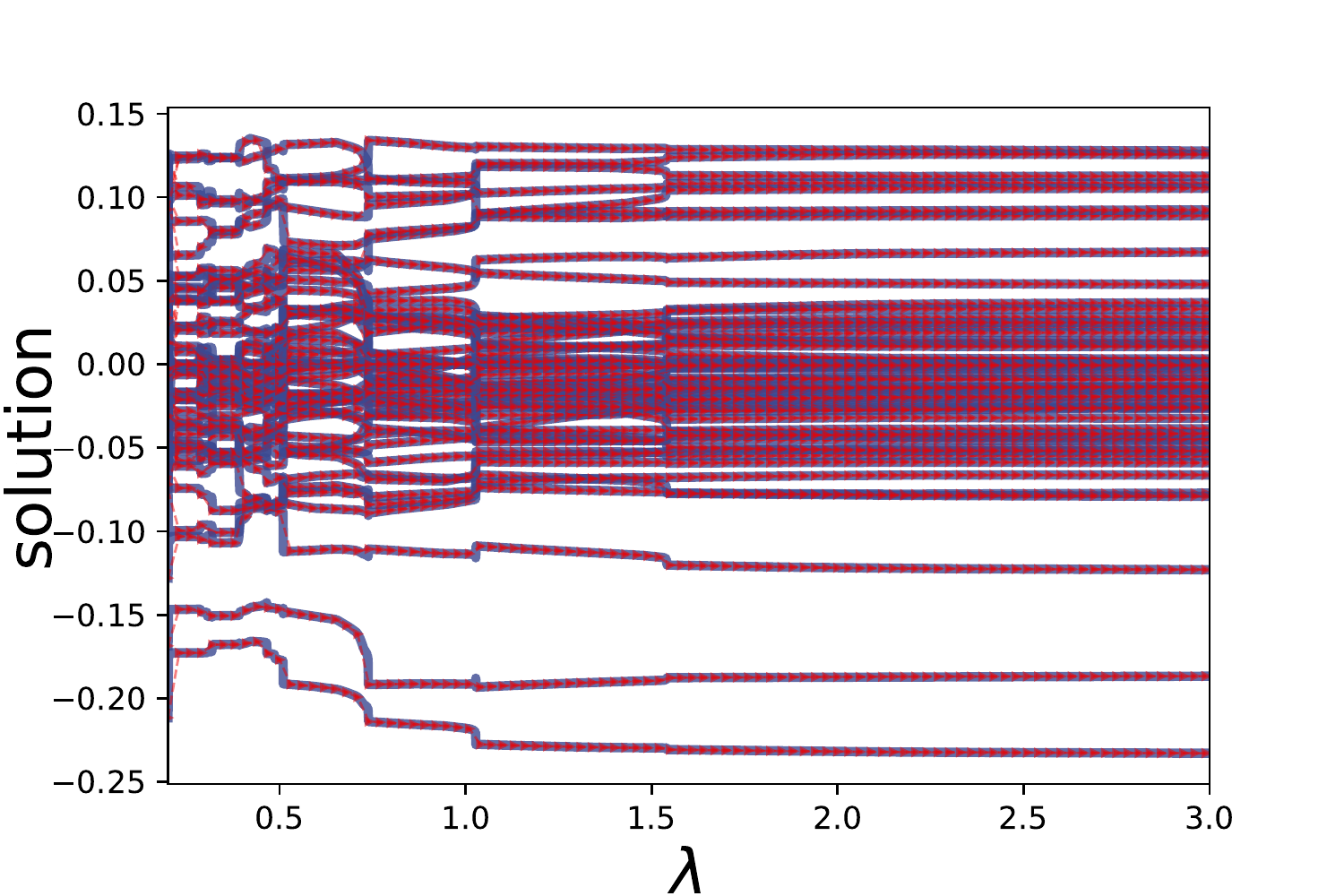} &         \includegraphics[width=0.33\linewidth]{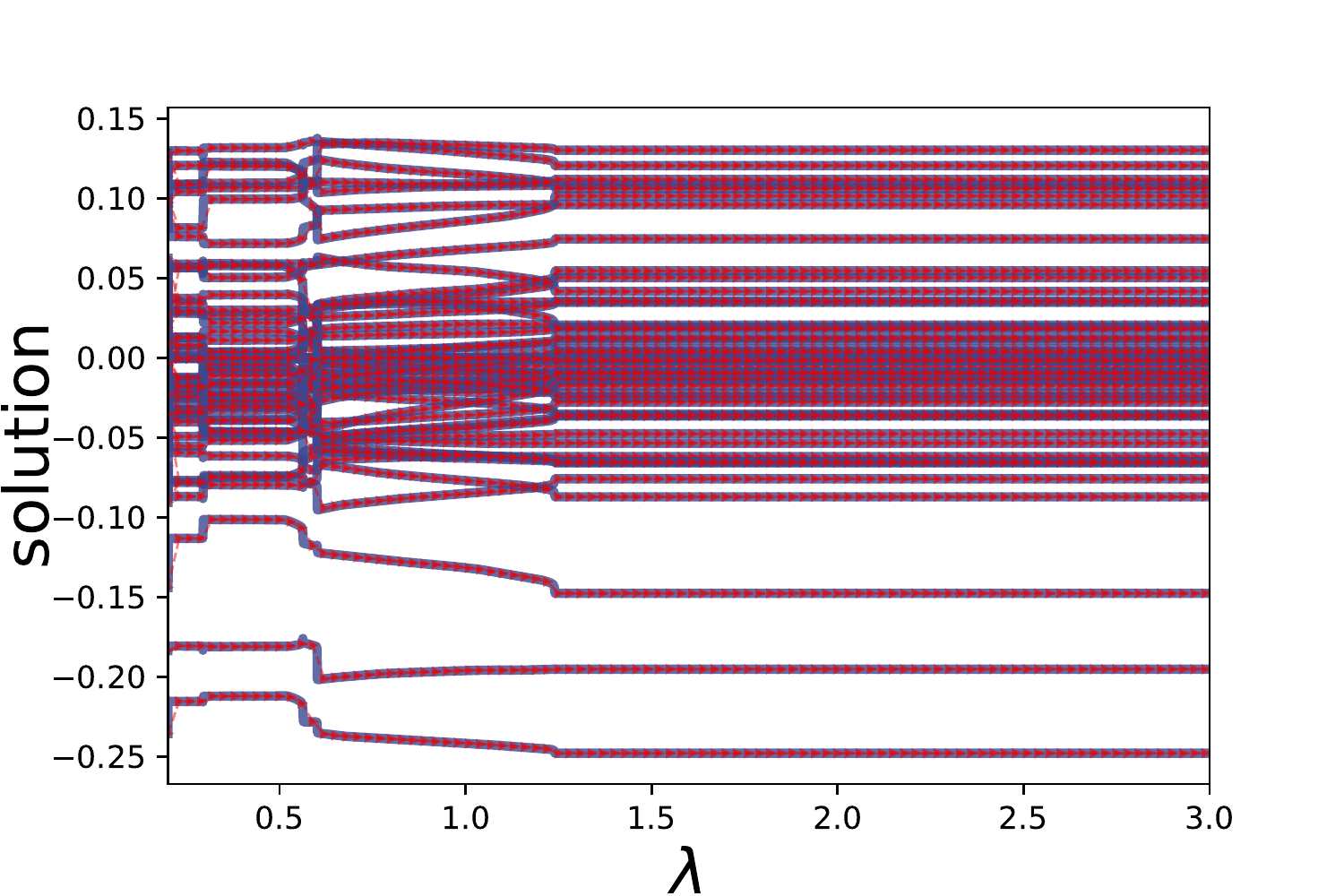} &          \includegraphics[width=0.33\linewidth]{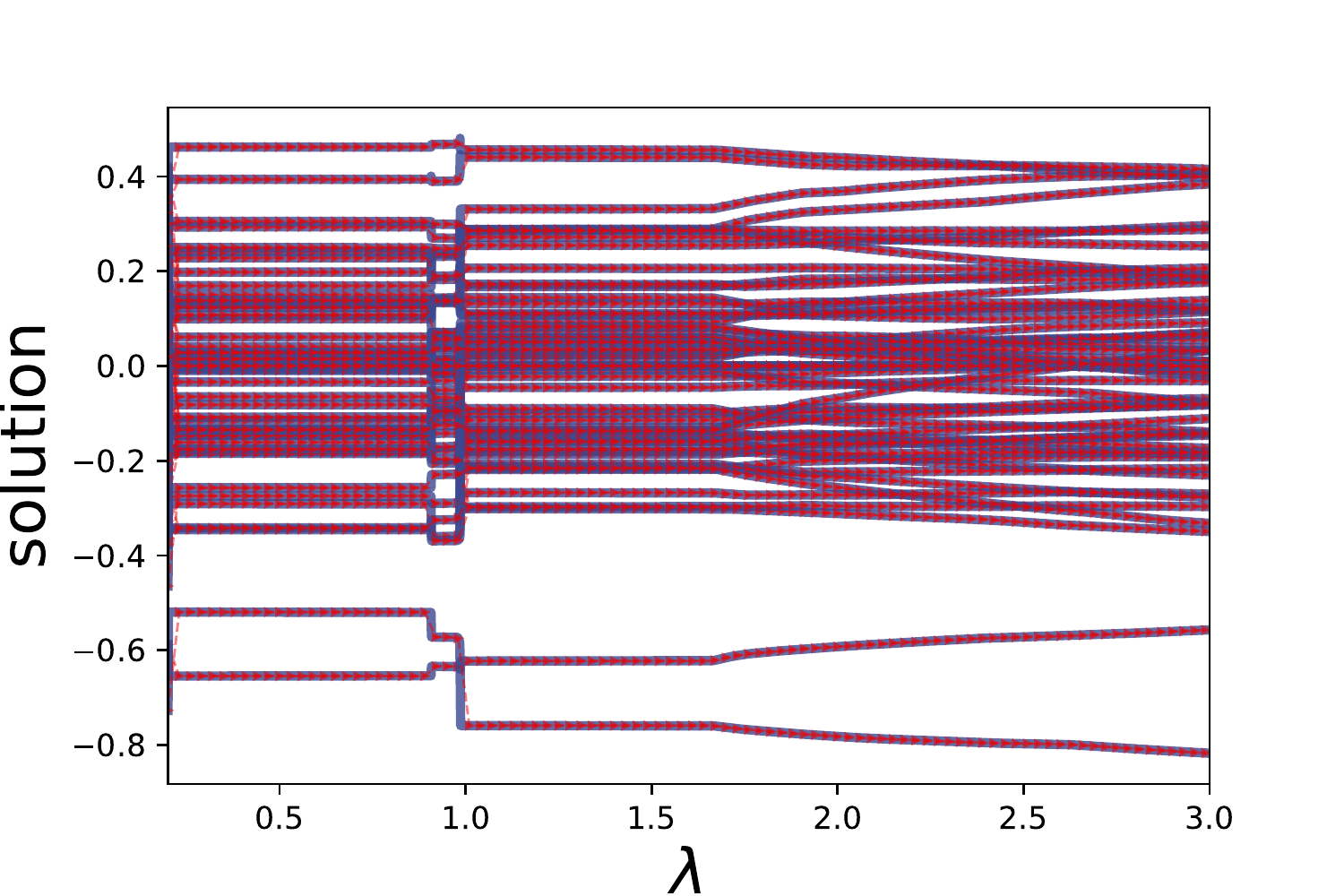} \\
\bottomrule
\end{tabular}
\caption{Age-path of Lasso with different parameters and datasets. The first two rows of subfigures illustrate age-path using the linear SP-regularizer while the last two rows of subfigures show age-path using the mixture SP-regularizer. For experiments in the first and third row, the $C=0.02$. For experiments in the second and fourth row, the $C=0.04$.}
\label{path_2}
\end{figure}
\begin{figure}
\begin{tabular}{lll}
\toprule
      \includegraphics[width=0.33\linewidth]{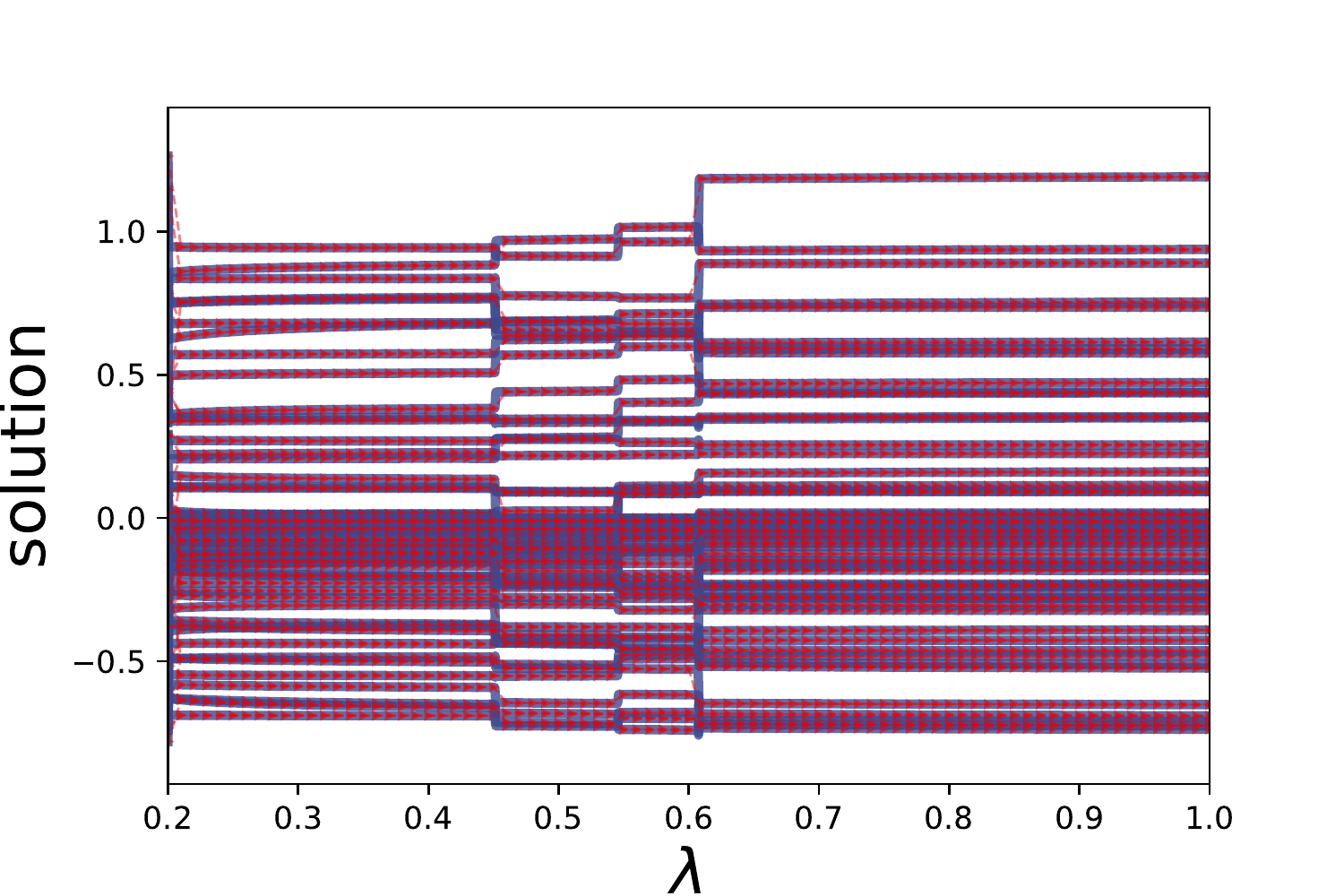} &      \includegraphics[width=0.33\linewidth]{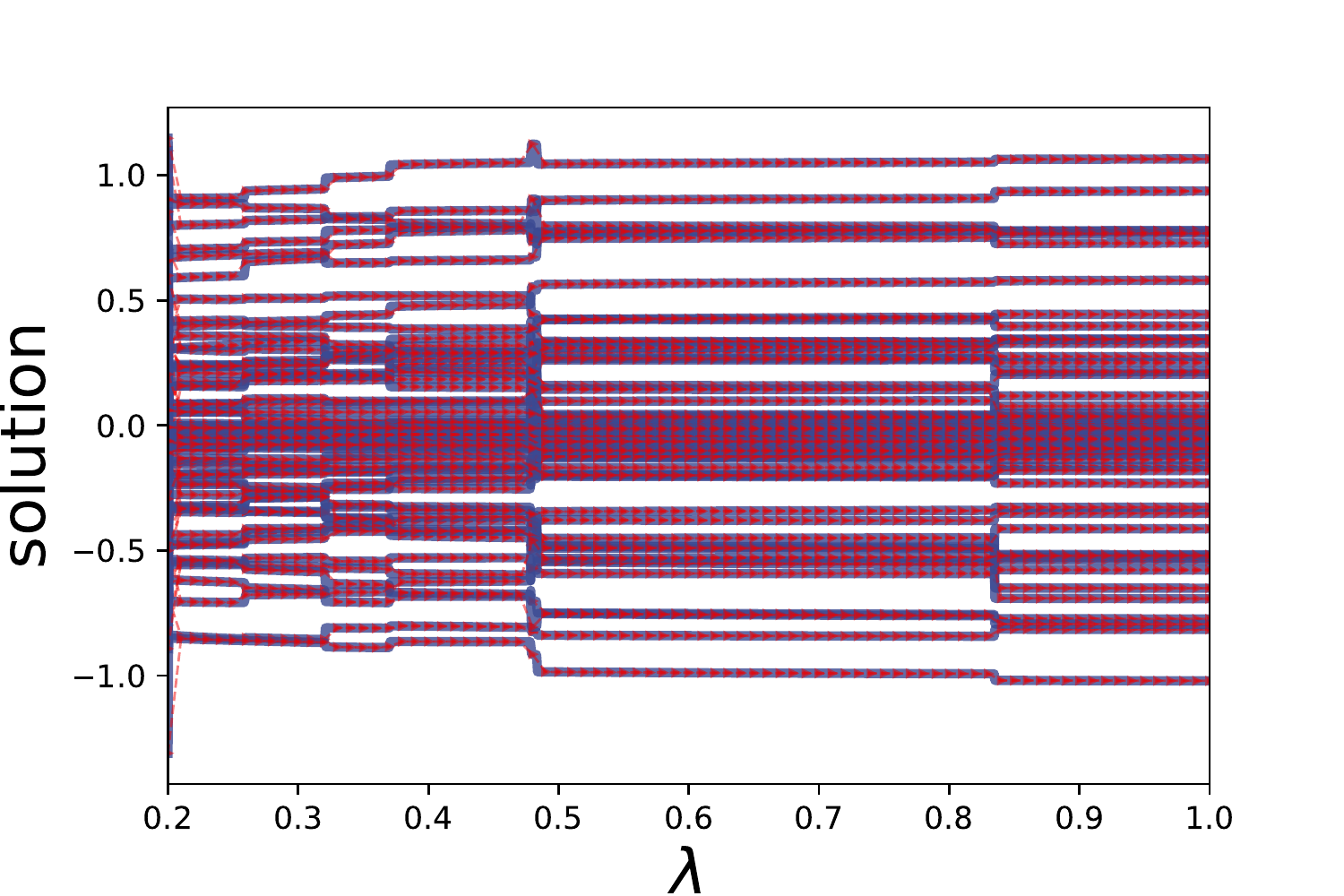} &        \includegraphics[width=0.33\linewidth]{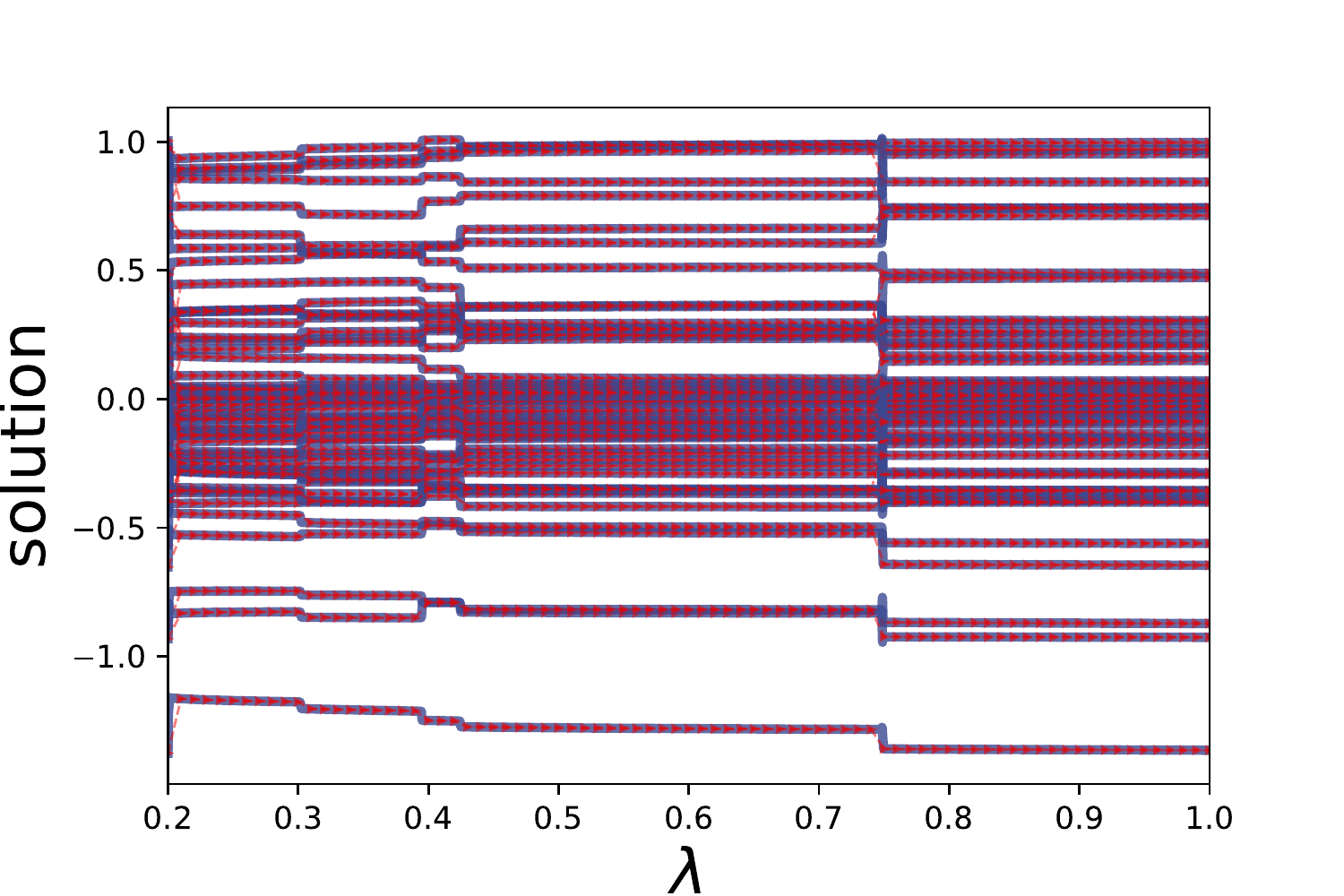} \\
       \includegraphics[width=0.33\linewidth]{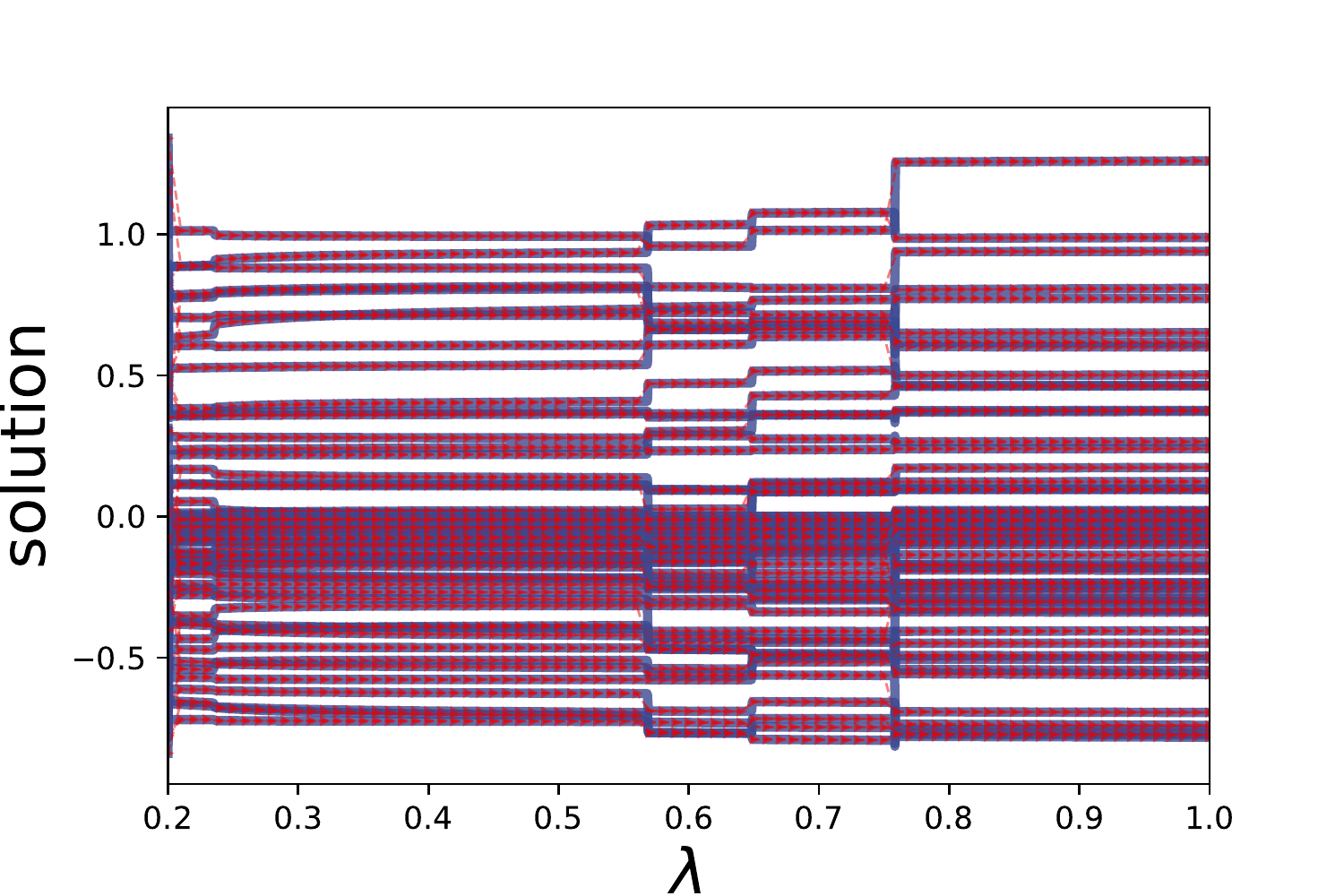} &         \includegraphics[width=0.33\linewidth]{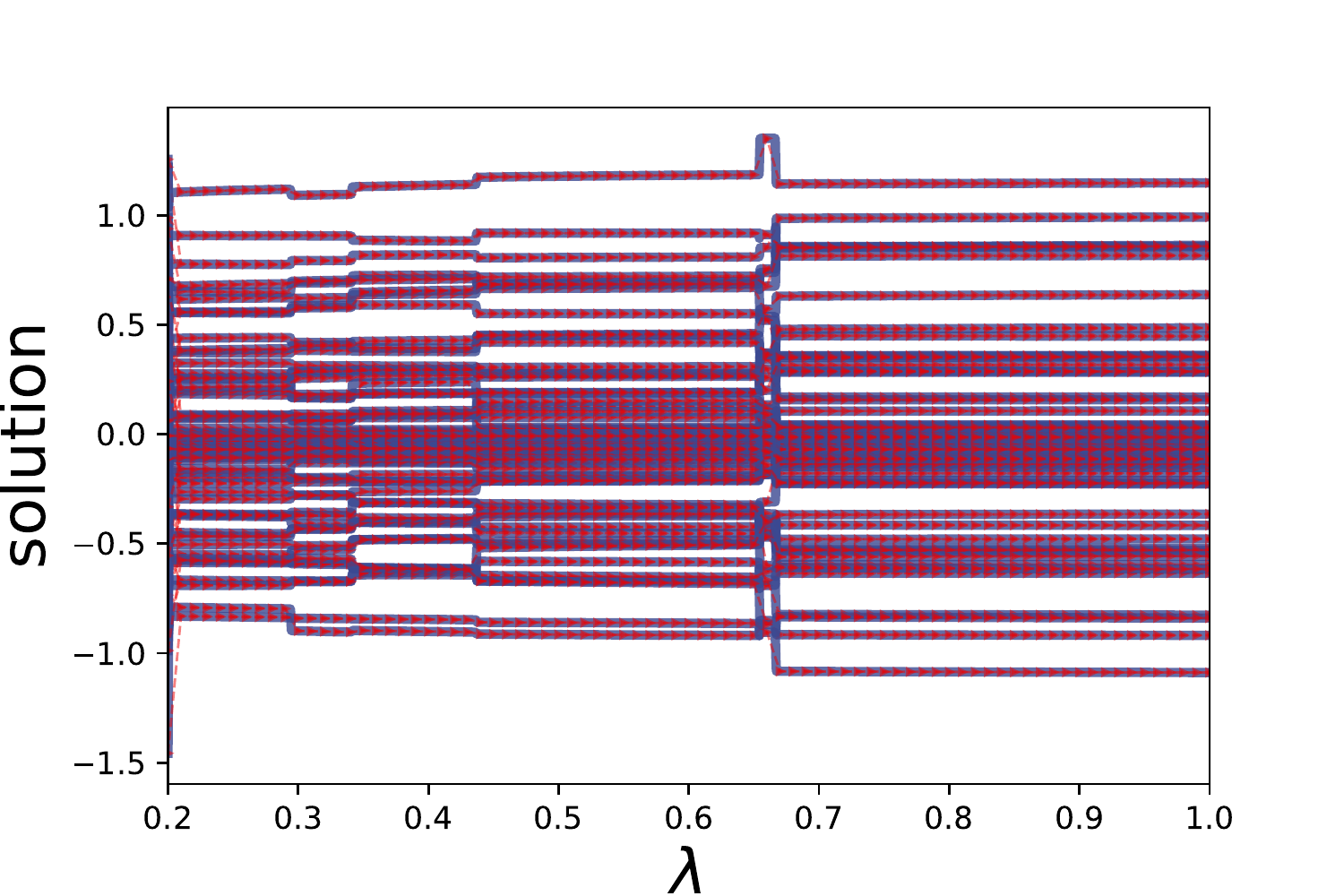} &          \includegraphics[width=0.33\linewidth]{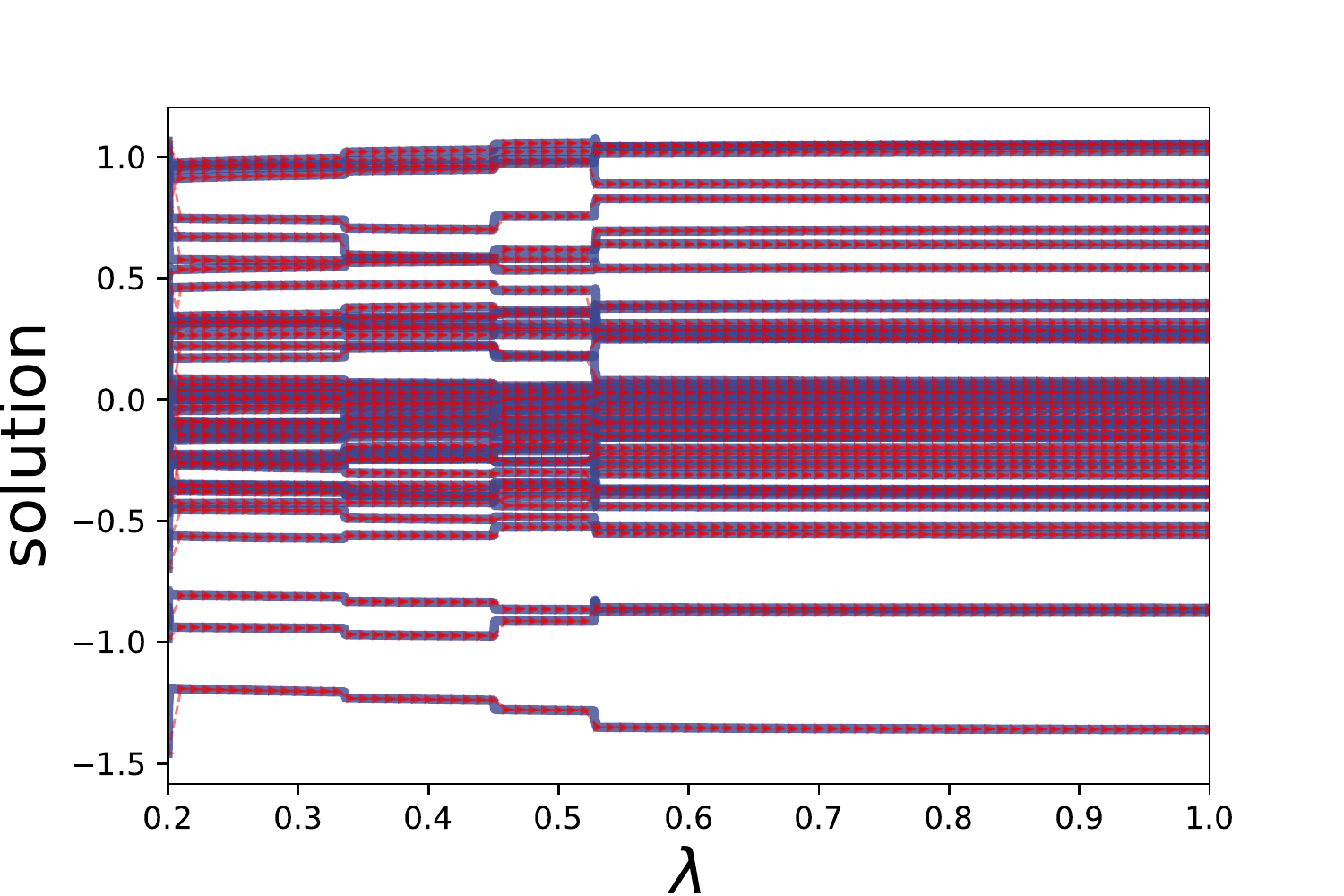} \\
       \includegraphics[width=0.33\linewidth]{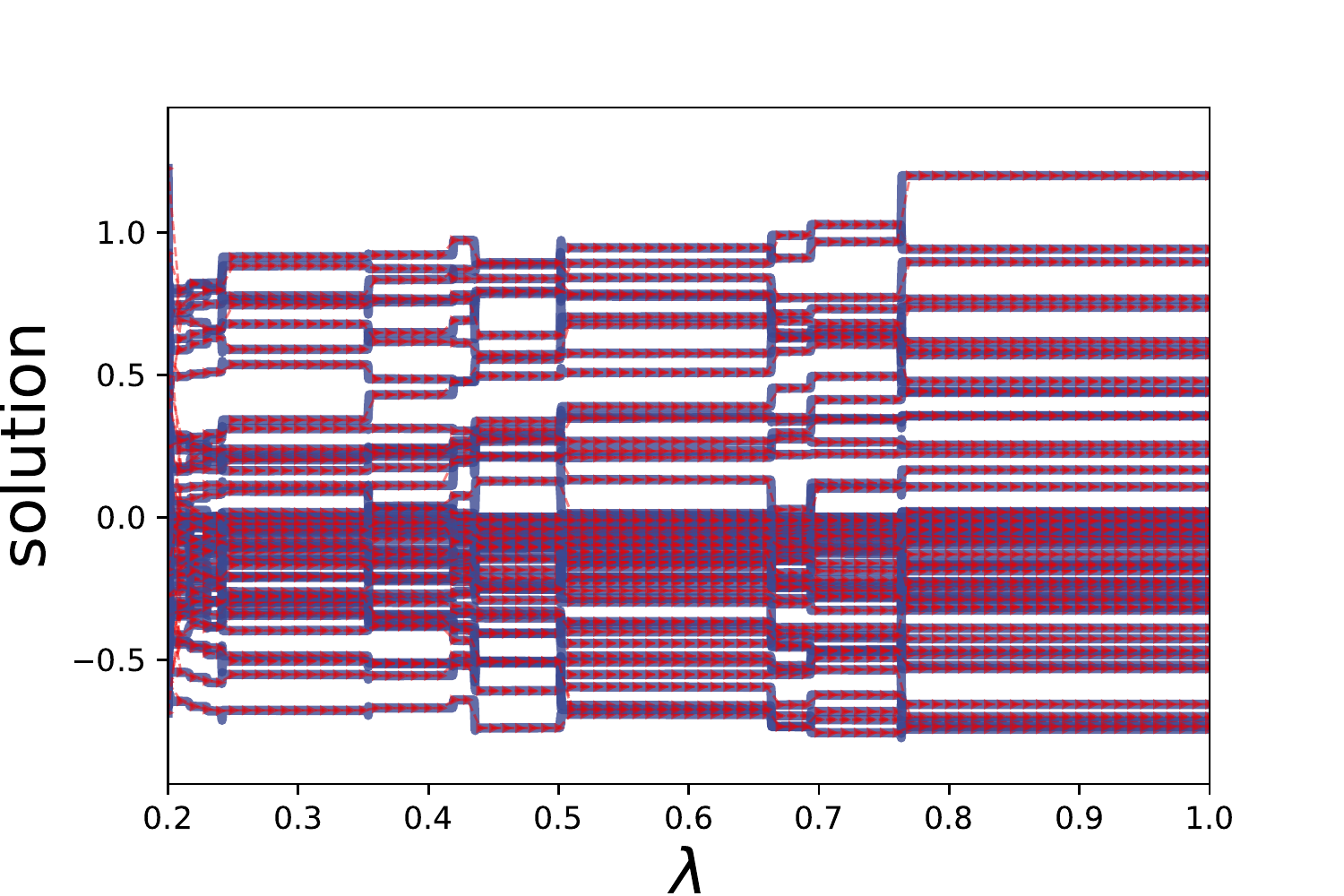} &      \includegraphics[width=0.33\linewidth]{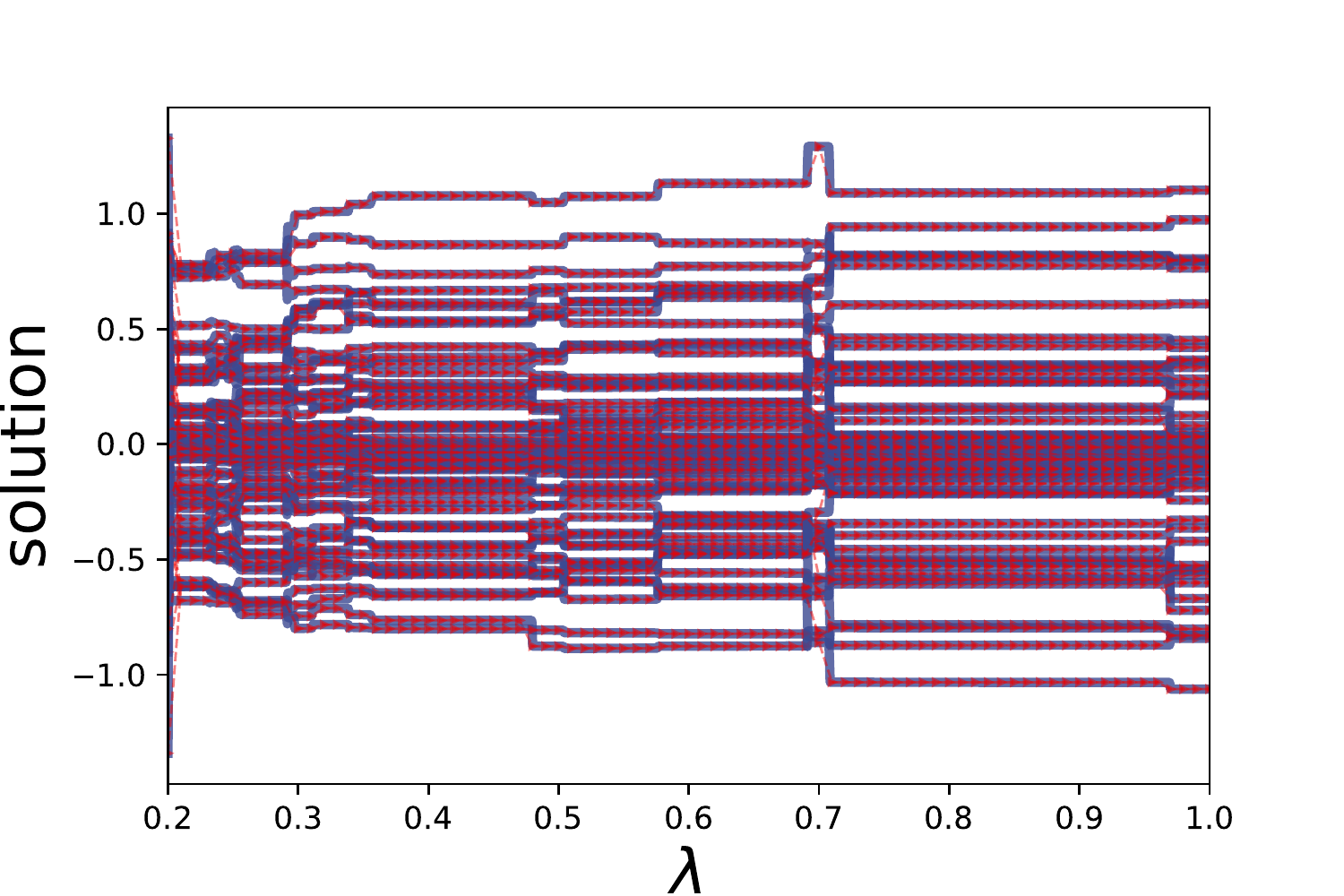} &        \includegraphics[width=0.33\linewidth]{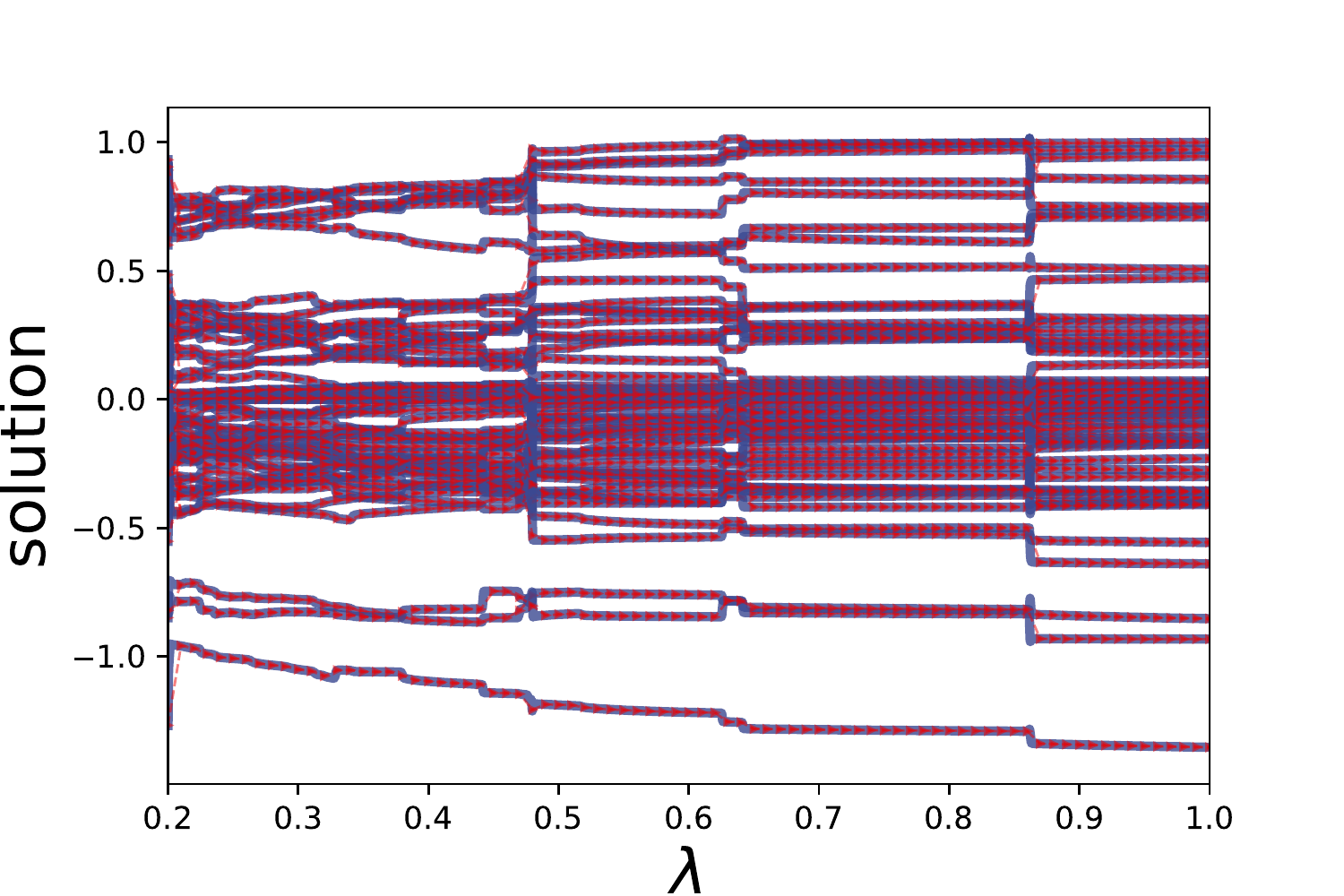} \\
       \includegraphics[width=0.33\linewidth]{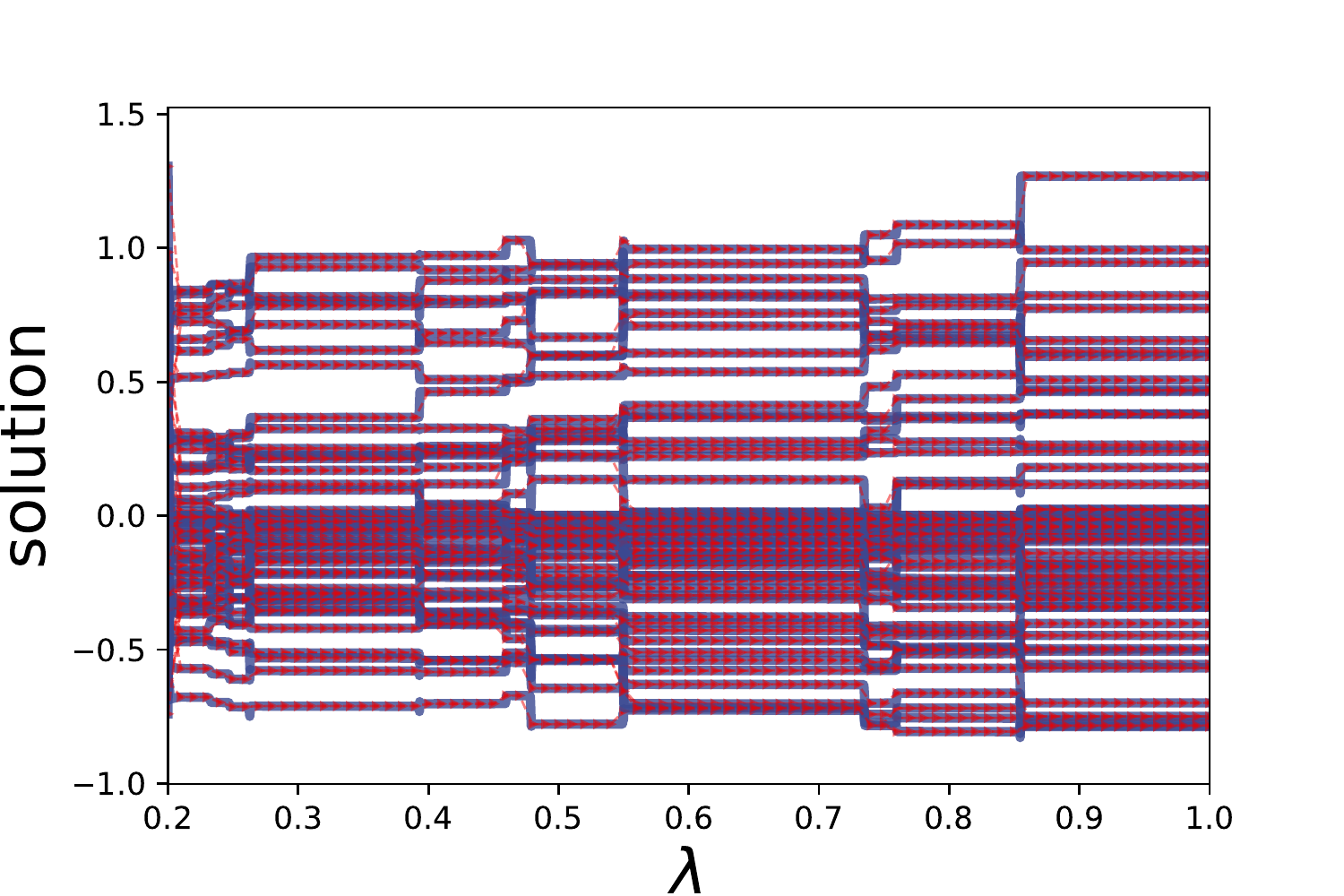} &         \includegraphics[width=0.33\linewidth]{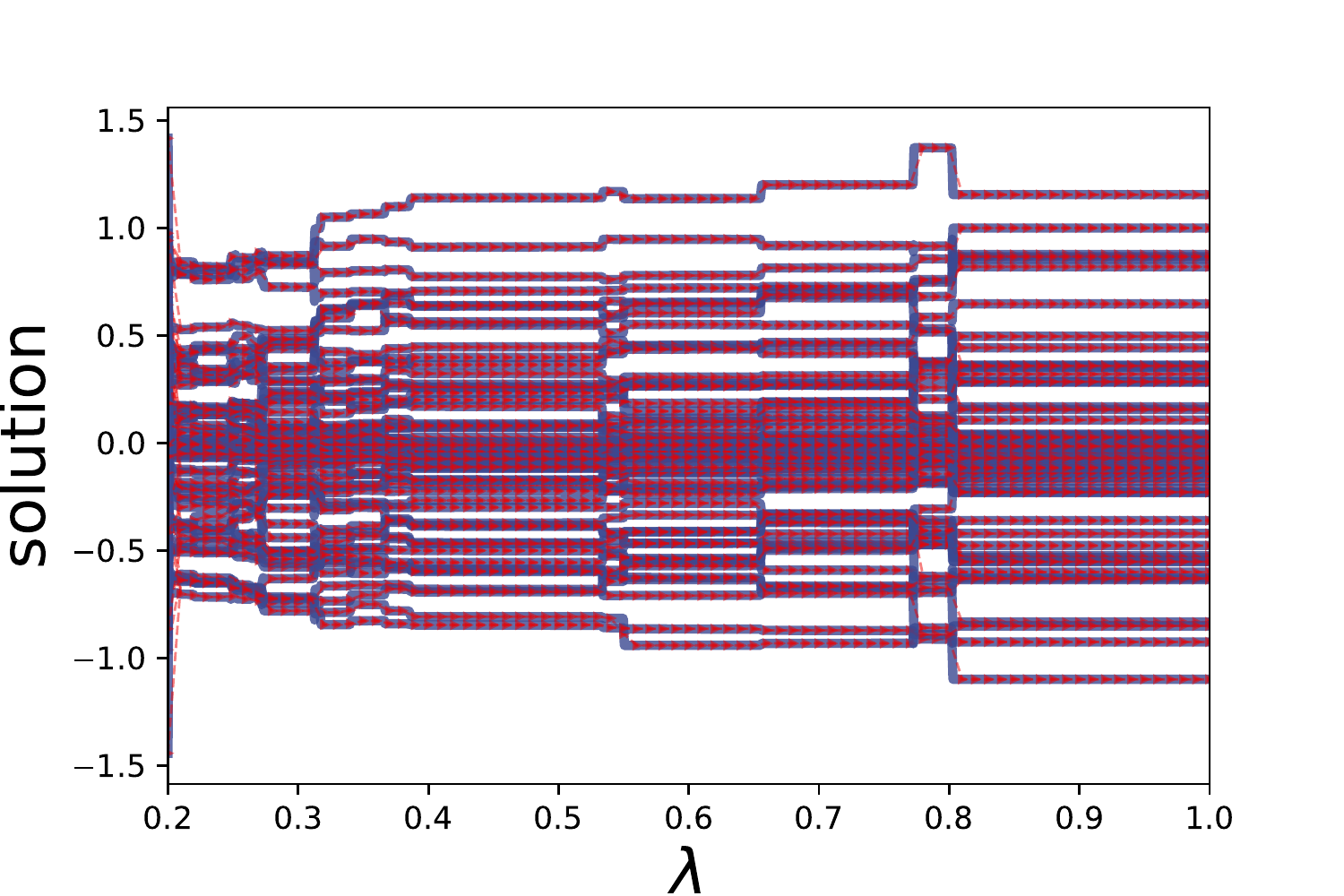} &          \includegraphics[width=0.33\linewidth]{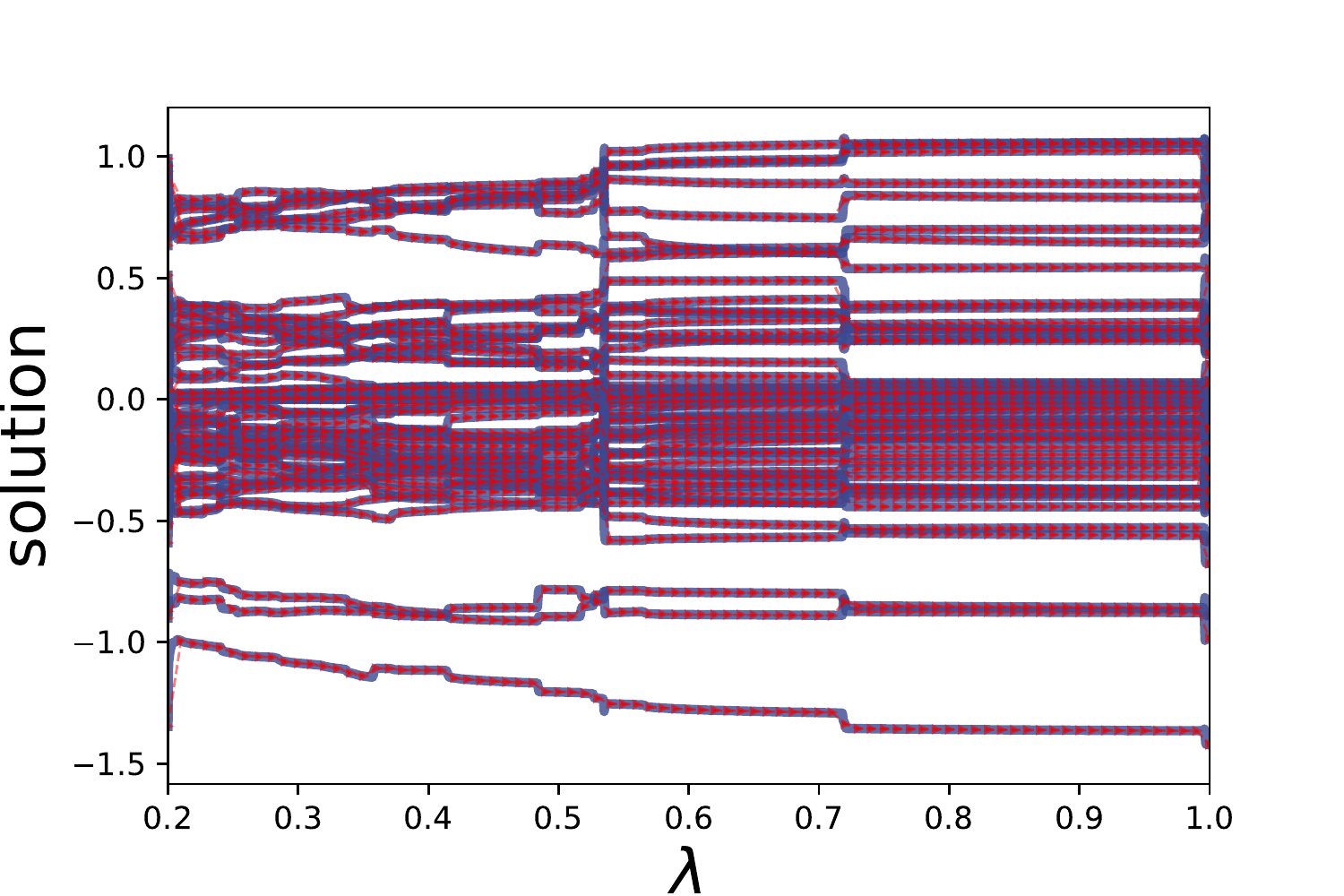} \\
\bottomrule
\end{tabular}

\caption{Age-path of logistic regression with different parameters and datasets. The first two rows of subfigures illustrate age-path using the linear SP-regularizer while the last two rows of subfigures show age-path using the mixture SP-regularizer. For experiments in the first and third row, the $C=25$. For experiments in the second and fourth row, the $C=40$.}
\label{path_3}
\end{figure}

\subsection{More Histograms}
We further show the intrinsic property of the age-path in more experiments. Specifically, we change the dataset, SP-regularizer as well as the value of other hyper-parameter in SVM, Lasso and logistic regression, while recording the number of different types of critical points. The results are illustrated in Figure \ref{hist_1}, \ref{hist_2} and \ref{hist_3}. 

These histograms demonstrate that there exists more turning points on the age-path compared with the jump points. As a result, the heuristic technique applied in \method can avoid extensive unnecessary warm starts by identifying the exact type of each critical point.
\begin{figure}
\centering
\begin{tabular}{ll}
\toprule
      \includegraphics[width=0.4\linewidth]{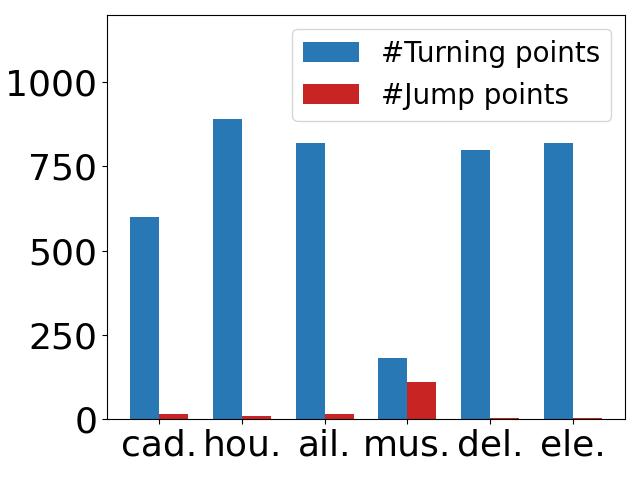}&      \includegraphics[width=0.4\linewidth]{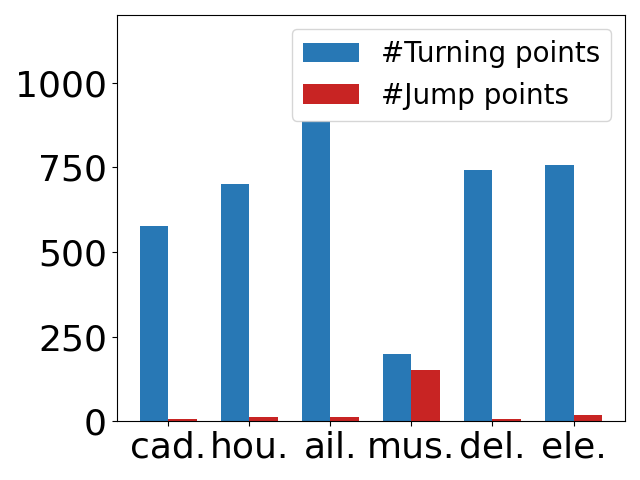}         \\
       \includegraphics[width=0.4\linewidth]{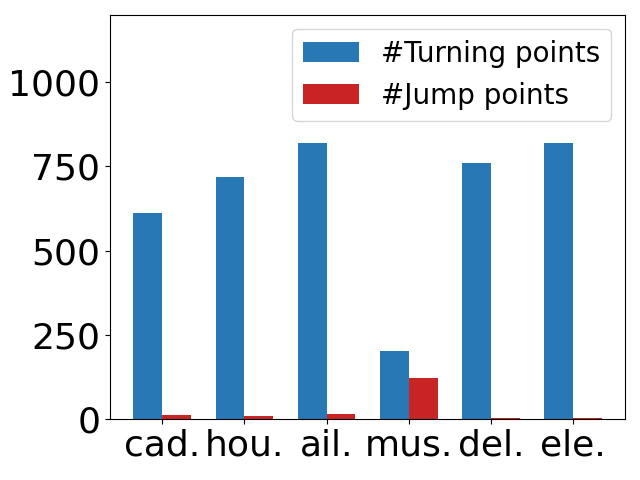} &         \includegraphics[width=0.4\linewidth]{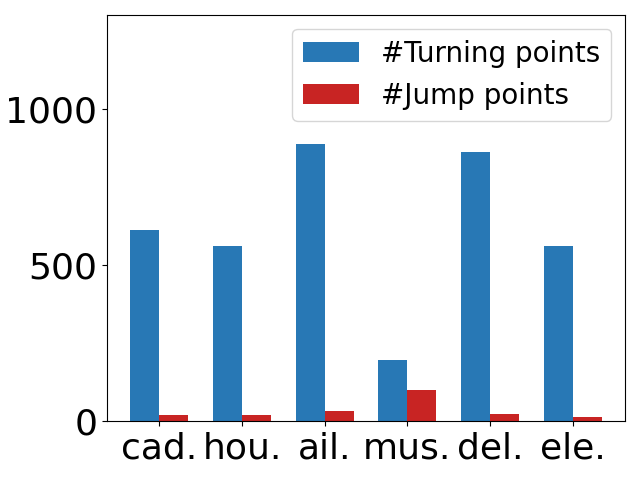}            \\
\bottomrule
\end{tabular}
\caption{The number of different types of critical points in Lasso with different parameters and regularizers. The first row of subfigures illustrate different critical points in the age-path with linear SP-regularizer while the last row of figures show the number of different critical points in
age-path using the mixture SP-regularizer. 
}
\label{hist_1}
\end{figure}
\begin{figure}
\centering
\begin{tabular}{ll}
\toprule
      \includegraphics[width=0.4\linewidth]{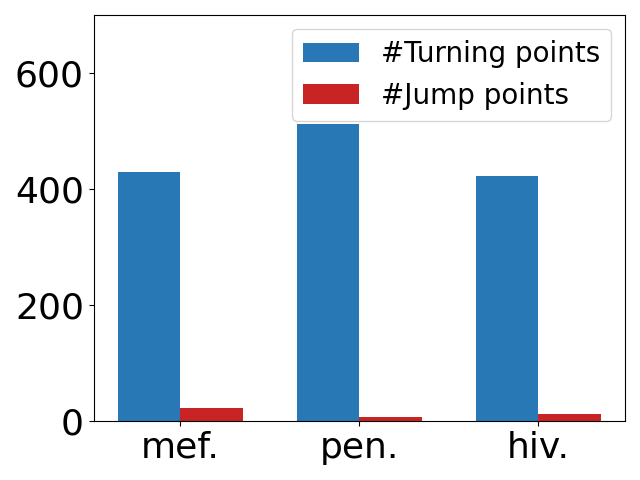}&      \includegraphics[width=0.4\linewidth]{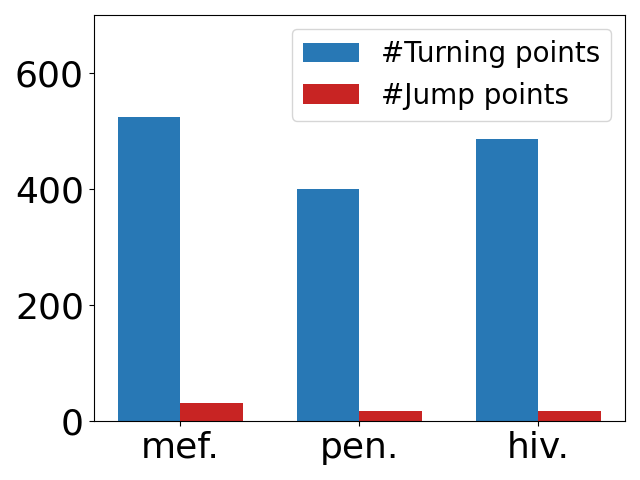}        \\
       \includegraphics[width=0.4\linewidth]{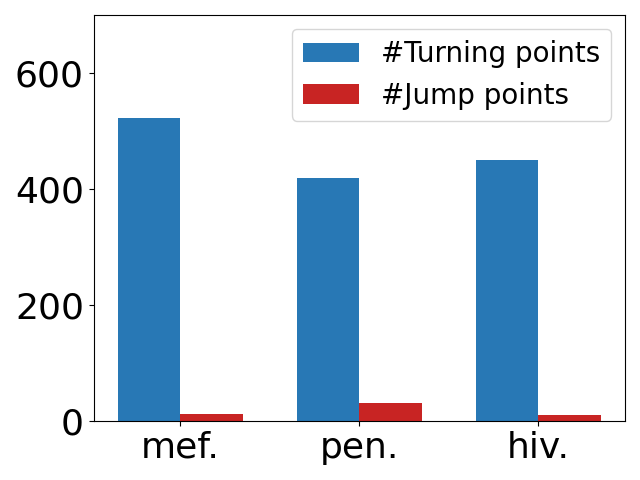} &         \includegraphics[width=0.4\linewidth]{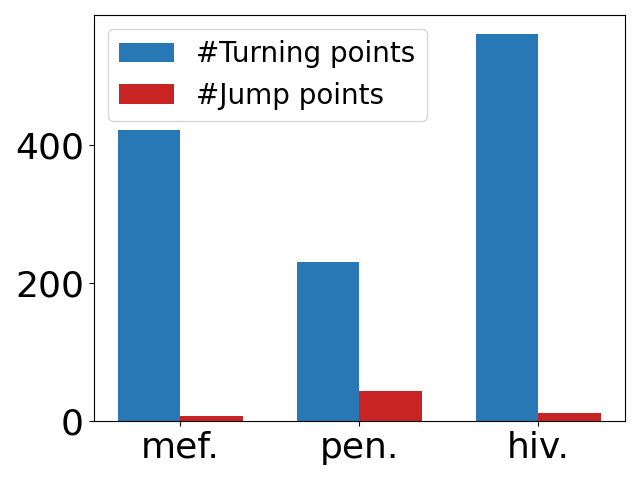}          \\
\bottomrule
\end{tabular}
\caption{The number of different types of critical points in classic SVM with different parameters and regularizers.The first row of subfigures illustrate different critical points in the age-path with linear SP-regularizer while the last row of figures show the number of different critical points in
age-path using the mixture SP-regularizer}
\label{hist_2}
\end{figure}
\begin{figure}
\centering
\begin{tabular}{ll}
\toprule
      \includegraphics[width=0.4\linewidth]{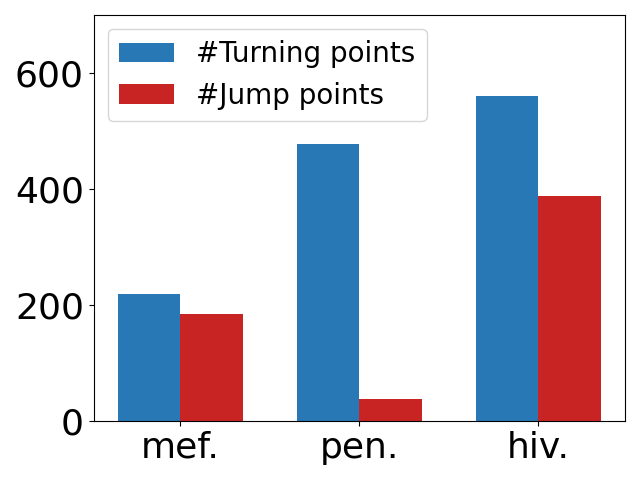}&      \includegraphics[width=0.4\linewidth]{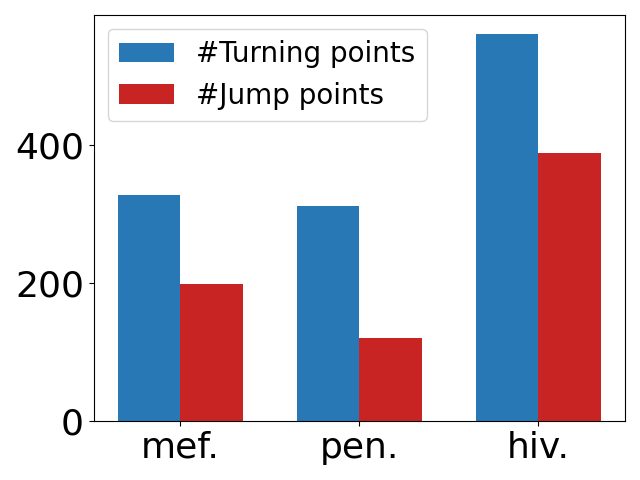}         \\
       \includegraphics[width=0.4\linewidth]{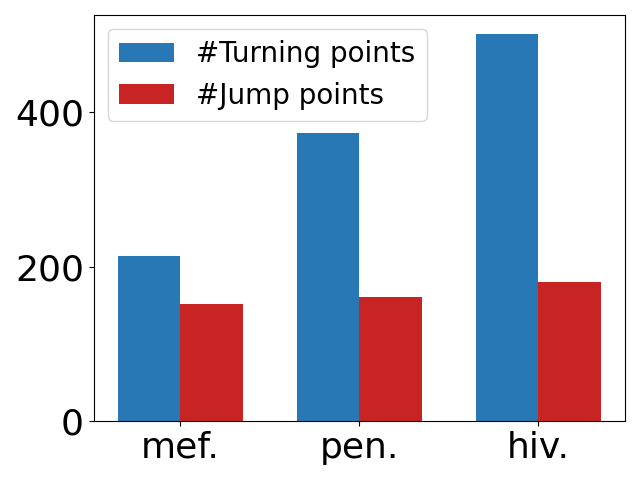} &         \includegraphics[width=0.4\linewidth]{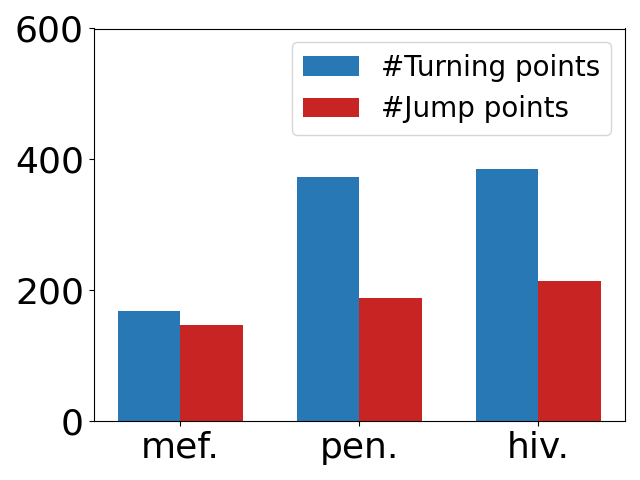}           \\
\bottomrule
\end{tabular}
\caption{The number of different types of critical points in logistic regression with different parameters and regularizers. The first row
of subfigures illustrate different critical points in the age path using linear SP-regularizer while the last row of subfigures show the number of different critical points in age-path using the mixture SP-regularizer.}
\label{hist_3}
\end{figure}

\subsection{Experimental Comparison of robust SVMs and Lasso}
{We also conduct a comparative study to the state-of-the-art robust model for SVMs \cite{chen2015robust,yang2014robust,biggio2011support} and Lasso \cite{nasrabadi2011robust} besides the SPL domain. Please take notice that RLSSVM and Re-LSSVM proposed by \cite{chen2015robust,yang2014robust} are variations of LS-SVM. Especially, we implement the Huber-loss Lasso as a special case of the generalized model proposed in \cite{nasrabadi2011robust}. The hyperparameters of these baselines are chosen from the best performance by grid search. All the experiments conducted on regression tasks are measured by the generalization error. We'd like to emphasize again that our \method framework pursues the best practice of conventional SPL while the SPL, as a special case of curriculum learning naturally owns 
certain shortcomings on the sample diversity \cite{soviany2022curriculum}. Even subject to the inherent defects of vanilla SPL, the result in Table.\ref{tab:compar1} reveals that our method still surpassed SOTA baselines in half of the experiments. The Table.\ref{tab:cmpare2} further implies that \method outperforms the robust SVM in all conducted trails. These numerical results strongly validate the performance of our method under various noise levels.} 

{Concretely speaking, the linear SP-regularizer is utilized in \method in the implementation. The $\gamma_1$ and $\gamma_2$ are hyperparameters of \textbf{RLSSVM} and \textbf{RE-LSSVM}. We use the same hyperparameters of \method in Appendix \ref{Sensitivity}.
}

 \begin{table}[htbp]
   \centering
   
   \scalebox{0.73}{
     \begin{tabular}{rrrrrrrrrrrrrrrl}
     \toprule
     \multicolumn{6}{c}{\textbf{Parameter}}                   & \multicolumn{2}{c}{\textbf{Huber Lasso}} & \multicolumn{2}{c}{\textbf{RLSSVM}} & \multicolumn{2}{c}{\textbf{Re-LSSVM}} & \multicolumn{2}{c}{\textbf{\method}} & \multicolumn{1}{c}{\multirow{2}[4]{*}{\textbf{Noise}}} & \multicolumn{1}{c}{\multirow{2}[4]{*}{\textbf{Dataset}}} \\
\cmidrule{1-6} \cmidrule{7-8} \cmidrule{9-10} \cmidrule{11-12} \cmidrule{13-14}   \multicolumn{1}{c}{$e$} & \multicolumn{1}{c}{$\alpha$} & \multicolumn{1}{c}{$\sigma$} & \multicolumn{1}{c}{$\gamma_1$} & \multicolumn{1}{c}{$\theta$} & \multicolumn{1}{c}{$\gamma_2$} & \multicolumn{1}{c}{Mean} & \multicolumn{1}{c}{Std} & \multicolumn{1}{c}{Mean} & \multicolumn{1}{c}{Std} & \multicolumn{1}{c}{Mean} & \multicolumn{1}{c}{Std} & \multicolumn{1}{c}{Mean} & \multicolumn{1}{c}{Std} &         &  \\
     \midrule
     \midrule
     \rowcolor[rgb]{ .851,  .851,  .851} 0.2     & 0.0006  & 0.7     & 1.9     & 0.2     & 1.9     & 0.399   & 0.003   & 0.449   & 0.107   & \textbf{0.318}   & 0.010   & 0.491   & 0.001   & 0.1     & ailerons \\
     0.1     & 0.0001  & 0.8     & 1.9     & 0.4     & 1.8     & \textbf{0.410}   & 0.003   & 0.573   & 0.082   & 0.499   & 0.010   & 0.49    & 0.002   & 0.2     & ailerons \\
     \rowcolor[rgb]{ .851,  .851,  .851} 0.2     & 0.0016  & 0.5     & 1.9     & 0.2     & 1.9     & \textbf{0.407}   & 0.003   & 0.419   & 0.079   & 0.696   & 0.018   & 0.489   & 0.002   & 0.3     & ailerons \\
     0.2     & 0.0006  & 0.7     & 1.9     & 0.9     & 1.8     & \textbf{0.428}   & 0.005   & 0.500   & 0.050   & 0.877   & 0.014   & 0.491   & 0.001   & 0.4     & ailerons \\
     \rowcolor[rgb]{ .851,  .851,  .851} 0.1     & 0.0008  & 0.8     & 1.9     & 0.2     & 1.9     & \textbf{0.452}   & 0.003   & 0.705   & 0.052   & 1.118   & 0.017   & 0.491   & 0.001   & 0.5     & ailerons \\
     0.2     & 0.0018  & 0.8     & 1.9     & 0.3     & 1.7     & \textbf{0.478}   & 0.003   & 0.565   & 0.044   & 1.159   & 0.025   & 0.490 & 0.015   & 0.6     & ailerons \\
     \rowcolor[rgb]{ .851,  .851,  .851} 0.1     & 0.0002  & 0.8     & 1.9     & 0.2     & 1.7     & \textbf{0.180}   & 0.000   & 0.207   & 0.082   & 0.315   & 0.024   & 0.201   & 0.046   & 0.2     & houses \\
     0.1     & 0.0013  & 0.5     & 1.9     & 0.9     & 1.8     & \textbf{0.182}   & 0.000   & 0.206   & 0.042   & 0.343   & 0.170   & 0.221   & 0.015   & 0.3     & houses \\
     \rowcolor[rgb]{ .851,  .851,  .851} 0.2     & 0.001   & 0.8     & 1.7     & 0.8     & 1.9     & 0.345   & 0.002   & 0.303   & 0.083   & 0.387   & 0.014   & \textbf{0.214}   & 0.002   & 0.1     & music \\
     0.1     & 0.0013  & 0.5     & 1.9     & 0.6     & 1.9     & 0.365   & 0.001   & 0.315   & 0.011   & 0.486   & 0.013   & \textbf{0.213}   & 0.005   & 0.2     & music \\
     \rowcolor[rgb]{ .851,  .851,  .851} 0.2     & 0.0011  & 0.9     & 1.9     & 0.4     & 1.8     & 0.357   & 0.002   & 0.308   & 0.057   & 0.444   & 0.057   & \textbf{0.214}   & 0.003   & 0.3     & music \\
     0.1     & 0.0019  & 1       & 1.9     & 0.4     & 1.7     & 0.406   & 0.003   & 0.422   & 0.087   & 0.412   & 0.015   & \textbf{0.211}   & 0.008   & 0.4     & music \\
     \rowcolor[rgb]{ .851,  .851,  .851} 0.2     & 0.0017  & 0.5     & 1.9     & 0.6     & 1.8     & 0.431   & 0.004   & 0.489   & 0.015   & 0.477   & 0.035   & \textbf{0.210}    & 0.003   & 0.5     & music \\
     0.3     & 0.0011  & 0.8     & 1.9     & 0.3     & 1.9  &   0.484   & 0.005   & 0.598   & 0.017   & 0.521   & 0.019   & \textbf{0.214}  & 0.001   & 0.6     & music \\
     \rowcolor[rgb]{ .851,  .851,  .851} 0.1     & 0.0005  & 0.7     & 1.9     & 0.8     & 1.9     & 0.663   & 0.000   & 0.747   & 0.089   & 0.689   & 0.010   & \textbf{0.647}   & 0.132   & 0.2     & delta elevators \\
     0.1     & 0.0014  & 0.8     & 1.9     & 0.4     & 1.8     & 0.666   & 0.000   & 0.744   & 0.025   & 0.694   & 0.033   & \textbf{0.634}   & 0.132   & 0.3     & delta elevators \\
     \bottomrule
     \end{tabular}%
    }
    \caption{Average generalization erros with the standard deviation in 20 runs on different datasets. The top results in each row are in boldface.}
   \label{tab:compar1}%
 \end{table}%

 \begin{table}[htbp]
   \centering
     \begin{tabular}{lccrrrrr}
     \toprule
     \multicolumn{1}{c}{\multirow{2}[4]{*}{\textbf{Dataset}}} & \multicolumn{2}{c}{\textbf{Parameter}} & \multicolumn{2}{c}{\textbf{Robust-SVM}} & \multicolumn{2}{c}{\textbf{\method}} & \multicolumn{1}{c}{\multirow{2}[4]{*}{\textbf{Noise Level}}} \\
\cmidrule{2-7}             & C       & $\gamma$   & \multicolumn{1}{c}{Mean} & \multicolumn{1}{c}{Std} & \multicolumn{1}{c}{Mean} & \multicolumn{1}{c}{Std} &  \\
     \midrule
     \midrule
     mfeat-pixel & 1       & 0.8     & 0.948   & 0.001   & \textbf{0.980} & \textbf{0.006} & 0.1 \\
     \rowcolor[rgb]{ .749,  .749,  .749} mfeat-pixel & 1       & 0.2     & 0.939   & 0.001   & \textbf{0.988} & \textbf{0.007} & 0.2 \\
     mfeat-pixel & 1       & 0.7     & 0.927   & 0.002   & \textbf{0.980} & \textbf{0.015} & 0.3 \\
     \rowcolor[rgb]{ .749,  .749,  .749} pendigts & 1       & 0.7     & 0.945   & 0.003   & \textbf{0.998} & \textbf{0.007} & 0.1 \\
     pendigts & 1       & 0.6     & 0.944   & 0.002   & \textbf{0.995} & \textbf{0.004} & 0.2 \\
     \rowcolor[rgb]{ .749,  .749,  .749} pendigts & 1       & 0.4     & 0.923   & 0.001   & \textbf{0.995} & \textbf{0.006} & 0.3 \\
     \bottomrule
     \end{tabular}%
        \caption{Average classification accuracy with the standard deviation in 20 runs under different noise levels. The top results in each row are in boldface.}
   \label{tab:cmpare2}%
 \end{table}%

\subsection{Sensitivity Analysis}\label{Sensitivity}
{In this subsection, we use the same settings except the backbone parameters $\alpha$, $C$ and the noise level. The linear SP-regularizer is utilized in GAGA. The classical SVM with linear kernel is chosen as the base model. 
Results in Table \ref{tab:airelons_alpha}, \ref{tab:music_alpha}, \ref{tab:data2_C}, \ref{tab:data3_C} confirm the performances of competing methods with different backbone parameters while retaining the same noise level. Table \ref{tab:ailerons_noise}, \ref{tab:music_noise}, \ref{tab:data2_noise}, \ref{tab:data3_noise} display the running results under different noise levels while the other backbone parameters are kept. The results of the massive simulation studies again strongly demonstrate that our \method owns the best practice of the conventional SPL compared with the baseline methods, regardless of the specific parametric selections.}

 \begin{table}[htbp]
   \centering
     \begin{tabular}{lrrrrrrrr}
     \toprule
     \multicolumn{1}{c}{\multirow{2}[4]{*}{\textbf{Dataset}}} & \multicolumn{1}{c}{\textbf{Parameter}} & \multicolumn{6}{c}{\textbf{Competing Methods}}            & \multicolumn{1}{c}{\multirow{2}[4]{*}{\textbf{Noise Level}}} \\
\cmidrule{2-8}             & \multicolumn{1}{c}{$\alpha$} & \multicolumn{2}{c}{ACS} & \multicolumn{2}{c}{MOSPL} & \multicolumn{2}{c}{\textbf{\method}} &  \\
     \midrule
     \midrule
     \rowcolor[rgb]{ .749,  .749,  .749} ailerons & 0.001   & 0.494   & 0.001   & 0.493   & 0.008   & \textbf{0.490} & \textbf{0.001} & 0.3 \\
     ailerons & 0.002   & 0.493   & 0.007   & 0.492   & 0.016   & \textbf{0.490} & \textbf{0.002} & 0.3 \\
     \rowcolor[rgb]{ .749,  .749,  .749} ailerons & 0.003   & 0.493   & 0.001   & 0.493   & 0.001   & \textbf{0.489} & \textbf{0.002} & 0.3 \\
     ailerons & 0.004   & 0.493   & 0.001   & 0.492   & 0.001   & \textbf{0.491} & \textbf{0.001} & 0.3 \\
     \rowcolor[rgb]{ .749,  .749,  .749} ailerons & 0.005   & 0.493   & 0.001   & 0.493   & 0.001   & \textbf{0.49} & \textbf{0.001} & 0.3 \\
     ailerons & 0.006   & 0.493   & 0.001   & 0.493   & 0.001   & \textbf{0.491} & \textbf{0.001} & 0.3 \\
     \rowcolor[rgb]{ .749,  .749,  .749} ailerons & 0.007   & 0.494   & 0.002   & 0.492   & 0.002   & \textbf{0.491} & \textbf{0.001} & 0.3 \\
     ailerons & 0.008   & 0.493   & 0.002   & 0.493   & 0.002   & \textbf{0.493} & \textbf{0.002} & 0.3 \\
     \rowcolor[rgb]{ .749,  .749,  .749} ailerons & 0.009   & 0.493   & 0.002   & 0.493   & 0.002   & \textbf{0.493} & \textbf{0.002} & 0.3 \\
     ailerons & 0.010   & 0.493   & 0.002   & 0.493   & 0.002   & \textbf{0.493} & \textbf{0.002} & 0.3 \\
     \rowcolor[rgb]{ .749,  .749,  .749} ailerons & 0.011   & 0.493   & 0.002   & 0.493   & 0.002   & \textbf{0.493} & \textbf{0.002} & 0.3 \\
     ailerons & 0.012   & 0.493   & 0.002   & 0.493   & 0.002   & \textbf{0.493} & \textbf{0.002} & 0.3 \\
     \rowcolor[rgb]{ .749,  .749,  .749} ailerons & 0.013   & 0.493   & 0.002   & 0.493   & 0.002   & \textbf{0.493} & \textbf{0.002} & 0.3 \\
     ailerons & 0.014   & 0.493   & 0.002   & 0.493   & 0.002   & \textbf{0.493} & \textbf{0.002} & 0.3 \\
     \rowcolor[rgb]{ .749,  .749,  .749} ailerons & 0.015   & 0.493   & 0.002   & 0.493   & 0.002   & \textbf{0.493} & \textbf{0.002} & 0.3 \\
     \bottomrule
     \end{tabular}%
     \caption{Average generalization errors and the standard deviation in 20 runs with different values of $\alpha$. The top results in each row are in boldface.}
   \label{tab:airelons_alpha}%
 \end{table}%
 
 \begin{table}[htbp]
   \centering
     \begin{tabular}{lrrrrrrrr}
     \toprule
     \multicolumn{1}{c}{\multirow{2}[4]{*}{\textbf{Dataset}}} & \multicolumn{1}{c}{\textbf{Parameter}} & \multicolumn{6}{c}{\textbf{Competing Methods}}            & \multicolumn{1}{c}{\multirow{2}[4]{*}{\textbf{Noise Level}}} \\
\cmidrule{2-8}             & \multicolumn{1}{c}{$\alpha$} & \multicolumn{2}{c}{ACS} & \multicolumn{2}{c}{MOSPL} & \multicolumn{2}{c}{\textbf{\method}} &  \\
     \midrule
     \midrule
     music   & 0.001   & 0.230   & 0.011   & \textbf{0.226}   & 0.009   &  0.227 & \textbf{0.011} & 0.3 \\
     \rowcolor[rgb]{ .749,  .749,  .749} music   & 0.002   & 0.224   & 0.005   & 0.222   & 0.006   & \textbf{0.218} & \textbf{0.004} & 0.3 \\
     music   & 0.003   & 0.227   & 0.011   & 0.223   & 0.008   & \textbf{0.219} & \textbf{0.009} & 0.3 \\
     \rowcolor[rgb]{ .749,  .749,  .749} music   & 0.004   & 0.221   & 0.011   & 0.219   & 0.010   & \textbf{0.213} & \textbf{0.009} & 0.3 \\
     music   & 0.005   & 0.221   & 0.005   & 0.220   & 0.004   & \textbf{0.209} & \textbf{0.005} & 0.3 \\
     \rowcolor[rgb]{ .749,  .749,  .749} music   & 0.006   & 0.218   & 0.009   & 0.216   & 0.009   & \textbf{0.213} & \textbf{0.027} & 0.3 \\
     music   & 0.007   & 0.222   & 0.003   & 0.215   & 0.005   & \textbf{0.211} & \textbf{0.004} & 0.3 \\
     \rowcolor[rgb]{ .749,  .749,  .749} music   & 0.008   & 0.215   & 0.004   & 0.213   & 0.003   & \textbf{0.208} & \textbf{0.002} & 0.3 \\
     music   & 0.009   & 0.216   & 0.006   & 0.213   & 0.005   & \textbf{0.207} & \textbf{0.003} & 0.3 \\
     \rowcolor[rgb]{ .749,  .749,  .749} music   & 0.010   & 0.216   & 0.007   & 0.214   & 0.007   & \textbf{0.208} & \textbf{0.003} & 0.3 \\
     music   & 0.011   & 0.215   & 0.008   & 0.212   & 0.006   & \textbf{0.206} & \textbf{0.005} & 0.3 \\
     \rowcolor[rgb]{ .749,  .749,  .749} music   & 0.012   & 0.211   & 0.007   & 0.208   & 0.007   & \textbf{0.205} & \textbf{0.005} & 0.3 \\
     music   & 0.013   & 0.210   & 0.005   & 0.210   & 0.003   & \textbf{0.207} & \textbf{0.003} & 0.3 \\
     \rowcolor[rgb]{ .749,  .749,  .749} music   & 0.014   & 0.212   & 0.004   & 0.208   & 0.005   & \textbf{0.206} & \textbf{0.003} & 0.3 \\
     music   & 0.015   & 0.211   & 0.006   & 0.210   & 0.004   & \textbf{0.207} & \textbf{0.003} & 0.3 \\
     \rowcolor[rgb]{ .749,  .749,  .749} music   & 0.016   & 0.211   & 0.006   & 0.211   & 0.007   & \textbf{0.206} & \textbf{0.004} & 0.3 \\
     music   & 0.017   & 0.207   & 0.004   & 0.207   & 0.006   & \textbf{0.205} & \textbf{0.004} & 0.3 \\
     \rowcolor[rgb]{ .749,  .749,  .749} music   & 0.018   & 0.210   & 0.006   & 0.206   & 0.007   & \textbf{0.205} & \textbf{0.005} & 0.3 \\
     music   & 0.019   & 0.210   & 0.006   & 0.208   & 0.007   & \textbf{0.206} & \textbf{0.003} & 0.3 \\
     \rowcolor[rgb]{ .749,  .749,  .749} music   & 0.020   & 0.206   & 0.003   & 0.206   & 0.003   & \textbf{0.206} & \textbf{0.003} & 0.3 \\
     music   & 0.021   & 0.207   & 0.005   & 0.207   & 0.005   & \textbf{0.205} & \textbf{0.003} & 0.3 \\
     \rowcolor[rgb]{ .749,  .749,  .749} music   & 0.022   & 0.207   & 0.004   & 0.207   & 0.003   & \textbf{0.205} & \textbf{0.002} & 0.3 \\
     music   & 0.023   & 0.206   & 0.004   & 0.205   & 0.003   & \textbf{0.205} & \textbf{0.003} & 0.3 \\
     \rowcolor[rgb]{ .749,  .749,  .749} music   & 0.024   & 0.206   & 0.003   & 0.207   & 0.005   & \textbf{0.205} & \textbf{0.004} & 0.3 \\
     music   & 0.025   & 0.207   & 0.003   & 0.207   & 0.003   & \textbf{0.204} & \textbf{0.003} & 0.3 \\
     \rowcolor[rgb]{ .749,  .749,  .749} music   & 0.026   & 0.206   & 0.002   & 0.206   & 0.003   & \textbf{0.206} & \textbf{0.003} & 0.3 \\
     music   & 0.027   & 0.208   & 0.004   & 0.206   & 0.004   & \textbf{0.205} & \textbf{0.004} & 0.3 \\
     \rowcolor[rgb]{ .749,  .749,  .749} music   & 0.028   & 0.208   & 0.005   & 0.208   & 0.005   & \textbf{0.205} & \textbf{0.004} & 0.3 \\
     music   & 0.029   & 0.205   & 0.005   & 0.205   & 0.005   & \textbf{0.205} & \textbf{0.003} & 0.3 \\
     \rowcolor[rgb]{ .749,  .749,  .749} music   & 0.03    & 0.207   & 0.004   & 0.207   & 0.004   & \textbf{0.205} & \textbf{0.002} & 0.3 \\
     \bottomrule
     \end{tabular}%
     \caption{Average generalization errors and the standard deviation in 20 runs with different values of $\alpha$. The top results in each row are in boldface.}
   \label{tab:music_alpha}%
 \end{table}%

 \begin{table}[htbp]
   \centering
     \begin{tabular}{lcrrrrrrr}
     \toprule
     \multicolumn{1}{c}{\multirow{2}[4]{*}{\textbf{Dataset}}} & \textbf{Parameter} & \multicolumn{6}{c}{\textbf{Competing Methods}}            & \multicolumn{1}{c}{\multirow{2}[4]{*}{\textbf{Noise Level}}} \\
\cmidrule{2-8}             & C       & \multicolumn{2}{c}{ACS} & \multicolumn{2}{c}{MOSPL} & \multicolumn{2}{c}{\textbf{\method}} &  \\
     \midrule
     \midrule
     \rowcolor[rgb]{ .749,  .749,  .749} mfeat-pixel & 0.100   & 0.932   & 0.012   & 0.959   & 0.005   & \textbf{0.965} & \textbf{0.005} & 0.3 \\
     mfeat-pixel & 0.200   & 0.929   & 0.009   & 0.957   & 0.006   & \textbf{0.966} & \textbf{0.003} & 0.3 \\
     \rowcolor[rgb]{ .749,  .749,  .749} mfeat-pixel & 0.300   & 0.937   & 0.054   & 0.962   & 0.281   & \textbf{0.980} & \textbf{0.015} & 0.3 \\
     mfeat-pixel & 0.400   & 0.936   & 0.011   & 0.966   & 0.013   & \textbf{0.979} & \textbf{0.014} & 0.3 \\
     \rowcolor[rgb]{ .749,  .749,  .749} mfeat-pixel & 0.500   & 0.928   & 0.008   & 0.951   & 0.004   & \textbf{0.961} & \textbf{0.005} & 0.3 \\
     mfeat-pixel & 0.600   & 0.933   & 0.008   & 0.960   & 0.006   & \textbf{0.967} & \textbf{0.004} & 0.3 \\
     \rowcolor[rgb]{ .749,  .749,  .749} mfeat-pixel & 0.700   & 0.922   & 0.012   & 0.946   & 0.006   & \textbf{0.948} & \textbf{0.005} & 0.3 \\
     mfeat-pixel & 0.800   & 0.927   & 0.014   & 0.952   & 0.008   & \textbf{0.960} & \textbf{0.002} & 0.3 \\
     \rowcolor[rgb]{ .749,  .749,  .749} mfeat-pixel & 0.900   & 0.927   & 0.006   & 0.945   & 0.010   & \textbf{0.955} & \textbf{0.010} & 0.3 \\
     mfeat-pixel & 1.000   & 0.925   & 0.007   & 0.948   & 0.006   & \textbf{0.951} & \textbf{0.001} & 0.3 \\
     \rowcolor[rgb]{ .749,  .749,  .749} mfeat-pixel & 1.100   & 0.928   & 0.009   & 0.952   & 0.006   & \textbf{0.955} & \textbf{0.002} & 0.3 \\
     mfeat-pixel & 1.200   & 0.933   & 0.018   & 0.961   & 0.009   & \textbf{0.972} & \textbf{0.005} & 0.3 \\
     \rowcolor[rgb]{ .749,  .749,  .749} mfeat-pixel & 1.300   & 0.929   & 0.011   & 0.953   & 0.009   & \textbf{0.964} & \textbf{0.004} & 0.3 \\
     mfeat-pixel & 1.400   & 0.934   & 0.007   & 0.966   & 0.013   & \textbf{0.971} & \textbf{0.004} & 0.3 \\
     \rowcolor[rgb]{ .749,  .749,  .749} mfeat-pixel & 1.500   & 0.936   & 0.007   & 0.958   & 0.013   & \textbf{0.978} & \textbf{0.003} & 0.3 \\
     mfeat-pixel & 1.600   & 0.920   & 0.007   & 0.923   & 0.007   & \textbf{0.945} & \textbf{0.004} & 0.3 \\
     \rowcolor[rgb]{ .749,  .749,  .749} mfeat-pixel & 1.700   & 0.931   & 0.010   & 0.948   & 0.014   & \textbf{0.969} & \textbf{0.004} & 0.3 \\
     mfeat-pixel & 1.800   & 0.934   & 0.015   & 0.963   & 0.015   & \textbf{0.970} & \textbf{0.005} & 0.3 \\
     \rowcolor[rgb]{ .749,  .749,  .749} mfeat-pixel & 1.900   & 0.929   & 0.007   & 0.957   & 0.006   & \textbf{0.962} & \textbf{0.007} & 0.3 \\
\cmidrule{1-8}     \end{tabular}%
 \caption{Average classification accuracy and the standard deviation in 20 runs with different values of C. The top results in each row are in boldface.}
   \label{tab:data2_C}%
 \end{table}%

 \begin{table}[htbp]
   \centering
     \begin{tabular}{lcrrrrrrr}
     \toprule
     \multicolumn{1}{c}{\multirow{2}[4]{*}{\textbf{Dataset}}} & \textbf{Parameter} & \multicolumn{6}{c}{\textbf{Competing Methods}}            & \multicolumn{1}{c}{\multirow{2}[4]{*}{\textbf{Noise Level}}} \\
\cmidrule{2-8}             & C       & \multicolumn{2}{c}{ACS} & \multicolumn{2}{c}{MOSPL} & \multicolumn{2}{c}{\textbf{\method}} &  \\
     \midrule
     \midrule
     \rowcolor[rgb]{ .749,  .749,  .749} pendigts & 0.100   & 0.984   & 0.003   & 0.984   & 0.006   & \textbf{0.998} & \textbf{0.007} & 0.3 \\
     pendigts & 0.200   & 0.989   & 0.003   & 0.991   & 0.008   & \textbf{0.995} & \textbf{0.004} & 0.3 \\
     \rowcolor[rgb]{ .749,  .749,  .749} pendigts & 0.3   & 0.983   & 0.004   & 0.989   & 0.009   & \textbf{0.995} & \textbf{0.006} & 0.3 \\
     pendigts & 0.400   & 0.985   & 0.002   & 0.988   & 0.007   & \textbf{0.992} & \textbf{0.006} & 0.3 \\
     \rowcolor[rgb]{ .749,  .749,  .749} pendigts & 0.500   & 0.982   & 0.005   & 0.983   & 0.005   & \textbf{0.993} & \textbf{0.004} & 0.3 \\
     pendigts & 0.600   & 0.987   & 0.007   & 0.989   & 0.010   & \textbf{0.992} & \textbf{0.004} & 0.3 \\
     \rowcolor[rgb]{ .749,  .749,  .749} pendigts & 0.700   & 0.986   & 0.007   & 0.988   & 0.006   & \textbf{0.996} & \textbf{0.002} & 0.3 \\
     pendigts & 0.800   & 0.979   & 0.003   & 0.987   & 0.009   & \textbf{0.995} & \textbf{0.007} & 0.3 \\
     \rowcolor[rgb]{ .749,  .749,  .749} pendigts & 0.900   & 0.987   & 0.003   & 0.985   & 0.011   & \textbf{0.991} & \textbf{0.005} & 0.3 \\
     pendigts & 1.000   & 0.989   & 0.006   & 0.986   & 0.012   & \textbf{0.992} & \textbf{0.006} & 0.3 \\
     \rowcolor[rgb]{ .749,  .749,  .749} pendigts & 1.100   & 0.982   & 0.005   & 0.987   & 0.005   & \textbf{0.990} & \textbf{0.554} & 0.3 \\
     pendigts & 1.200   & 0.983   & 0.008   & 0.982   & 0.007   & \textbf{0.994} & \textbf{0.004} & 0.3 \\
     \rowcolor[rgb]{ .749,  .749,  .749} pendigts & 1.300   & 0.981   & 0.006   & 0.991   & 0.007   & \textbf{0.995} & \textbf{0.004} & 0.3 \\
     pendigts & 1.400   & 0.980   & 0.003   & 0.983   & 0.002   & \textbf{0.989} & \textbf{0.005} & 0.3 \\
     \rowcolor[rgb]{ .749,  .749,  .749} pendigts & 1.500   & 0.981   & 0.007   & 0.986   & 0.006   & \textbf{0.995} & \textbf{0.002} & 0.3 \\
     pendigts & 1.600   & 0.981   & 0.009   & 0.985   & 0.011   & \textbf{0.992} & \textbf{0.003} & 0.3 \\
     \rowcolor[rgb]{ .749,  .749,  .749} pendigts & 1.700   & 0.979   & 0.006   & 0.981   & 0.004   & \textbf{0.995} & \textbf{0.006} & 0.3 \\
     pendigts & 1.800   & 0.977   & 0.004   & 0.983   & 0.005   & \textbf{0.993} & \textbf{0.006} & 0.3 \\
     \bottomrule
     \end{tabular}%
        \caption{Average classification accuracy and the standard deviation in 20 runs with different values of C. The top results in each row are in boldface.}
   \label{tab:data3_C}%
 \end{table}%

 \begin{table}[htbp]
   \centering
     \begin{tabular}{lrrrrrrrr}
     \toprule
     \multicolumn{1}{c}{\multirow{2}[4]{*}{\textbf{Dataset}}} & \multicolumn{1}{c}{\textbf{Parameter}} & \multicolumn{6}{c}{\textbf{Competing Methods}}            & \multicolumn{1}{c}{\multirow{2}[4]{*}{\textbf{Noise Level}}} \\
\cmidrule{2-8}             & \multicolumn{1}{c}{$\alpha$} & \multicolumn{2}{c}{ACS} & \multicolumn{2}{c}{MOSPL} & \multicolumn{2}{c}{\textbf{\method}} &  \\
     \midrule
     \midrule
     \rowcolor[rgb]{ .749,  .749,  .749} ailerons & 0.006   & 0.492   & 0.001   & 0.492   & 0.005   & \textbf{0.491} & \textbf{0.001} & 0.1 \\
     ailerons & 0.006   & 0.492   & 0.001   & 0.491   & 0.016   & \textbf{0.49} & \textbf{0.002} & 0.2 \\
     \rowcolor[rgb]{ .749,  .749,  .749} ailerons & 0.006   & 0.493   & 0.001   & 0.493   & 0.001   & \textbf{0.489} & \textbf{0.002} & 0.3 \\
     ailerons & 0.006   & 0.493   & 0.001   & 0.492   & 0.001   & \textbf{0.491} & \textbf{0.001} & 0.4 \\
     \rowcolor[rgb]{ .749,  .749,  .749} ailerons & 0.006   & 0.493   & 0.001   & 0.493   & 0.001   & \textbf{0.491} & \textbf{0.001} & 0.5 \\
     ailerons & 0.006   & 0.493   & 0.001   & 0.492   & 0.001   & \textbf{0.491} & \textbf{0.001} & 0.6 \\
     \rowcolor[rgb]{ .749,  .749,  .749} ailerons & 0.006   & 0.494   & 0.002   & 0.492   & 0.002   & \textbf{0.491} & \textbf{0.001} & 0.7 \\
     ailerons & 0.006   & 0.493   & 0.002   & 0.493   & 0.002   & \textbf{0.493} & \textbf{0.002} & 0.8 \\
     \bottomrule
     \end{tabular}%
     \caption{Average generalization errors with the standard deviation in 20 runs under different noise levels. The top results in each row are in boldface.}
   \label{tab:ailerons_noise}%
 \end{table}%

 \begin{table}[htbp]
   \centering
     \begin{tabular}{lrrrrrrrr}
     \toprule
     \multicolumn{1}{c}{\multirow{2}[4]{*}{\textbf{Dataset}}} & \multicolumn{1}{c}{\textbf{Parameter}} & \multicolumn{6}{c}{\textbf{Competing Methods}}            & \multicolumn{1}{c}{\multirow{2}[4]{*}{\textbf{Noise Level}}} \\
\cmidrule{2-8}             & \multicolumn{1}{c}{$\alpha$} & \multicolumn{2}{c}{ACS} & \multicolumn{2}{c}{MOSPL} & \multicolumn{2}{c}{\textbf{\method}} &  \\
     \midrule
     \midrule
     \rowcolor[rgb]{ .749,  .749,  .749} music   & 0.006 & 0.22    & 0.004   & 0.219   & 0.003   & \textbf{0.214} & \textbf{0.002} & 0.1 \\
     music   & 0.006 & 0.218   & 0.005   & 0.215   & 0.016   & \textbf{0.213} & \textbf{0.005} & 0.2 \\
     \rowcolor[rgb]{ .749,  .749,  .749} music   & 0.006 & 0.221   & 0.009   & 0.217   & 0.006   & \textbf{0.214} & \textbf{0.003} & 0.3 \\
     music   & 0.006 & 0.226   & 0.011   & 0.221   & 0.01    & \textbf{0.211} & \textbf{0.008} & 0.4 \\
     \rowcolor[rgb]{ .749,  .749,  .749} music   & 0.006 & 0.228   & 0.011   & 0.22    & 0.006   & \textbf{0.21} & \textbf{0.003} & 0.5 \\
     music   & 0.006 & 0.493   & 0.001   & 0.492   & 0.001   & \textbf{0.491} & \textbf{0.001} & 0.6 \\
     \rowcolor[rgb]{ .749,  .749,  .749} music   & 0.006 & 0.494   & 0.002   & 0.492   & 0.002   & \textbf{0.491} & \textbf{0.001} & 0.7 \\
     music   & 0.006 & 0.493   & 0.002   & 0.493   & 0.002   & \textbf{0.493} & \textbf{0.002} & 0.8 \\
     \bottomrule
     \end{tabular}%
     \caption{Average generalization errors with the standard deviation in 20 runs under different noise levels. The top results in each row are in boldface.}
   \label{tab:music_noise}%
 \end{table}%

 \begin{table}[htbp]
   \centering
     \begin{tabular}{lcrrrrrrr}
     \toprule
     \multicolumn{1}{c}{\multirow{2}[4]{*}{\textbf{Dataset}}} & \textbf{Parameter} & \multicolumn{6}{c}{\textbf{Competing Methods}}            & \multicolumn{1}{c}{\multirow{2}[4]{*}{\textbf{Noise Level}}} \\
\cmidrule{2-8}             & C       & \multicolumn{2}{c}{ACS} & \multicolumn{2}{c}{MOSPL} & \multicolumn{2}{c}{\textbf{\method}} &  \\
     \midrule
     \midrule
     mfeat-pixel & 1.000   & 0.946   & 0.018   & 0.967   & 0.007   & \textbf{0.980} & \textbf{0.006} & 0.1 \\
     \rowcolor[rgb]{ .749,  .749,  .749} mfeat-pixel & 1.000   & 0.949   & 0.016   & 0.978   & 0.013   & \textbf{0.988} & \textbf{0.007} & 0.2 \\
     mfeat-pixel & 1.000   & 0.937   & 0.054   & 0.962   & 0.281   & \textbf{0.980} & \textbf{0.015} & 0.3 \\
     \rowcolor[rgb]{ .749,  .749,  .749} mfeat-pixel & 1.000   & 0.939   & 0.011   & 0.955   & 0.012   & \textbf{0.968} & \textbf{0.010} & 0.4 \\
     mfeat-pixel & 1.000   & 0.941   & 0.008   & 0.962   & 0.005   & \textbf{0.972} & \textbf{0.008} & 0.5 \\
     \rowcolor[rgb]{ .749,  .749,  .749} mfeat-pixel & 1.000   & 0.946   & 0.011   & 0.971   & 0.006   & \textbf{0.978} & \textbf{0.006} & 0.6 \\
     mfeat-pixel & 1.000   & 0.934   & 0.013   & 0.955   & 0.005   & \textbf{0.959} & \textbf{0.005} & 0.7 \\
     \bottomrule
     \end{tabular}%
        \caption{Average classification accuracy with the standard deviation in 20 runs under different noise levels. The top results in each row are in boldface.}
   \label{tab:data2_noise}%
 \end{table}%

 \begin{table}[htbp]
   \centering
     \begin{tabular}{lcrrrrrrr}
     \toprule
     \multicolumn{1}{c}{\multirow{2}[4]{*}{\textbf{Dataset}}} & \textbf{Parameter} & \multicolumn{6}{c}{\textbf{Competing Methods}}            & \multicolumn{1}{c}{\multirow{2}[4]{*}{\textbf{Noise Level}}} \\
\cmidrule{2-8}             & C       & \multicolumn{2}{c}{ACS} & \multicolumn{2}{c}{MOSPL} & \multicolumn{2}{c}{\textbf{\method}} &  \\
     \midrule
     \midrule
     \rowcolor[rgb]{ .749,  .749,  .749} pendigts & 1.000   & 0.979   & 0.012   & 0.984 & 0.006   & \textbf{0.994} & \textbf{0.004} & 0.1 \\
     pendigts & 1.000   & 0.980   & 0.003   & 0.991   & 0.008   & \textbf{0.998} & \textbf{0.004} & 0.2 \\
     \rowcolor[rgb]{ .749,  .749,  .749} pendigts & 1.000   & 0.989   & 0.006   & 0.986   & 0.012   & \textbf{0.992} & \textbf{0.006} & 0.3 \\
     pendigts & 1.000   & 0.980   & 0.009   & 0.990   & 0.005   & \textbf{0.996} & \textbf{0.006} & 0.4 \\
     \rowcolor[rgb]{ .749,  .749,  .749} pendigts & 1.000   & 0.974   & 0.004   & 0.994   & 0.005   & \textbf{0.994} & \textbf{0.004} & 0.5 \\
     pendigts & 1.000   & 0.984   & 0.007   & 0.986   & 0.010   & \textbf{0.996} & \textbf{0.004} & 0.6 \\
     \rowcolor[rgb]{ .749,  .749,  .749} pendigts & 1.000   & 0.970   & 0.007   & 0.972   & 0.006   & \textbf{0.974} & \textbf{0.006} & 0.7 \\
     \bottomrule
     \end{tabular}%
        \caption{Average classification accuracy with the standard deviation in 20 runs under different noise levels. The top results in each row are in boldface.}
   \label{tab:data3_noise}%
 \end{table}%

\clearpage

\section{Limitations and Broader Impact}
\label{limit_impact}
\paragraph{Limitations.} The current framework \method is only applicable to the vanilla SPL paradigm, while some variants of SPL (\emph{e.g.}, self-paced learning with diversity \cite{jiang2014self}) could have a group of hyperparameters more than the mere $\lambda$. In this instance, the original solution paths escalates into the solution surfaces, and made it harder to track the solutions. Another point is that our argument assumed that the target function is a biconvex problem, while many complex losses (\emph{e.g.}, deep neural networks) could be the function of strongly non-convex.

\paragraph{Broader Impact.} 
The parties with limited computational resource may benefit from the efficiency of our proposed work. Meanwhile, the research groups that sensitive to age parameters in SPL may benefit from the robustness of our work. For social impact, our work improve the hyperparameter search and helps reducing the computation cost, which saves carbon emissions during the training process.

\section{Code Readme}
This section explains how to use the code implementing the proposed \method.
\subsection{Archive content}

\begin{enumerate}
\item The fundamental implementation of our methodology \method is called {\color{blue}\texttt{GAGA.py}}, which includes all the functions used in our experiment. In more detail, they are named in the form of [GAGA]\_[Model\_Name]\_[SPL\_Regularizer], \emph{e.g.}, GAGA\_svm\_linear.
\item The {\color{blue}\texttt{Evaluation.py}} provides functions for evaluating solution path from different models including ACS.
\item The {\color{blue}\texttt{Input\_Data.py}} offers file-IO for all datasets used in our experiment.
\item The {\color{blue}\texttt{ACS.py}} provides ACS algorithm for solving SPL.
\end{enumerate}

\subsection{Reproducing the results of the article}

 For the sake of quick and convenient reproducing experiments and  checking our findings, we provide a demo named {\color{blue}\texttt{main.py}}. We use open-source datasets and provide file-IO functions along with  detailed pre-processing pipeline for users. Any used datasets in our experiment can be easily downloaded them from the UCI and OpenML website and put them in the {\color{blue}\texttt{datasets}} folder. The required dependencies and running environment is recorded in  {\color{blue}\texttt{Environment.txt}}.
\end{document}